\documentclass{article}

\PassOptionsToPackage{numbers, compress}{natbib}

\usepackage[main, final]{neurips_2026}

\usepackage[utf8]{inputenc} 
\usepackage[T1]{fontenc}    
\usepackage{hyperref}       
\usepackage{url}            
\usepackage{booktabs}       
\usepackage{amsfonts}       
\usepackage{nicefrac}       
\usepackage{microtype}      
\usepackage{xcolor}         
\usepackage{tcolorbox}

\usepackage[table]{xcolor}
\usepackage{enumitem}
\usepackage{multirow}
\usepackage{makecell}
\usepackage{subcaption}
\usepackage{wrapfig}
\usepackage{graphicx}

\usepackage{amsmath}
\usepackage{amssymb}
\usepackage{mathtools}
\usepackage{amsthm}

\theoremstyle{plain}
\newtheorem{theorem}{Theorem}[section]

\newtheorem{lemma}{Lemma}[section]

\theoremstyle{definition}

\newtheorem{assumption}{Assumption}[section]

\theoremstyle{remark}
\newtheorem{remark}{Remark}[section]
\newtheorem{claim}{Claim}[section]

\newcommand{\KL}{\operatorname{\mathbb{D}_{\mathrm{KL}}}}

\newcommand{\Var}{\operatorname{Var}}
\newcommand{\supp}{\operatorname{Supp}}
\newcommand{\err}[1]{{\tiny $\pm$\,#1}}
\definecolor{boxborder}{RGB}{65, 105, 185} 
\definecolor{boxbg}{RGB}{235, 240, 250} 
\usepackage{fontawesome5}

\title{Reflection Anchors for Propagation-Aware Visual Retention in Long-Chain Multimodal Reasoning}

\author{
\textbf{Xuan Gong}$^{1,2}$ \quad
\textbf{Hanbo Huang}$^{1}$ \quad
\textbf{Hao Zheng}$^{1}$ \quad
\textbf{Yiran Zhang}$^{1}$ \quad
\textbf{Wenbin Dai}$^{1,2}$ \\
\textbf{Weishu Zhao}$^{1,2}$ \quad
\textbf{Shiyu Liang}$^{1,2,\dagger}$ \\
$^{1}$Shanghai Jiao Tong University \quad
$^{2}$Shanghai AI Laboratory \\
{
\href{https://github.com/coder-gx/RAPO}{\faGithub\ Code}
\quad
\href{https://coder-gx.github.io/RAPO/}{\faGlobe\ Project Page}
}
}

\begin{document}


\maketitle

{\let\thefootnote\relax\footnotetext{\hspace{-1.8em}$^\dagger$Corresponding author. Correspondence to \texttt{lsy18602808513@sjtu.edu.cn}.\\%
\phantom{$^\dagger$}Technical inquiries: Xuan Gong (\texttt{gongxuan0610@sjtu.edu.cn}).}}



\begin{abstract}
 
Long chain-of-thought (CoT) reasoning improves large vision–language models, but visual information often fades during generation, limiting long-horizon multimodal reasoning. Existing methods either re-inject vision at inference or train policies for stronger grounding, but \emph{where} to intervene relies on perception heuristics rather than principled gain analysis, and \emph{how} local visual influence propagates remains implicit. We study this problem from an information-theoretic standpoint and derive a lower bound on the downstream visual gain of a one-step intervention, which suggests two factors: local branching room (token entropy) and downstream visual propagation potential (suffix divergence from a vision-marginalized reference).  Guided by this analysis, we propose reflection-anchor policy optimization \textbf{(RAPO)}, a GRPO-based policy optimization method that selects high-entropy reflection anchors and optimizes a chain-masked finite-window KL surrogate for downstream visual dependence. Experiments on reasoning-intensive and general-domain benchmarks show that RAPO delivers substantial gains over strong baselines across multiple LVLM backbones. Mechanism analyses further indicate that reflection anchors are enriched for visually sensitive decision points and that RAPO increases contrastive visual-dependence signals along generated trajectories.
\end{abstract}

\section{Introduction}

Test-time scaling via long Chain-of-Thought (CoT)~\cite{wei2023chainofthoughtpromptingelicitsreasoning,kojima2023largelanguagemodelszeroshot} has emerged as a cornerstone of recent advancements in Large Language Models (LLMs)~\citep{Guo_2025,kimiteam2025kimik2openagentic,yang2025qwen3technicalreport}.
Inspired by this, recent works have adapted CoT paradigms to Large Vision-Language Models (LVLMs)~\citep{kimiteam2025kimivltechnicalreport,bai2025qwen3vltechnicalreport,huang2025visionr1incentivizingreasoningcapability,meng2025mmeurekaexploringfrontiersmultimodal}. However, unlike text-only LLMs that effectively maintain context through linguistic recurrence, LVLMs struggle to sustain visual information across extended reasoning sequences. This ``visual forgetting'' reflects a mismatch between prefix-based visual embeddings and visual reasoning demands~\citep{li2025lostembeddingsinformationloss,sun2025mitigating,masry2025alignvlmbridgingvisionlanguage}.

Recent studies have followed two main directions to address visual forgetting. 
The former revisits or injects visual evidence through re-attention, reflection, or interleaved visual-token reasoning~\citep{sun2025mitigating,jian2025look,chu2025qwen,chen2025mintcot,gao2025interleaved,chung2025v1learningpointvisual,qin2025chain,cheng2026visual,sun2025latentchainofthoughtvisualreasoning}.
These methods help restore or maintain visual evidence during long reasoning, but they often require additional supervision, intervention rules, architectural changes, or generation-time computation.
The second line trains policies to generate more visually grounded trajectories via perception-aware RL~\citep{wang2025perceptionawarepolicyoptimizationmultimodal,huang2025spotlighttokenperceptionmultimodal,li2026rethinkingtokenlevelpolicyoptimization}, yet leaves the control implicit in the learned policy. 
Despite their different mechanisms, both lines treat downstream visual dependence as an indirect consequence of either context injection or learned policy bias. What remains unspecified is an explicit rule, for a given trajectory, for how native decoding itself should be modified to preserve visual influence in future reasoning.
This gap motivates the following question:

\begin{figure*}[t]
    \centering
    \includegraphics[width=\linewidth]{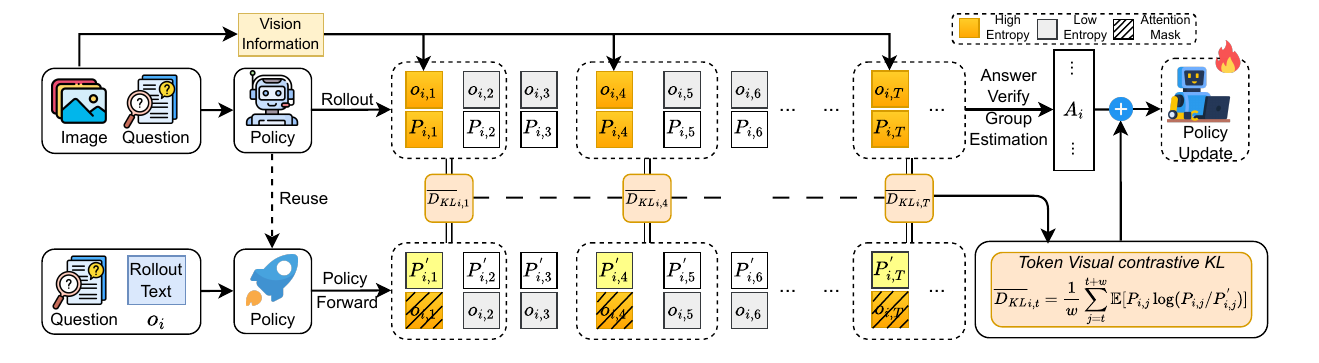}
    \vspace{-0.2cm}
    \caption{\small\textbf{Overview of RAPO.}  
    Given an input, the policy generates the $i$-th reasoning rollout $\{o_{i,t}, P_{i,t}\}_{t=1}^T$, where $o_{i,t}$ is the $t$-th generated token and $P_{i,t}$ is the corresponding logit-induced next-token distribution.
A chain-masked forward pass produces vision-masked distributions $\{P'_{i,t}\}_{t=1}^T$. RAPO targets high-entropy reflection anchors and uses the KL divergence between $P_{i,t}$ and $P_{i,t}'$ to encourage visual influence propagation  across subsequent tokens, encouraging stronger visual-dependence along reasoning.
    }
    \label{fig:RAPOFlowchart}
    \vspace{-0.2cm}
\end{figure*}

\textit{Can visual dependence be preserved by minimally reshaping native decoding at sparse steps, each selected to maximize downstream influence on future reasoning?}

We formulate this as optimizing the downstream visual-dependence gain induced by in-place intervention of the next-token distribution.
We derive a lower bound showing that nontrivial visual gains arise at tokens where (1) the next-token distribution has high entropy, leaving room for effective local intervention, and (2) the future continuation remains distributionally separated from its visual-marginalized counterpart, allowing visual influence to propagate. This suggests a simple algorithmic principle: select intervention sites by uncertainty, and train interventions by propagation. 
We call the selected positions \emph{reflection anchors}---a name that reflects an empirical observation: high-entropy positions coincide with reasoning-procedural tokens such as ``examine'' and ``reconsider.''
Based on this principle, we propose \textbf{RAPO} (Figure~\ref{fig:RAPOFlowchart}), a GRPO-based policy optimization algorithm~\citep{shao2024deepseekmathpushinglimitsmathematical} that selects reflection anchors by reference-policy entropy and optimizes a branching-room-weighted contrastive KL between normal continuations and their chain-masked counterfactuals.
Thus, RAPO realizes \emph{propagation-aware sparse policy reshaping} without modifying the inference-time decoder.

Our contributions are as follows:

\begin{itemize}[leftmargin=*, itemsep=1pt, parsep=0pt, topsep=0pt]
\item We provide an information-theoretic view on the downstream visual-dependence gain in long CoT reasoning, characterizing where and how in-place intervention can shape future reasoning.
\item Based on the principle, we propose \textbf{RAPO}, a GRPO-based reinforcement learning algorithm that targets reflection anchors and reinforces propagation of visual influence across reasoning steps.
\item Extensive experiments on various benchmarks show that RAPO outperforms baselines across multiple base models in long-chain multimodal reasoning.
\item Empirical studies highlight the contribution of each component and provide guidance for effective downstream visual-dependence gain.
\end{itemize}


{
\section{Related Work}
\label{sec:related_work}

\textbf{Visual restoration for long CoT reasoning.} Long CoT reasoning in LVLMs is prone to \emph{visual forgetting}: later steps follow the textual prefix and gradually lose dependence on the image. Existing methods address this by restoring visual evidence during generation. Interleaved methods inject vision explicitly. ICoT~\cite{gao2025interleaved} inserts image--text rationales, and MINT-CoT~\cite{chen2025mintcot} selects visual tokens on the fly. CoVT~\cite{qin2025chain} keeps the same idea but moves the evidence into a continuous visual-token space. Decoder-side methods inject vision internally. TVC~\cite{sun2025mitigating} reintroduces image conditioning at selected stages, while v1~\cite{chung2025v1learningpointvisual} copies image-patch representations into decoding. Adaptive methods further decide when to restore vision: Qwen-LA~\cite{chu2025qwen} learns when to revisit the image, and Reflection-V~\cite{jian2025look} repairs reasoning through vision-centered reflection.
These methods share one pattern: they detect drift, revisit the image, and inject visual evidence back into generation. They are effective, but they preserve visual dependence by adding external interventions to the native decoder, including extra triggers, modules, supervision, or inference-time computation.

\textbf{Policy optimization for visual grounding.}
A complementary line avoids explicit visual operations at inference time. Instead of re-injecting vision during decoding, it optimizes the policy so that visual grounding is absorbed into native generation. This direction builds on RL frameworks for long-chain reasoning, such as GRPO~\cite{shao2024deepseekmathpushinglimitsmathematical} and DAPO~\cite{yu2025dapo}, and extends them to multimodal reasoning through methods such as Vision-R1~\cite{huang2025visionr1incentivizingreasoningcapability} and MM-Eureka~\cite{meng2025mmeurekaexploringfrontiersmultimodal}. Beyond final-answer rewards, VL-Rethinker~\cite{wang2025vlrethinkerincentivizingselfreflectionvisionlanguage}, SRPO~\cite{wan2025srpo}, and Mulberry~\cite{yao2024mulberry} improve intermediate trajectories through reflection or refinement. More recent perception-aware methods, including PAPO~\cite{wang2025perceptionawarepolicyoptimizationmultimodal}, VPPO~\cite{huang2025spotlighttokenperceptionmultimodal}, PeRL~\cite{zhang2025perl}, and VAPO~\cite{tian2026more}, push the policy toward stronger perceptual grounding and visual utilization.
This yields a clean test-time interface: normal decoding, no visual re-injection. Yet the mechanism remains implicit: rewards promote grounding, but do not locate visual forgetting or trace how a local correction propagates through later reasoning.

}






\section{Problem Formulation: Selective Policy Reshaping for Visual Retention}
\label{sec:prelim}

\textbf{Visual reasoning.}
Let \(V\) denote the visual input, \(X\) the textual prompt, and \(Y_{1:T}\) the reasoning trajectory\footnote{In this paper, uppercase and lowercase letters denote the random variables and their realizations, respectively.}. A native LVLM \(p_\theta\) generates
\(
p_\theta(Y_{1:T}\mid V,X)=\prod_{t=1}^T p_\theta(Y_t\mid V,X,Y_{<t}).
\)
We write \(H_t=(X,Y_{<t})\) for the history before step \(t\), and \(Y_{\ge t}=(Y_t,\ldots,Y_T)\) for the future continuation. In long-horizon visual reasoning, later tokens may increasingly condition on the textual prefix rather than the image. Visual retention therefore requires not only final-answer correctness, but also sustained dependence of \(Y_{\ge t}\) on \(V\), which final-answer-only policy optimization does not directly enforce.

\textbf{Selective policy reshaping.}
We learn a single decoder \(\pi\) by locally reshaping the native decoder \(p_\theta\) at only a few steps. For a native trajectory \(y_{1:T}\), let \(\mathcal A\subseteq[T]\triangleq\{1,\dots,T\}\) be the selected steps and {let \(h_t=(x,y_{<t})\) denote a history instance.} If \(t\notin\mathcal A\), the decoder stays native, \(\pi(\cdot\mid v,h_t)=p_\theta(\cdot\mid v,h_t)\). If \(t\in\mathcal A\), the next-token distribution is reshaped. A useful reshaping should not merely make the current token more visual. It should make the later continuation more visual. {During training, \(\mathcal A\) is an instance-level oracle used to construct the learning signal for a single decoder \(\pi\). The learned decoder then internalizes where and how to deviate from \(p_\theta\). At inference, no oracle set is provided and \(\pi\) decodes normally.}

\begin{wrapfigure}{r}{0.58\textwidth}
    \centering
     \vspace{-0.3cm}
    \includegraphics[width=\linewidth]{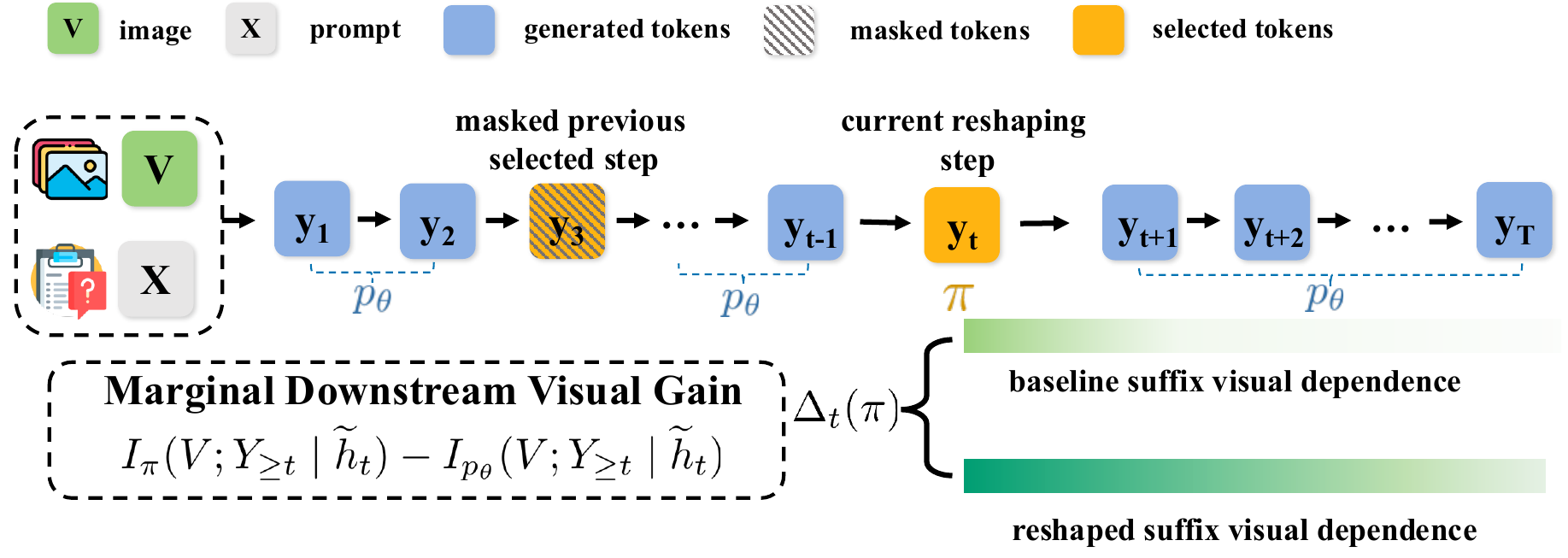}
    \caption{\textbf{Marginal downstream visual gain.}}
    \vspace{-0.3cm}
    \label{fig:formulation}
\end{wrapfigure}

\textbf{Downstream visual gain.}
The key question is whether a local change at step \(t\) carries visual evidence into later reasoning. To assign this effect to step \(t\), we condition on a masked history \(\widetilde h_t=(x,M_t(y_{<t};\mathcal{A}))\), where \(M_t\) masks the previously selected positions in \(\mathcal A\cap[t-1]\). This masking is essential: 
{masking suppresses the direct contribution of earlier selected steps, allowing the measured gain to be attributed marginally to step \(t\).}
Based on the masked history \(\widetilde h_t\), we perform a one-step intervention: sample \(Y_t\) from \(\pi(\cdot\mid v,\widetilde h_t)\), then roll back to the native decoder \(p_\theta\) for all later steps. This rollback makes the credit local: any downstream change is attributed to the reshaping at step \(t\), not to later use of \(\pi\). 
Let \(I_\pi(V;Y_{\ge t}\mid \widetilde h_t)\) denote the conditional mutual information under this one-step-then-rollback rollout. We define the \textit{marginal downstream visual gain} (Figure~\ref{fig:formulation}) as \(\Delta_t(\pi):=I_\pi(V;Y_{\ge t}\mid \widetilde h_t)-I_{p_\theta}(V;Y_{\ge t}\mid \widetilde h_t)\). Thus, \(\Delta_t(\pi)\) rewards persistent downstream visual influence, not only transient visual evidence at the current token. {We do not compute \(\Delta_t(\pi)\) during training. It is the population-level target that motivates the surrogate introduced next.}

\textbf{Instance-level objective.}
Given a native trajectory {instance} \(y_{1:T}\), we jointly decide where to reshape and how to reshape. Let \(\rho\) be the budget ratio and \(\varepsilon\) be the local KL radius. We solve
\begin{equation}
\max_{\mathcal A,\pi}
\,
\sum_{t\in\mathcal A} \Delta_t(\pi)
\quad
\text{s.t.}\quad
\mathcal A\subseteq[T],\ |\mathcal A|\le \rho T,\ 
D_{\mathrm{KL}}\!\left(\pi(\cdot\mid v,\widetilde h_t)\ \|\ p_\theta(\cdot\mid v,\widetilde h_t)\right)\le \varepsilon
\ \ \forall t\in\mathcal A .
\label{eq:local_rate_max_new}
\end{equation}
{The objective makes the mechanism explicit: \(\mathcal A\) chooses where to intervene, \(\pi\) chooses how to intervene, and \(\Delta_t(\pi)\) rewards only visual influence that persists into future reasoning. The KL constraint keeps each deviation local and close to the native decoder. Thus, visual retention is cast as sparse instance-level reshaping: change only a few steps, but make each change propagate.}

\section{Methodology}
\label{sec:method}

\subsection{Reflection Anchors: Where Local Reshaping Matters}
\label{sec::reflection-anchor}

\textbf{Which steps matter more?}
Section~\ref{sec:prelim} reduces visual retention to a selection problem: choose a few decoding steps whose local reshaping can make the future continuation \(Y_{\ge t}\) more image-dependent. A useful selector should be causal and predictive: it should use only information available before step \(t\), yet identify changes that can affect later reasoning. Therefore, when scoring step \(t\), we only remove earlier selected positions. Let \(\mathcal A_{<t}=\mathcal A\cap[t-1]\) and \(\widetilde h_t=(x,M_t(y_{<t};\mathcal A_{<t}))\), where \(M_t\) masks only positions in \(\mathcal A_{<t}\). The current token \(y_t\) and the future suffix \(y_{\ge t}\) remain intact. Selection therefore asks a local intervention question: \emph{given the masked prefix \(\widetilde h_t\), is step \(t\) worth reshaping?}

\textbf{A theoretical guide.}
We next characterize what makes a step worth reshaping.  Fix a candidate set \(\mathcal A\) and a step \(t\). Let \(m_{t,v}=p_\theta(\cdot\mid v,\widetilde h_t)\) be the next-token distribution under the masked prefix. To quantify future visual relevance, we compare two suffix laws: \(p_{t,v}(y_{\ge t})=p_\theta(y_{\ge t}\mid v,h_t)\), the suffix law under the realized image \(v\), and \(p_{\mathcal M,t}(y_{\ge t})=\mathbb E_{V\mid \widetilde h_t}[p_\theta(y_{\ge t}\mid V,\widetilde h_t)]\), the masked-reference suffix law. The first standard suffix law keeps the image, while the second suffix law marginalizes over images compatible with the cleaned prefix and removes visual evidence leaked from earlier reshaping steps. 
In Appendix~\ref{app::theory}, we establish a lower bound on the optimal
downstream visual gain under a local KL constraint:
\begin{equation*}
\Delta_t^\star(A_{<t},\varepsilon)
\ge
\left[
c_1\sqrt{\varepsilon}\,
\mathbb{E}_{V\mid\tilde h_t}
\left[
\left(e^{H(m_{t,V})}-2\right)_+\cdot
\sqrt{
D_{\mathrm{KL}}(p_{t,V}\|p_{M,t})
}
\right]
-c_2\varepsilon
\right]_+,
\end{equation*}
where $(a)_+=\max\{a,0\}$ and $c_1,c_2>0$ depend
only on the regularity constants.
The bound captures a leading-order
$\sqrt{\varepsilon}$ contribution under sufficiently
small KL perturbations.

\textbf{From oracle scores to practical anchors.}
The bound gives two signals: \emph{local branching room}, through \(\mathcal H(m_{t,v})\), and \emph{downstream visual relevance}, through \(D_{\mathrm{KL}}(p_{t,v}\|p_{\mathcal M,t})\). Exact use is impractical: \(\mathcal H(m_{t,v})\) is masked-prefix dependent, so selecting earlier positions changes \(\widetilde h_t\) and forces recomputation of later next-token distributions; the suffix divergence is harder, requiring counterfactual continuation laws, suffix rollouts, and marginalization over images compatible with \(\widetilde h_t\). We therefore use the bound as a design principle, not as a scoring rule.

Although the bound is not directly usable, its message is clear: useful reshaping needs branching room. Entropy captures this room: low entropy means the decoder is nearly committed to one next token, while high entropy means several continuations remain plausible. We therefore approximate the masked-prefix entropy \(\mathcal H(m_{t,v})=\mathcal H(p_\theta(\cdot\mid v,\widetilde h_t))\) by the native-prefix entropy \(c_t\triangleq\mathcal H(p_\theta(\cdot\mid v,h_t))\). 
This fixes the prefix before selection, gives all scores in one native forward pass, and avoids circular dependence on \(\mathcal A\). 
Selecting the highest-entropy positions
maximizes the aggregate branching-room score under
a fixed intervention budget. This yields an efficient
and principled localization strategy, with the
subsequent optimization determining how visual
information propagates from the selected anchors.

\textbf{Reflection anchors.}
Given a native trajectory \(y_{1:T}\), we score each step by its native uncertainty $c_t$.
For a budget ratio \(\rho\in(0,1)\), let \(K_\rho=\lfloor \rho T\rfloor\), and define
$\mathcal A_\rho
=
\operatorname{TopK}_{t\in[T]}
\left\{
c_t
\right\}_{K_\rho},$
Here \(\operatorname{TopK}\{\cdot\}_{K}\) returns any set of \(K\) indices with the largest scores. This gives a one-pass selector: choose the steps where the native decoder still has branching room, then let training decide how to reshape them.

\subsection{RAPO: \underline{R}eflection-\underline{A}nchor \underline{P}olicy \underline{O}ptimization}
\label{sec:rapo}


\textbf{From the theoretical bound to an implementable principle.}
Section~\ref{sec::reflection-anchor} shows that local reshaping requires two ingredients: local branching room and downstream visual dependence. RAPO turns this characterization into a practical policy optimization procedure by closing two implementation gaps. First, the intractable selection of the sparse set \(\mathcal A\) in Eq.~\ref{eq:local_rate_max_new} is replaced by top-\(\rho\) native-entropy anchors, which provide a fixed single-pass proxy for uncertain branch points. Second, the suffix divergence \(D_{\mathrm{KL}}(p_{t,V}\|p_{\mathcal{M},t})\) in the bound, which requires counterfactual suffix laws and visual marginalization, is replaced by a finite-window chain-masked KL target. Thus, RAPO converts the intractable oracle objective into a practical surrogate: entropy determines where the policy-gradient update is applied, while chain-masked finite-window KL supplies a local proxy for downstream visual dependence.

\textbf{Anchor selection at training time.}
For each sampled trajectory \(y_{1:T}\), RAPO applies the reflection-anchor rule from Section~\ref{sec::reflection-anchor}. With \(h_t=(x,y_{<t})\), it computes \(c_t=\mathcal H(p_\theta(\cdot\mid V,h_t))\) and keeps the top-\(\rho\) fraction of positions:
$\mathcal A
    =
    \operatorname{TopK}_{t\in[T]}\{c_t\}_{\lfloor \rho T\rfloor}.$
Equivalently, \(w_t=\mathbf 1[t\in\mathcal A]\), so \(\sum_{t=1}^T w_t=\lfloor \rho T\rfloor\). All policy-gradient and visual-dependence terms below are evaluated only on \(t\in\mathcal A\).

\begin{wrapfigure}{r}{0.56\textwidth}
    \centering
     \vspace{-0.4cm}
    \includegraphics[width=0.95\linewidth]{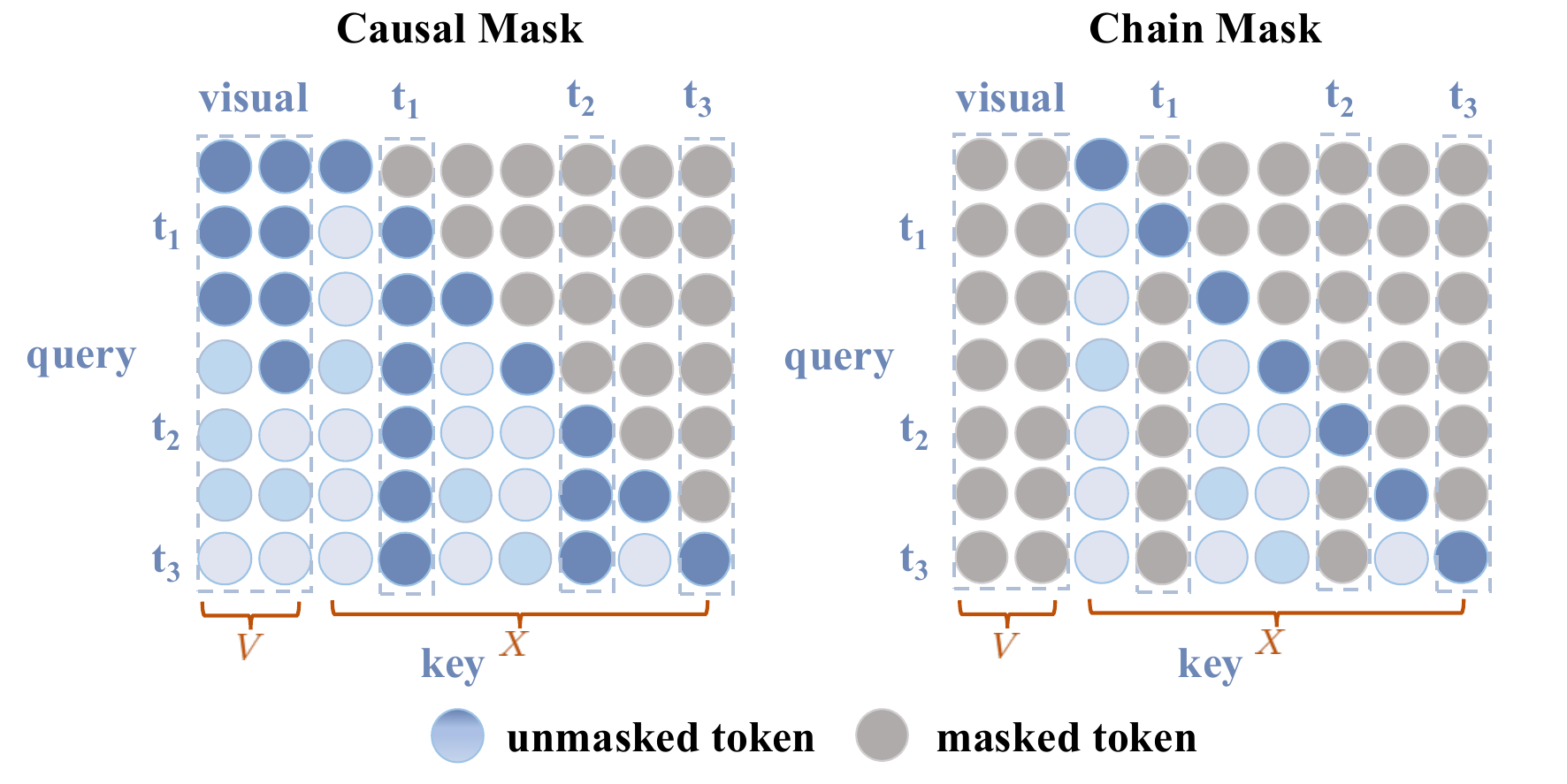}
     \vspace{-0.1cm}
    \caption{\small\textbf{Chain mask.} For an anchor $t_k \in \mathcal{A}$, the chain mask $\mathcal M_{t_k}$ enforces three constraints: (i) the query at $t_k$ cannot attend to the visual tokens $V$; (ii) it cannot attend to any preceding anchors $\{t_1,\dots,t_{k-1}\}$; and (iii) attention to all other textual positions remains unchanged.}
    \vspace{-0.4cm}
    \label{fig:chain_mask}
\end{wrapfigure}

\textbf{Chain-masked reference.}
The oracle reference \(p_{\mathcal M,t}\) marginalizes over images compatible with the masked history, which is not computable inside RL. RAPO uses a local counterfactual instead. Let \(\mathcal A=\{t_1,\ldots,t_K\}\) be the ordered  anchor set. For an anchor \(t_k\), the chain mask \(\mathcal M_{t_k}\) blocks visual access and earlier-anchor leakage, while preserving the remaining textual chain, as shown in Figure~\ref{fig:chain_mask}. 
Running the policy under this mask gives \(\pi_\theta^{\mathrm{mask}}(\cdot\mid h_{t_k};\mathcal M_{t_k})\), a vision-reduced reference under the same generated prefix. It plays the role of \(p_{\mathcal M,t}\), but is  tractable and compatible with policy optimization.

\textbf{Finite-horizon surrogate for downstream visual dependence.}
The bound uses the suffix divergence \(D_{\mathrm{KL}}(p_{t,v}\|p_{\mathcal M,t})\), which is a full-continuation quantity. RAPO turns it into a local, finite-window comparison. For an anchor \(t_k\), we keep one chain mask \(\mathcal M_{t_k}\). At each suffix position \(t_k+i\), the full visual law is \(\pi_\theta(\cdot\mid v,h_{t_k+i})\), while the vision-reduced reference is \(\pi_\theta^{\mathrm{mask}}(\cdot\mid h_{t_k+i};\mathcal M_{t_k})\). Their KL measures how much the next step still changes when visual access and earlier-anchor evidence are removed. Averaging this comparison over a window of length \(w'=\min(w,T-t_k+1)\), we define
$\overline{\mathbb D}_{\mathrm{KL}_{t_k}}(\theta)
=
\frac{1}{w'}
\sum_{i=0}^{w'-1}
\mathbb D_{\mathrm{KL}}
\left(
\pi_\theta(\cdot\mid v,h_{t_k+i})
\,\middle\|\,
\pi_\theta^{\mathrm{mask}}(\cdot\mid h_{t_k+i};\mathcal M_{t_k})
\right).$
This finite-window KL is the trainable surrogate for downstream visual dependence: it preserves the suffix spirit of the bound, but avoids counterfactual suffix laws and visual marginalization.

\textbf{Optimization objective.}
We now lift the above construction to the group-based GRPO objective by indexing trajectories with $i$. For each sampled trajectory $i$, GRPO assigns a scalar reward $R_i$, computed from the task-specific objective. Under the reflection-anchor mask, RAPO applies the standard GRPO clipped surrogate objective only at anchor positions and normalizes the loss by the anchor length. The resulting objective is
\begin{equation*}
\begin{aligned}
\mathcal{J}(\theta)
=
\mathbb{E}
\Bigg[
\frac{1}{G}
\sum_{i=1}^G
\frac{1}{|\mathcal{A}_i|}
\sum_{t\in \mathcal{A}_i}
\Big(
\mathcal{L}_{i,t}^{\text{clip}}(\theta)+
\gamma \omega^{\mathrm{br}}_{i,t}\sqrt{\overline{\mathbb{D}_{\mathrm{KL}}}_{i,t}(\theta)}-\beta\;\mathbb{D}_{\text{KL}}[\pi_{\theta}||\pi_{\text{ref}}]
\Big)
\Bigg].
\end{aligned}
\end{equation*}

 where $\mathcal{L}_{i,t}^{\text{clip}}(\theta)=
\min\left(
r_{i,t}(\theta)\hat A_{i,t},
\mathrm{clip}\left(r_{i,t}(\theta),1-\varepsilon_{\text{clip}},1+\varepsilon_{\text{clip}}\right)\hat A_{i,t}
\right)$,  $\mathcal{A}_i$ denotes the reflection-anchor set for trajectory $i$, $\omega^{\mathrm{br}}_{i,t}=\max(e^{\mathcal H(p_\theta(\cdot\mid V_i,h_{i,t}))}-2,0)$ denotes the branching-room weight, the importance ratio $r_{i,t}(\theta)=\frac{\pi_\theta(y_{i,t}| v,x,y_{i,<t})}{\pi_{\theta_{\mathrm{old}}}(y_{i,t}| v,x,y_{i,<t})} $, group-normalized advantage 
$\hat A_{i,t}
=
\frac{R_i-\frac{1}{G}\sum_{j=1}^G R_j}{\mathrm{std}(\{R_j\}_{j=1}^G)},$ and $\gamma>0$ denotes the visual-dependence surrogate coefficient.
All other components of GRPO, including rollouts, rewards, clipping, group normalization, and the KL penalty term  remain unchanged. During training, the masked reference and $\omega^{\mathrm{br}}_{i,t}$ are stop-gradient.


\textbf{Connection to prior work.}
RAPO combines the explicitness of visual restoration with the clean interface of policy optimization. Visual-restoration methods repair forgetting by revisiting or re-injecting vision during decoding~\cite{gao2025interleaved,chen2025mintcot,jian2025look}, while policy-optimization methods internalize grounding into normal generation~\cite{shao2024deepseekmathpushinglimitsmathematical,yu2025dapo,wang2025perceptionawarepolicyoptimizationmultimodal,huang2025spotlighttokenperceptionmultimodal}. RAPO keeps the latter interface but adds a local retention mechanism: entropy-selected anchors specify where a KL-bounded update has room to matter, and the branching-room-weighted chain-masked KL rewards whether visual dependence persists over downstream continuations. Thus, visual retention is trained as selective policy reshaping inside the native decoder, rather than added as inference-time repair or left as a global grounding reward.

\section{Experiments}
\label{sec:exp}
In this section, we answer the following research questions:

{\textbf{RQ1}: \textit{Method Effectiveness.} Can RAPO outperform baseline methods across model scales and diverse multimodal benchmarks?

\textbf{RQ2}:  \textit{Mechanism Study.}
Do reflection anchors correspond to visually sensitive
decision points, and does RAPO increase downstream visual-dependence signals
along generated trajectories?

\textbf{RQ3}: \textit{Ablation and Robustness Analysis.} How sensitive is RAPO to design choices such as KL window length and KL coefficient, and how robust is it across reasoning length and model architecture?
}

\begin{table*}[t]
\centering
\caption{\textbf{Main results (mean@8 acc \%).}
Performance comparison across reasoning-intensive and general-domain multimodal benchmarks.
The best results in each group are highlighted in \textbf{bold}.
The second best results in each group are highlighted in \underline{underlined}.
RAPO$_G$ and RAPO$_D$ denote RAPO built upon GRPO and DAPO, respectively.}
\footnotesize
\setlength{\tabcolsep}{1.0pt}
\renewcommand{\arraystretch}{0.95}
\begin{tabular}{clccccccc}
\toprule
\multirow{2}{*}{\textbf{Base Model}} 
& \multirow{2}{*}{\textbf{Model}} 
& \multicolumn{4}{c}{\textbf{Reasoning-Intensive}} 
& \multicolumn{2}{c}{\textbf{General-Domain}} 
& \multirow{2}{*}{\textbf{Avg.}} \\
\cmidrule(lr){3-6} \cmidrule(lr){7-8}
& & MathVision & MathVerse & EMMA & LogicVista & MMMU-Pro & RealWorldQA & \\

\midrule
\multirow{5}{*}{\makecell{\textbf{Qwen3-VL-8B}\\\textbf{Instruct}}}
& Base  & 41.74 & 50.21 & 38.34 & 62.28 & 37.59 & 67.20 & 49.56 \\
& + GRPO                & \underline{48.03} & 55.23 & \underline{41.34} & 59.51 & 48.50 & 67.66 & 53.38 ($\uparrow$ 3.82) \\
& + PAPO$_G$            & 43.75 & 54.57 & 39.00 & \underline{64.29} & \underline{50.00} & \underline{69.41} & 53.50 ($\uparrow$ 3.94) \\
& + VPPO$_G$            & \underline{48.03} & \underline{56.35} & 36.00 & \textbf{66.07} & 49.16 & 67.48 & \underline{53.85} ($\uparrow$ 4.29) \\
& \cellcolor{blue!8}\textbf{+ RAPO$_G$}   
& \cellcolor{blue!8}\textbf{48.27} & \cellcolor{blue!8}\textbf{56.72} & \cellcolor{blue!8}\textbf{42.50} & \cellcolor{blue!8}62.50 & \cellcolor{blue!8}\textbf{51.10} & \cellcolor{blue!8}\textbf{69.54} & \cellcolor{blue!8}\textbf{55.11} ($\uparrow$ 5.55) \\

\midrule

\multirow{5}{*}{\makecell{\textbf{Qwen3-VL-2B}\\\textbf{Instruct}}}
& Base   & 27.30 & 36.75 & 26.87 & 47.13 & 24.44 & 60.47 & 37.16 \\
& + GRPO                & 28.29 & 44.35 & 28.50 & \underline{48.71} & 30.61 & \underline{64.28} & 40.79 ($\uparrow$ 3.63) \\
& + PAPO$_G$            & 28.29 & 45.18 & 27.25 & 46.43 & 31.45 & 64.05 & 40.44 ($\uparrow$ 3.28) \\
& + VPPO$_G$            & \underline{30.59} & \underline{{45.43}} & \underline{28.75} & 45.76 & \underline{31.73} & 63.14 & \underline{40.90} ($\uparrow$ 3.74) \\
& \cellcolor{blue!8}\textbf{+ RAPO$_G$}   
& \cellcolor{blue!8}\textbf{33.55} & \cellcolor{blue!8}\textbf{46.13} & \cellcolor{blue!8}\textbf{31.25} & \cellcolor{blue!8}\textbf{49.55} & \cellcolor{blue!8}\textbf{32.20} & \cellcolor{blue!8}\textbf{65.36} & \cellcolor{blue!8}\textbf{43.01} ($\uparrow$ 5.85) \\

\midrule
\multirow{6}{*}{\makecell{\textbf{Qwen2.5-VL-7B}\\\textbf{Instruct}}}
& Base  & 21.05 & 29.44 & 23.25 & 42.19 & 21.22 & 54.51 & 31.94 \\
& + GRPO
& 25.33 & 43.40 & 28.00 & 44.42
& 32.83 & 56.21 & 38.37 ($\uparrow$ 6.43) \\
& MM-Eureka         & \underline{32.57} & 45.18 & {28.75} & 44.64 & 26.82 & \textbf{61.96} & 39.99 ($\uparrow$ 8.05) \\
& VL-Rethinker       & \textbf{34.21} & 45.30 & {28.75} & 44.20 & 34.10 & {59.22} & 40.96 ($\uparrow$ 9.02) \\
& R1-ShareVL         & 25.99 & {46.22} & \textbf{30.50} & \underline{46.88} & {34.45} & 56.47 & 40.09 ($\uparrow$ 8.15) \\
& Base+RAPO$_G$       & 27.96 & \underline{{47.34}} & \underline{29.00} & 45.31 & \underline{36.01} & \underline{60.26}& \underline{40.98} ($\uparrow$ 9.04) \\
& \cellcolor{blue!8}\textbf{Base+RAPO$_D$}     &\cellcolor{blue!8}30.26 & \cellcolor{blue!8}\textbf{47.59} & \cellcolor{blue!8}\textbf{30.50} & \cellcolor{blue!8}\textbf{47.13} &\cellcolor{blue!8}\textbf{36.04} &\cellcolor{blue!8}58.40& \cellcolor{blue!8}\textbf{41.65} ($\uparrow$ 9.71) \\

\midrule
\multirow{5}{*}{\makecell{\textbf{Qwen2-VL-7B}\\\textbf{Instruct}}}
& Base    & 15.79 & 10.25 & 11.75 & 13.17 & 11.73 & 31.83 & 15.75 \\
& + GRPO
& 21.05 & 31.85 & 24.25 & 35.27
& 27.11 & 61.44 & 33.50 ($\uparrow$ 17.75) \\
& TVC                   & 18.75 & 21.98 & 20.75 & \textbf{41.96} & {23.34} & {52.88} & {29.94} ($\uparrow$ 14.19) \\
& MINT-CoT              & 22.04 & 24.62 & 19.00 & 31.47 & 17.63 & 49.15 & 27.32 ($\uparrow$ 11.57) \\
& Base+RAPO$_G$       & \underline{22.37} & \underline{32.99} & \underline{25.25} & 37.72 & \underline{28.22} & \underline{64.67} & \underline{35.21 }($\uparrow$ 19.46) \\
& \cellcolor{blue!8}\textbf{Base+RAPO$_D$}     & \cellcolor{blue!8}\textbf{23.36}& \cellcolor{blue!8}\textbf{36.31} & \cellcolor{blue!8}\textbf{28.75} & \cellcolor{blue!8}\underline{39.73} & \cellcolor{blue!8}\textbf{29.39} & \cellcolor{blue!8}\textbf{65.01} & \cellcolor{blue!8}\textbf{37.09} ($\uparrow$ 21.34) \\

\bottomrule
\end{tabular}
\label{tab:main_results}
\vspace{-0.09cm}
\end{table*}

\subsection{Experimental Setup}

All  experimental details are provided in Appendix~\ref{app:training}.

\textbf{Models and Baselines.}
We adopt LVLMs from the Qwen family~\citep{wang2024qwen2vlenhancingvisionlanguagemodels,bai2025qwen25vltechnicalreport,bai2025qwen3vltechnicalreport} as our base models.
To ensure fair comparison, we group baselines by base model and restrict comparisons to methods built on the same base model.
\textbf{(i) Qwen3-VL-Instruct.}
We evaluate our method on the Qwen3-VL-2B-Instruct and 8B-Instruct models~\citep{bai2025qwen3vltechnicalreport}, and compare it with representative RL--based optimization methods, including GRPO~\citep{shao2024deepseekmathpushinglimitsmathematical}, PAPO~\citep{wang2025perceptionawarepolicyoptimizationmultimodal}, and VPPO~\citep{huang2025spotlighttokenperceptionmultimodal}.
All baselines in this group are retrained under their original RL settings for controlled comparison.
\textbf{(ii) Qwen2.5-VL-7B-Instruct.}
We compare against strong open-source reasoning models built on the Qwen2.5-VL-7B-Instruct, including MM-Eureka~\citep{meng2025mmeurekaexploringfrontiersmultimodal}, VL-Rethinker~\citep{wang2025vlrethinkerincentivizingselfreflectionvisionlanguage}, and R1-ShareVL~\citep{yao2025r1sharevlincentivizingreasoningcapability}.
These models are evaluated using the officially released checkpoints without additional training.
\textbf{(iii) Qwen2-VL-7B-Instruct.}
We further compare with prior methods built on the earlier Qwen2-VL-7B-Instruct, including TVC~\citep{sun2025mitigating} and MINT-CoT~\citep{chen2025mintcot}, which are trained with supervised fine-tuning followed by reinforcement learning.

\textbf{Training Settings.}
We implement RAPO on top of GRPO, and also implement PAPO and VPPO following the GRPO-based training settings as reported in their original papers.
All models are trained on the ViRL dataset~\citep{wang2025vlrethinkerincentivizingselfreflectionvisionlanguage}, which is designed for long CoT multimodal reasoning.
Across all experiments, we use a learning rate of $1\mathrm{e}{-6}$.
For group (i), models are trained for 100 RL steps with a rollout batch size of 256 and a PPO minibatch size of 64, which is sufficient to reach stable reward values.
For group (ii) and (iii), models are trained for
200 RL steps, respectively, using a rollout batch size of 384 and a PPO batch size of 128.
For all RAPO training across three groups, we set the KL coefficient $\gamma$ to $0.01$, the window size $w$ to 3 for Qwen3-VL-2B-Instruct and 1 for the others, and select reflection anchors as the top $20\%$ highest-entropy tokens.

\textbf{Evaluation Settings.}
We evaluate all models using the vLLM framework~\citep{kwon2023efficient} on six multimodal benchmarks spanning two categories.
(1) \textbf{Reasoning-intensive (Reas.)} benchmarks assess multimodal reasoning ability and include MathVision~\citep{wang2024measuring},
MathVerse~\citep{zhang2024mathversedoesmultimodalllm},
EMMA~\citep{hao2025mllmsreasonmultimodalityemma}
and LogicVista~\citep{xiao2024logicvistamultimodalllmlogical}.
(2) \textbf{General-domain (Gen.)} benchmarks evaluate overall ability across diverse domains and include
MMMU-Pro~\citep{yue2025mmmuprorobustmultidisciplinemultimodal}
and RealWorldQA\footnote{\url{https://huggingface.co/datasets/xai-org/RealworldQA}}.
Across all evaluations, we use a sampling temperature of 1.0, a maximum generation length of 8192 tokens
and sample 8 responses per input, reporting final performance using the mean@8 accuracy metric following prior work~\citep{wang2025perceptionawarepolicyoptimizationmultimodal,huang2025spotlighttokenperceptionmultimodal}.

\subsection{Main Results (RQ1)}
\textbf{RAPO achieves the best overall performance in both controlled and same-backbone reference comparisons.}
We evaluate RAPO across three experimental settings covering different base models and benchmark types, as summarized in Table~\ref{tab:main_results}.
Across all three experimental groups, RAPO achieves the best overall average performance across both reasoning-intensive and general-domain benchmarks. In the pure RL setting on Qwen3-VL-Instruct models, RAPO yields consistent gains over GRPO, PAPO, and VPPO; since Qwen3-VL is already a strong backbone, these gains are best viewed as controlled improvements under reduced headroom.
Moreover, RAPO shows improvements across both reasoning-intensive and general-domain benchmarks, indicating that the gains are not driven by overfitting to specific tasks. Under a compute-matched setting on Qwen3-VL-2B-Instruct, extending GRPO from 100 to 120 steps improves the average score from 40.79 to 41.45, but remains below RAPO-100 at 43.01. This suggests that the improvement is not explained solely by additional wall-clock training compute. Detailed wall-clock comparisons are provided in Appendix~\ref{app:compute_matched}.

\begin{wrapfigure}{r}{0.6\textwidth}
\vspace{-0.4cm}
    \centering
    \includegraphics[width=\linewidth]{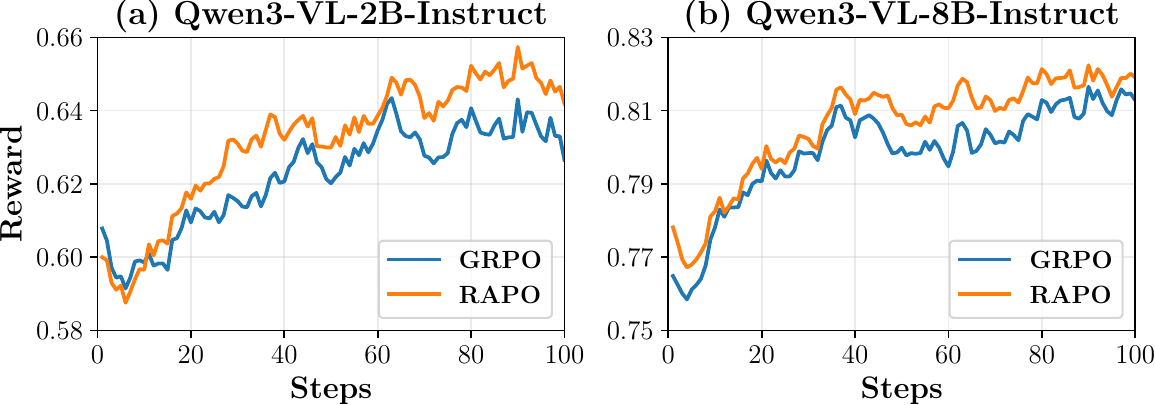}
    \caption{\textbf{Training dynamics.} Reward curves of GRPO and RAPO on Qwen3-VL-2B-Instruct and 8B during RL.}
    \vspace{-0.3cm}
    \label{fig:ACCRewards_of_training}
\end{wrapfigure}
 \textbf{RAPO exhibits faster and more stable reward optimization during training.} Figure~\ref{fig:ACCRewards_of_training} illustrates the accuracy reward trajectories on Qwen3-VL-2B-Instruct and Qwen3-VL-8B-Instruct. For both model scales, RAPO achieves faster reward growth and maintains higher reward values than GRPO across training steps, indicating improved optimization efficiency across different model sizes.




\subsection{Empirical Analysis of RAPO's Mechanism (RQ2)}\label{subsec:why}

In this subsection, we empirically examine why RAPO improves long-chain visual reasoning by analyzing how visual information is used, where local intervention is effective, and how its influence propagates during generation. 
 All analyses in this subsection are conducted
using the same data and model configuration.
We randomly sample 500 instances from the validation set of ViRL dataset and evaluate the Qwen3-VL-2B-Instruct model both before training and after GRPO and RAPO training from the main experiments.  Additional experimental details are provided in Appendix~\ref{app:why}.

Our mechanism study combines trajectory-level
visual-dependence diagnostics with targeted
interventions. The trajectory analysis measures
token-wise contrastive KL throughout generation,
complementing RAPO's anchor-focused,
finite-window training objective.

\textbf{Sustained downstream visual dependence distinguishes successful reasoning and is reinforced by RAPO.}
We analyze token-wise contrastive KL along generated trajectories, defined as the divergence between vision-conditioned and vision-masked continuation distributions. Figure~\ref{fig:attentionDistribution}a shows that correct trajectories retain larger residual KL values and stronger late-stage fluctuations after the initial decay, whereas incorrect trajectories collapse more rapidly toward near-zero visual dependence. In contrast, visual-token attention, despite being a common proxy for visual use, follows similar positional trends for correct and incorrect outputs and fails to distinguish reasoning outcomes (Figure~\ref{fig:attentionDistribution}b). Figure~\ref{fig:attentionDistribution}c further shows that RAPO raises the token-wise KL profile over GRPO and maintains stronger late-stage KL signals. These results support the hypothesis that RAPO improves long-chain reasoning by maintaining stronger contrastive visual-dependence signals over the downstream continuation; unlike prior RL training where visual dependence remains an implicit property of the learned policy, RAPO explicitly reinforces this trajectory-level dependence beyond standard GRPO.

\begin{figure*}[t]
    \centering
    \includegraphics[width=\linewidth]{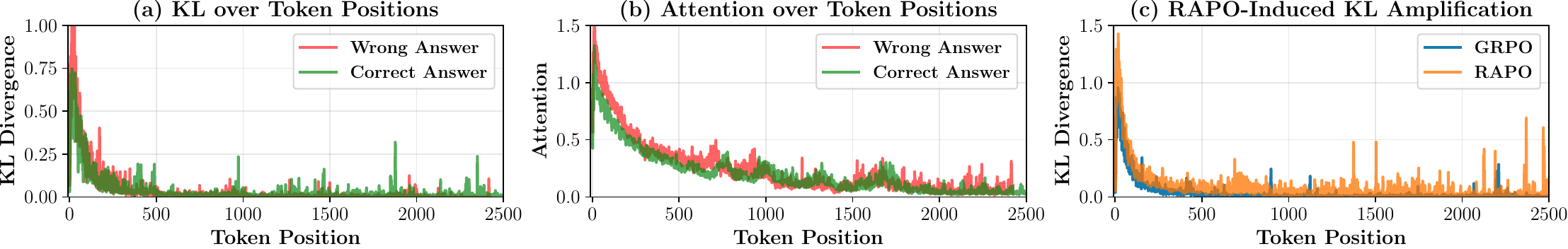}
    \caption{\textbf{Trajectory analysis.} \textbf{(a, b)} Visual dependence in correct vs. incorrect trajectories, measured by token-wise KL  and attention; \textbf{(c)} Token-wise KL  after GRPO and RAPO training. { The KL at token \(t\), \(\overline{\mathbb{D}_{\mathrm{KL}}}_t\), is the contrastive divergence between the distributions induced by vision-conditioned and vision-masked logits, with window length \(w = 1\); higher values indicate stronger visual dependence.
}}
    \label{fig:attentionDistribution}
\end{figure*}

\begin{figure*}[t]
    \newlength{\anchorpanelheight}
    \setlength{\anchorpanelheight}{3.2cm}
    \setlength{\tabcolsep}{2pt}
    \renewcommand{\arraystretch}{0.9}

    \begin{subfigure}[t]{0.45\textwidth}
        \vspace{0pt}
        \includegraphics[
            width=\linewidth,
            height=\anchorpanelheight,
            keepaspectratio
        ]{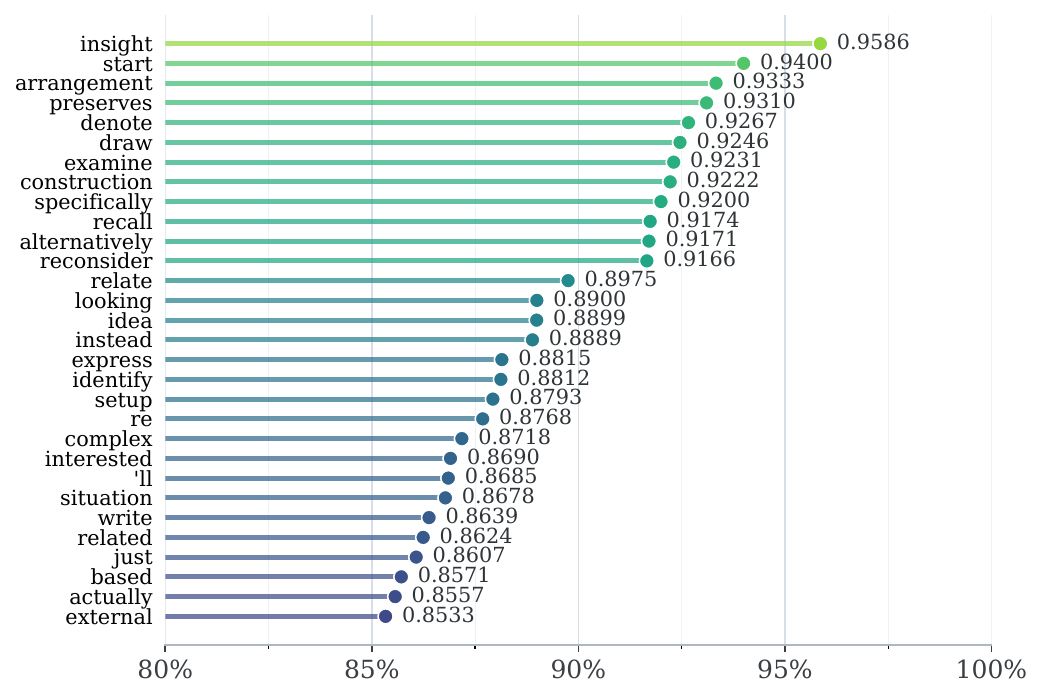}
        \captionsetup{
        width=0.45\linewidth,
        margin={-0.23\linewidth,0pt},
        justification=centering,
        singlelinecheck=false
      } 
        \caption{Anchor Token Concentration.}
        \label{fig:anchor_concentration}
    \end{subfigure}
    \hspace{-0.08\textwidth}
    \begin{subfigure}[t]{0.25\textwidth}
        \centering
        \vspace{0.7cm}
        \includegraphics[
            width=\linewidth,
            height=\anchorpanelheight,
            keepaspectratio
        ]{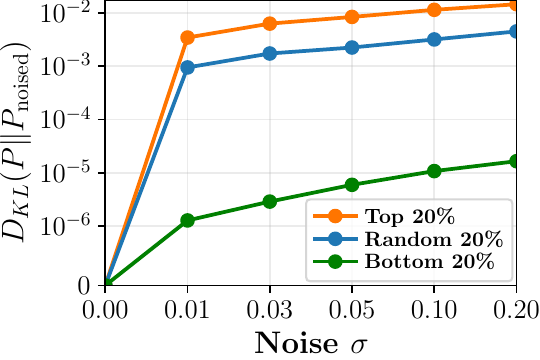}
        \vspace{-6pt}
        \caption{$D_{\mathrm{KL}}$ vs.\ Visual Noise.}
        \label{fig:kl_noise}
    \end{subfigure}
    \hfill
    \begin{subfigure}[t]{0.32\textwidth}
        \centering
        \vspace{0pt}
        \footnotesize
        \begin{minipage}[t][\anchorpanelheight][t]{\linewidth}
        \centering
        \resizebox{\linewidth}{!}{
        \begin{tabular}{llccc}
        \toprule
        \textbf{Axis} & \textbf{Method} & \textbf{Reas.} & \textbf{Gen.} & \textbf{Avg.} \\
        \midrule
        -- & GRPO & 37.46 & 47.45 & 40.79 \\
        \midrule
        \multirow{4}{*}{Quantity}
        & top-10\%  & 38.77 & 47.81 & 41.78 \\
        & \cellcolor{blue!8}top-20\%  & \cellcolor{blue!8}40.12 & \cellcolor{blue!8}48.78 & \cellcolor{blue!8}43.01 \\
        & top-50\%  & 37.90 & 49.59 & 41.80 \\
        & top-100\% & 29.19 & 45.84 & 34.75 \\
        \midrule
        \multirow{3}{*}{Position}
        & \cellcolor{blue!8}top-20\% & \cellcolor{blue!8}40.12 & \cellcolor{blue!8}48.78 & \cellcolor{blue!8}43.01 \\
        & bottom-20\% & 38.33 & 47.06 & 41.24 \\
        & random-20\% & 37.00 & 45.90 & 39.96 \\
        \bottomrule
        \end{tabular}
        }
        \end{minipage}
        \caption{Anchor-Selection Ablation.}
        \label{tab:aba_anchor}
    \end{subfigure}

    \caption{
    \textbf{Reflection anchor analysis.}
    \textbf{(a)} Reflection-anchor concentration, measured as the fraction of each token type's occurrences selected as top-entropy anchors.
    \textbf{(b)} KL divergence between original and noise-perturbed token distributions across entropy groups.
    \textbf{(c)} Analysis of anchor token selection strategies based on entropy ranking and token position.
    }
    \label{fig:anchor_analysis}
    \vspace{-0.2cm}
\end{figure*}

\textbf{Reflection anchors locate the positions whose local reshaping most plausibly carries visual evidence into future reasoning.}
We analyze the positions selected by the entropy-based reflection-anchor rule. 
Figure~\ref{fig:anchor_concentration} shows that high-entropy anchors are concentrated on tokens that control the next reasoning operation, including explicit inspection or revision verbs such as “examine,” “reconsider,” “identify,” and “draw,” but also discourse and relational operators such as “actually,” “instead,” “based,” and “specifically.” 
Such tokens often mark points where the model decides whether to inspect, revise, contrast, relate, or reparameterize the current reasoning state.
This supports the view that reflection anchors occur at operation-selecting positions, where local reshaping can alter how visual evidence is carried forward.
Figure~\ref{fig:kl_noise} further shows that high-entropy tokens consistently exhibit larger KL divergence under visual perturbations than random or low-entropy tokens. Correspondingly, restricting RAPO’s updates to the top 20\% high-entropy tokens outperforms random selection, low-entropy tokens, and alternative anchor set sizes (Table~\ref{tab:aba_anchor}).  Beyond these, Appendix~\ref{app:anchor_aba} evaluates noun-based, fixed-count, and outlier-threshold selectors. Together, these results indicate that RAPO intervenes at naturally occurring, decision-uncertain steps that are effective for local reshaping, in contrast to visual-injection methods that rely on fixed or predefined positions.


\begin{figure*}[t]
    \centering

    \begin{minipage}[t]{0.66\textwidth}
        \centering
        \vspace{-0.1cm}
        \includegraphics[width=\linewidth]{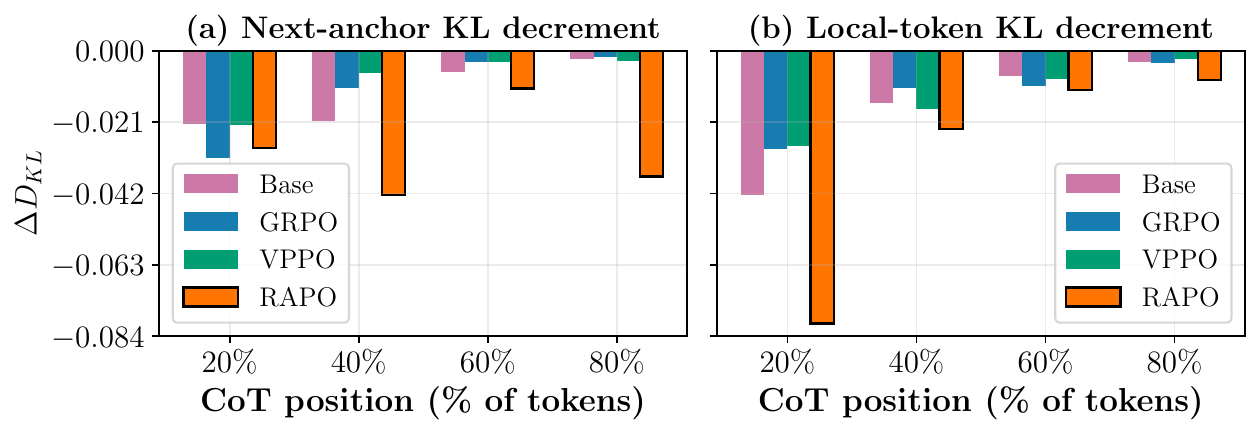}
    \end{minipage}
    \hfill
    \begin{minipage}[t]{0.30\textwidth}
        \centering
    \vspace{0pt}

    \scriptsize
    (c) Chain-mask ablation.
    \vspace{0.08cm}
    \setlength{\tabcolsep}{2.5pt}
    \renewcommand{\arraystretch}{0.95}

        \begin{tabular}{lcccc}
        \toprule
        \textbf{Mask} & \textbf{MathVis.} & \textbf{MathVer.} & \textbf{EMMA} \\
        \midrule
        Image & 26.64 & 44.80 & 27.00 \\
        \cellcolor{blue!8}Chain & \cellcolor{blue!8}\textbf{33.55} & \cellcolor{blue!8}\textbf{46.13} & \cellcolor{blue!8}\textbf{31.25}  \\
        \bottomrule
        \end{tabular}

        \vspace{0.08cm}

        \begin{tabular}{lccc}
        \toprule
        \textbf{Mask} & \textbf{LogicV.}& \textbf{MMMU-P.} & \textbf{RealWQ} \\
        \midrule
        Image & 47.32 &\textbf{32.79} & 64.02 \\
        \cellcolor{blue!8}Chain & \cellcolor{blue!8}\textbf{49.55}& \cellcolor{blue!8}32.20 & \cellcolor{blue!8}\textbf{65.36} \\
        \bottomrule
        \end{tabular}
    \end{minipage}

    \caption{
    \textbf{Propagation of visual influence.}
    \textbf{(a)} KL change at the next anchor token.
    \textbf{(b)} Average KL change at subsequent local tokens under window length \(w=3\), measured by visual-token masking at different CoT positions.
    \textbf{(c)} Comparison between image-level masking and the proposed chain mask.
    }
    \label{fig:delta_kl}
    \vspace{-0.2cm}
\end{figure*}

\textbf{RAPO strengthens downstream contrastive visual-dependence signals beyond reflection anchors.}
RAPO first selectively amplifies visual dependence at reflection anchors: after training, the average contrastive KL at the top 20\% entropy tokens rises from 0.30 to 0.55, while the bottom 20\% entropy tokens slightly decrease from 0.06 to 0.05. 
We then study whether this local amplification propagates by masking visual tokens at selected anchors and measuring the induced KL decrement at the next anchor and over subsequent local tokens. 
Figure~\ref{fig:delta_kl} shows that RAPO yields the largest downstream KL reductions in both cases, indicating that visual influence reinforced at reflection anchors persists into the future suffix rather than remaining a local token effect. 
A chain-mask ablation further shows that replacing chain masking with image-level masking lowers the average score from 43.01 to 40.43, indicating that structured anchor-chain masking provides a more effective propagation signal than blanket image removal.
Together with intervention-based causal diagnostics in Appendix~\ref{app:causal_sanity}, these results support the interpretation that RAPO reinforces not only local visual dependence at anchor positions, but also its propagation into the future suffix, yielding a more explicit propagation mechanism than prior RL methods whose visual influence remains localized or implicit.


\subsection{Ablation and Robustness Analysis (RQ3)}
We further evaluate RAPO through hyperparameter ablations and robustness checks. More experimental details and additional analysis are reported in Appendix~\ref{app:abaltion}. 

\textbf{Hyperparameter Ablations.} Table~\ref{tab:aba_gamma} shows that a small positive KL target coefficient \(\gamma\) improves over both GRPO and the \(\gamma=0\) variant, with the best result at \(\gamma=0.01\). Large $\gamma$ over-optimizes contrastive visual divergence and can dominate the task reward, causing the policy to increase visual sensitivity without preserving answer correctness.
Table~\ref{tab:aba_window} further studies the KL window length \(w\), which controls the temporal range of measured visual influence. A moderate window \(w=3\) performs best: \(w=1\) captures only short-term effects, while larger windows \(w\ge5\) introduce noise from weakly coupled distant continuations. We further evaluate two adaptive window strategies
based on entropy-gated stopping and inter-anchor
boundaries. The fixed three-token window achieves
the highest average accuracy (43.01), outperforming
both entropy-gated (41.32) and inter-anchor (42.04)
variants. This validates a compact propagation
horizon as an effective design choice for RAPO.
Detailed comparisons are provided in
Appendix~\ref{app:adaptive_windows}.

\textbf{Robustness.} We evaluate RAPO under more capable reasoning models and different model architectures. 
On Qwen3-VL-2B-Thinking, which generates substantially longer CoT sequences than Qwen3-VL-2B-Instruct (\(\sim\)6k vs.\ \(\sim\)2k tokens) with a maximum generation length of 16{,}384 tokens, RAPO remains effective under long-CoT regimes and shows pronounced gains on reasoning-intensive benchmarks (Table~\ref{tab:base_models_rasleng}). 
We further evaluate RAPO on Gemma3-4B-IT~\citep{gemmateam2025gemma3technicalreport}, whose architecture differs from the Qwen series used in the main results. 
RAPO outperforms GRPO (Table~\ref{tab:base_models_rasleng}), indicating that its effectiveness is not tied to a specific reasoning length or model family.

\begin{table*}[h]
    \caption{
    \textbf{Ablation and robustness analysis.}
    \textbf{(a)} Ablation on the KL coefficient \(\gamma\).
    \textbf{(b)} Ablation on KL window length \(w\), i.e., the temporal range of visual-influence estimation.
     \textbf{(c)} Robustness across longer reasoning trajectories and different base model architectures.
    }
    \renewcommand{\arraystretch}{1.0}
    \begin{subtable}[t]{0.34\linewidth}

        \caption{\textbf{Visual-Dependence Coefficient \(\gamma\)}}
        \vspace{-0.1cm}
        \footnotesize
        \label{tab:aba_gamma}
        \setlength{\tabcolsep}{2pt}
        \begin{tabular}{lccc}
        \toprule
        \textbf{Method} & \textbf{Reas.} & \textbf{Gen.} & \textbf{Avg.} \\
        \midrule
        GRPO & 37.46 & 47.45 & 40.79 \\
        \midrule
        RAPO$_G$ \((\gamma{=}0)\)     & 38.59 & 47.19 & 41.46 \\
        RAPO$_G$ \((\gamma{=}0.001)\) & 39.66 & 47.23 & 42.18 \\
        RAPO$_G$ \((\gamma{=}0.005)\) & 38.61 & 47.70 & 41.64 \\
        \cellcolor{blue!8}RAPO$_G$ \((\gamma{=}0.01)\)
            & \cellcolor{blue!8}40.12
            & \cellcolor{blue!8}48.78
            & \cellcolor{blue!8}43.01 \\
        RAPO$_G$ \((\gamma{=}0.02)\)  & 37.68 & 48.14 & 41.16 \\
        RAPO$_G$ \((\gamma{=}0.1)\)   & \multicolumn{3}{c}{Collapse} \\
        \bottomrule
        \end{tabular}
    \end{subtable}
    \hspace{0.01\textwidth}
    \renewcommand{\arraystretch}{1.15}
    \begin{subtable}[t]{0.30\linewidth}
        \centering
        \caption{\textbf{Window Length \(w\)}}
        \vspace{-0.1cm}
        \footnotesize
        \label{tab:aba_window}
        \setlength{\tabcolsep}{1.5pt}
        \begin{tabular}{lccc}
        \toprule
        \textbf{Method} & \textbf{Reas.} & \textbf{Gen.} & \textbf{Avg.} \\
        \midrule
        GRPO & 37.46 & 47.45 & 40.79 \\
        \midrule
        RAPO$_G$ \((w{=}1)\)  & 39.76 & 48.59 & 42.70 \\
        RAPO$_G$ \((w{=}2)\)  & 38.52 & 47.93 & 41.66 \\
        \cellcolor{blue!8}RAPO$_G$ \((w{=}3)\)
            & \cellcolor{blue!8}40.12
            & \cellcolor{blue!8}48.78
            & \cellcolor{blue!8}43.01 \\
        RAPO$_G$ \((w{=}5)\)  & 38.24 & 48.93 & 41.80 \\
        RAPO$_G$ \((w{=}10)\) & 39.64 & 48.06 & 42.45 \\
        \bottomrule
        \end{tabular}
    \end{subtable}
    \hspace{0.01\textwidth}
    \renewcommand{\arraystretch}{1.15}
    \begin{subtable}[t]{0.26\linewidth}
        \centering
        \caption{\textbf{Base Model}}
        \vspace{-0.1cm}
        \footnotesize
        \label{tab:base_models_rasleng}
        \setlength{\tabcolsep}{1.6pt}
        \begin{tabular}{lccc}
        \toprule
        \textbf{Model} & \textbf{Reas.} & \textbf{Gen.} & \textbf{Avg.} \\
        \midrule
        Thinking & 38.48 & 41.41 & 39.45 \\
        + GRPO   & 40.43 & 47.79 & 42.88 \\
        \cellcolor{blue!8}+ RAPO$_G$
            & \cellcolor{blue!8}44.04
            & \cellcolor{blue!8}47.82
            & \cellcolor{blue!8}45.30 \\
        \midrule
        Gemma3-4B-IT & 15.86 & 26.99 & 19.57 \\
        + GRPO       & 27.47 & 34.39 & 29.77 \\
        \cellcolor{blue!8}+ RAPO$_G$
            & \cellcolor{blue!8}28.62
            & \cellcolor{blue!8}33.67
            & \cellcolor{blue!8}30.30 \\
        \bottomrule
        \end{tabular}
    \end{subtable}

    \renewcommand{\arraystretch}{1.0}
    \vspace{-0.3cm}
    \label{tab:ablation_robustness}
\end{table*}

\newpage
\section{Conclusion}
\label{sec:conclusion}
We frame visual retention in long-chain multimodal reasoning as selective policy reshaping: local deviations from the native decoder should make future continuations remain visually dependent. This led to RAPO, an RL method that updates high-entropy reflection anchors and optimizes a chain-masked finite-window target for downstream visual dependence, without architectural changes or explicit visual-token injection at inference. Experiments across multiple LVLM backbones and multimodal benchmarks show that RAPO improves average performance over strong baselines, while mechanism analyses indicate that entropy anchors correspond to visually sensitive decision points and that RAPO strengthens downstream visual-dependence signals. The results point to propagation-aware visual-retention optimization as a useful direction for long-CoT LVLM reasoning. We discuss the current scope and future extensions of RAPO in Appendix~\ref{app:limitations}.

\section*{Acknowledgments}

This work was supported by the National Key R\&D Program of China
(2025YFF0516900 \& 2025YFF0516904), NSFC U25B2039 and National Natural Science Foundation of China (No.~62306179).





\newpage
\bibliography{ref}

\begin{thebibliography}{40}
\providecommand{\natexlab}[1]{#1}
\providecommand{\url}[1]{\texttt{#1}}
\expandafter\ifx\csname urlstyle\endcsname\relax
  \providecommand{\doi}[1]{doi: #1}\else
  \providecommand{\doi}{doi: \begingroup \urlstyle{rm}\Url}\fi

\bibitem[Wei et~al.(2023)Wei, Wang, Schuurmans, Bosma, Ichter, Xia, Chi, Le, and Zhou]{wei2023chainofthoughtpromptingelicitsreasoning}
Jason Wei, Xuezhi Wang, Dale Schuurmans, Maarten Bosma, Brian Ichter, Fei Xia, Ed~Chi, Quoc Le, and Denny Zhou.
\newblock Chain-of-thought prompting elicits reasoning in large language models, 2023.
\newblock URL \url{https://arxiv.org/abs/2201.11903}.

\bibitem[Kojima et~al.(2023)Kojima, Gu, Reid, Matsuo, and Iwasawa]{kojima2023largelanguagemodelszeroshot}
Takeshi Kojima, Shixiang~Shane Gu, Machel Reid, Yutaka Matsuo, and Yusuke Iwasawa.
\newblock Large language models are zero-shot reasoners, 2023.
\newblock URL \url{https://arxiv.org/abs/2205.11916}.

\bibitem[Guo et~al.(2025)Guo, Yang, Zhang, Song, Wang, Zhu, Xu, Zhang, Ma, Bi, Zhang, Yu, Wu, Wu, Gou, Shao, Li, Gao, Liu, Xue, Wang, Wu, Feng, Lu, Zhao, Deng, Ruan, Dai, Chen, Ji, Li, Lin, Dai, Luo, Hao, Chen, Li, Zhang, Xu, Ding, Gao, Qu, Li, Guo, Li, Chen, Yuan, Tu, Qiu, Li, Cai, Ni, Liang, Chen, Dong, Hu, You, Gao, Guan, Huang, Yu, Wang, Zhang, Zhao, Wang, Zhang, Xu, Xia, Zhang, Zhang, Tang, Zhou, Li, Wang, Li, Tian, Huang, Zhang, Wang, Chen, Du, Ge, Zhang, Pan, Wang, Chen, Jin, Chen, Lu, Zhou, Chen, Ye, Wang, Yu, Zhou, Pan, Li, Zhou, Wu, Yun, Pei, Sun, Wang, Zeng, Liu, Liang, Gao, Yu, Zhang, Xiao, An, Liu, Wang, Chen, Nie, Cheng, Liu, Xie, Liu, Yang, Li, Su, Lin, Li, Jin, Shen, Chen, Sun, Wang, Song, Zhou, Wang, Shan, Li, Wang, Wei, Zhang, Xu, Li, Zhao, Sun, Wang, Yu, Zhang, Shi, Xiong, He, Piao, Wang, Tan, Ma, Liu, Guo, Ou, Wang, Gong, Zou, He, Xiong, Luo, You, Liu, Zhou, Zhu, Huang, Li, Zheng, Zhu, Ma, Tang, Zha, Yan, Ren, Ren, Sha, Fu, Xu, Xie, Zhang, Hao, Ma, Yan, Wu, Gu, Zhu, Liu, Li, Xie, Song,
  Pan, Huang, Xu, Zhang, and Zhang]{Guo_2025}
Daya Guo, Dejian Yang, Haowei Zhang, Junxiao Song, Peiyi Wang, Qihao Zhu, Runxin Xu, Ruoyu Zhang, Shirong Ma, Xiao Bi, Xiaokang Zhang, Xingkai Yu, Yu~Wu, Z.~F. Wu, Zhibin Gou, Zhihong Shao, Zhuoshu Li, Ziyi Gao, Aixin Liu, Bing Xue, Bingxuan Wang, Bochao Wu, Bei Feng, Chengda Lu, Chenggang Zhao, Chengqi Deng, Chong Ruan, Damai Dai, Deli Chen, Dongjie Ji, Erhang Li, Fangyun Lin, Fucong Dai, Fuli Luo, Guangbo Hao, Guanting Chen, Guowei Li, H.~Zhang, Hanwei Xu, Honghui Ding, Huazuo Gao, Hui Qu, Hui Li, Jianzhong Guo, Jiashi Li, Jingchang Chen, Jingyang Yuan, Jinhao Tu, Junjie Qiu, Junlong Li, J.~L. Cai, Jiaqi Ni, Jian Liang, Jin Chen, Kai Dong, Kai Hu, Kaichao You, Kaige Gao, Kang Guan, Kexin Huang, Kuai Yu, Lean Wang, Lecong Zhang, Liang Zhao, Litong Wang, Liyue Zhang, Lei Xu, Leyi Xia, Mingchuan Zhang, Minghua Zhang, Minghui Tang, Mingxu Zhou, Meng Li, Miaojun Wang, Mingming Li, Ning Tian, Panpan Huang, Peng Zhang, Qiancheng Wang, Qinyu Chen, Qiushi Du, Ruiqi Ge, Ruisong Zhang, Ruizhe Pan, Runji Wang, R.~J.
  Chen, R.~L. Jin, Ruyi Chen, Shanghao Lu, Shangyan Zhou, Shanhuang Chen, Shengfeng Ye, Shiyu Wang, Shuiping Yu, Shunfeng Zhou, Shuting Pan, S.~S. Li, Shuang Zhou, Shaoqing Wu, Tao Yun, Tian Pei, Tianyu Sun, T.~Wang, Wangding Zeng, Wen Liu, Wenfeng Liang, Wenjun Gao, Wenqin Yu, Wentao Zhang, W.~L. Xiao, Wei An, Xiaodong Liu, Xiaohan Wang, Xiaokang Chen, Xiaotao Nie, Xin Cheng, Xin Liu, Xin Xie, Xingchao Liu, Xinyu Yang, Xinyuan Li, Xuecheng Su, Xuheng Lin, X.~Q. Li, Xiangyue Jin, Xiaojin Shen, Xiaosha Chen, Xiaowen Sun, Xiaoxiang Wang, Xinnan Song, Xinyi Zhou, Xianzu Wang, Xinxia Shan, Y.~K. Li, Y.~Q. Wang, Y.~X. Wei, Yang Zhang, Yanhong Xu, Yao Li, Yao Zhao, Yaofeng Sun, Yaohui Wang, Yi~Yu, Yichao Zhang, Yifan Shi, Yiliang Xiong, Ying He, Yishi Piao, Yisong Wang, Yixuan Tan, Yiyang Ma, Yiyuan Liu, Yongqiang Guo, Yuan Ou, Yuduan Wang, Yue Gong, Yuheng Zou, Yujia He, Yunfan Xiong, Yuxiang Luo, Yuxiang You, Yuxuan Liu, Yuyang Zhou, Y.~X. Zhu, Yanping Huang, Yaohui Li, Yi~Zheng, Yuchen Zhu, Yunxian Ma, Ying
  Tang, Yukun Zha, Yuting Yan, Z.~Z. Ren, Zehui Ren, Zhangli Sha, Zhe Fu, Zhean Xu, Zhenda Xie, Zhengyan Zhang, Zhewen Hao, Zhicheng Ma, Zhigang Yan, Zhiyu Wu, Zihui Gu, Zijia Zhu, Zijun Liu, Zilin Li, Ziwei Xie, Ziyang Song, Zizheng Pan, Zhen Huang, Zhipeng Xu, Zhongyu Zhang, and Zhen Zhang.
\newblock Deepseek-r1 incentivizes reasoning in llms through reinforcement learning.
\newblock \emph{Nature}, 645\penalty0 (8081):\penalty0 633–638, September 2025.
\newblock ISSN 1476-4687.
\newblock \doi{10.1038/s41586-025-09422-z}.
\newblock URL \url{http://dx.doi.org/10.1038/s41586-025-09422-z}.

\bibitem[Team et~al.(2025{\natexlab{a}})Team, Bai, Bao, Chen, Chen, Chen, Chen, Chen, Chen, Chen, Chen, Cui, Ding, Dong, Du, Du, Du, Du, Fan, Feng, Fu, Gao, Gao, Gao, Gao, Gu, Guan, Guo, Guo, Hu, Hao, He, He, He, Hong, Hu, Hu, Huang, Huang, Huang, Jiang, Jiang, Jin, Kang, Lai, Li, Li, Li, Li, Li, Li, Li, Li, Li, Lin, Lin, Lin, Liu, Liu, Liu, Liu, Liu, Liu, Liu, Liu, Liu, Liu, Liu, Liu, Liu, Liu, Liu, Lu, Lu, Ma, Ma, Ma, Mao, Mei, Men, Miao, Pan, Peng, Qin, Qu, Shang, Shi, Shi, Song, Su, Su, Sun, Sung, Tang, Tao, Teng, Wang, Wang, Wang, Wang, Wang, Wang, Wang, Wang, Wang, Wang, Wang, Wang, Wang, Wang, Wang, Wang, Wang, Wei, Wei, Wu, Wu, Wu, Xiao, Xie, Xiong, Xu, Xu, Xu, Xu, Xu, Xu, Xu, Xu, Xu, Xu, Yan, Yan, Yang, Yang, Yang, Yang, Yang, Yao, Yao, Ye, Ye, Yin, Yu, Yuan, Yuan, Yuan, Zhan, Zhang, Zhang, Zhang, Zhang, Zhang, Zhang, Zhang, Zhang, Zhang, Zhang, Zhang, Zhao, Zhao, Zheng, Zheng, Zhou, Zhou, Zhou, Zhu, Zhuang, and Zu]{kimiteam2025kimik2openagentic}
Kimi Team, Yifan Bai, Yiping Bao, Guanduo Chen, Jiahao Chen, Ningxin Chen, Ruijue Chen, Yanru Chen, Yuankun Chen, Yutian Chen, Zhuofu Chen, Jialei Cui, Hao Ding, Mengnan Dong, Angang Du, Chenzhuang Du, Dikang Du, Yulun Du, Yu~Fan, Yichen Feng, Kelin Fu, Bofei Gao, Hongcheng Gao, Peizhong Gao, Tong Gao, Xinran Gu, Longyu Guan, Haiqing Guo, Jianhang Guo, Hao Hu, Xiaoru Hao, Tianhong He, Weiran He, Wenyang He, Chao Hong, Yangyang Hu, Zhenxing Hu, Weixiao Huang, Zhiqi Huang, Zihao Huang, Tao Jiang, Zhejun Jiang, Xinyi Jin, Yongsheng Kang, Guokun Lai, Cheng Li, Fang Li, Haoyang Li, Ming Li, Wentao Li, Yanhao Li, Yiwei Li, Zhaowei Li, Zheming Li, Hongzhan Lin, Xiaohan Lin, Zongyu Lin, Chengyin Liu, Chenyu Liu, Hongzhang Liu, Jingyuan Liu, Junqi Liu, Liang Liu, Shaowei Liu, T.~Y. Liu, Tianwei Liu, Weizhou Liu, Yangyang Liu, Yibo Liu, Yiping Liu, Yue Liu, Zhengying Liu, Enzhe Lu, Lijun Lu, Shengling Ma, Xinyu Ma, Yingwei Ma, Shaoguang Mao, Jie Mei, Xin Men, Yibo Miao, Siyuan Pan, Yebo Peng, Ruoyu Qin, Bowen Qu, Zeyu
  Shang, Lidong Shi, Shengyuan Shi, Feifan Song, Jianlin Su, Zhengyuan Su, Xinjie Sun, Flood Sung, Heyi Tang, Jiawen Tao, Qifeng Teng, Chensi Wang, Dinglu Wang, Feng Wang, Haiming Wang, Jianzhou Wang, Jiaxing Wang, Jinhong Wang, Shengjie Wang, Shuyi Wang, Yao Wang, Yejie Wang, Yiqin Wang, Yuxin Wang, Yuzhi Wang, Zhaoji Wang, Zhengtao Wang, Zhexu Wang, Chu Wei, Qianqian Wei, Wenhao Wu, Xingzhe Wu, Yuxin Wu, Chenjun Xiao, Xiaotong Xie, Weimin Xiong, Boyu Xu, Jing Xu, Jinjing Xu, L.~H. Xu, Lin Xu, Suting Xu, Weixin Xu, Xinran Xu, Yangchuan Xu, Ziyao Xu, Junjie Yan, Yuzi Yan, Xiaofei Yang, Ying Yang, Zhen Yang, Zhilin Yang, Zonghan Yang, Haotian Yao, Xingcheng Yao, Wenjie Ye, Zhuorui Ye, Bohong Yin, Longhui Yu, Enming Yuan, Hongbang Yuan, Mengjie Yuan, Haobing Zhan, Dehao Zhang, Hao Zhang, Wanlu Zhang, Xiaobin Zhang, Yangkun Zhang, Yizhi Zhang, Yongting Zhang, Yu~Zhang, Yutao Zhang, Yutong Zhang, Zheng Zhang, Haotian Zhao, Yikai Zhao, Huabin Zheng, Shaojie Zheng, Jianren Zhou, Xinyu Zhou, Zaida Zhou, Zhen Zhu,
  Weiyu Zhuang, and Xinxing Zu.
\newblock Kimi k2: Open agentic intelligence, 2025{\natexlab{a}}.
\newblock URL \url{https://arxiv.org/abs/2507.20534}.

\bibitem[Yang et~al.(2025)Yang, Li, Yang, Zhang, Hui, Zheng, Yu, Gao, Huang, Lv, Zheng, Liu, Zhou, Huang, Hu, Ge, Wei, Lin, Tang, Yang, Tu, Zhang, Yang, Yang, Zhou, Zhou, Lin, Dang, Bao, Yang, Yu, Deng, Li, Xue, Li, Zhang, Wang, Zhu, Men, Gao, Liu, Luo, Li, Tang, Yin, Ren, Wang, Zhang, Ren, Fan, Su, Zhang, Zhang, Wan, Liu, Wang, Cui, Zhang, Zhou, and Qiu]{yang2025qwen3technicalreport}
An~Yang, Anfeng Li, Baosong Yang, Beichen Zhang, Binyuan Hui, Bo~Zheng, Bowen Yu, Chang Gao, Chengen Huang, Chenxu Lv, Chujie Zheng, Dayiheng Liu, Fan Zhou, Fei Huang, Feng Hu, Hao Ge, Haoran Wei, Huan Lin, Jialong Tang, Jian Yang, Jianhong Tu, Jianwei Zhang, Jianxin Yang, Jiaxi Yang, Jing Zhou, Jingren Zhou, Junyang Lin, Kai Dang, Keqin Bao, Kexin Yang, Le~Yu, Lianghao Deng, Mei Li, Mingfeng Xue, Mingze Li, Pei Zhang, Peng Wang, Qin Zhu, Rui Men, Ruize Gao, Shixuan Liu, Shuang Luo, Tianhao Li, Tianyi Tang, Wenbiao Yin, Xingzhang Ren, Xinyu Wang, Xinyu Zhang, Xuancheng Ren, Yang Fan, Yang Su, Yichang Zhang, Yinger Zhang, Yu~Wan, Yuqiong Liu, Zekun Wang, Zeyu Cui, Zhenru Zhang, Zhipeng Zhou, and Zihan Qiu.
\newblock Qwen3 technical report, 2025.
\newblock URL \url{https://arxiv.org/abs/2505.09388}.

\bibitem[Team et~al.(2025{\natexlab{b}})Team, Du, Yin, Xing, Qu, Wang, Chen, Zhang, Du, Wei, Wang, Zhang, Du, Wang, Yuan, Lu, Li, Sung, Wei, Lai, Zhu, Ding, Hu, Yang, Zhang, Wu, Yao, Lu, Wang, Gao, Zheng, Li, Su, Wang, Deng, Qiu, Xie, Wang, Liu, Yan, Ouyang, Chen, Sui, Yu, Dong, Dong, Xu, Cheng, Gu, Zhou, Liu, Cao, Yu, Song, Bai, Song, He, Huang, Xu, Yuan, Yao, Wu, Li, Zu, Zhou, Wang, Charles, Zhong, Li, Hu, Chen, Wang, Liu, Miao, Qin, Chen, Bao, Wang, Kang, Liu, Dong, Du, Wu, Wang, Yan, Zhou, Li, Jiang, Zhang, Yang, Huang, Huang, Zhao, Chen, and Lin]{kimiteam2025kimivltechnicalreport}
Kimi Team, Angang Du, Bohong Yin, Bowei Xing, Bowen Qu, Bowen Wang, Cheng Chen, Chenlin Zhang, Chenzhuang Du, Chu Wei, Congcong Wang, Dehao Zhang, Dikang Du, Dongliang Wang, Enming Yuan, Enzhe Lu, Fang Li, Flood Sung, Guangda Wei, Guokun Lai, Han Zhu, Hao Ding, Hao Hu, Hao Yang, Hao Zhang, Haoning Wu, Haotian Yao, Haoyu Lu, Heng Wang, Hongcheng Gao, Huabin Zheng, Jiaming Li, Jianlin Su, Jianzhou Wang, Jiaqi Deng, Jiezhong Qiu, Jin Xie, Jinhong Wang, Jingyuan Liu, Junjie Yan, Kun Ouyang, Liang Chen, Lin Sui, Longhui Yu, Mengfan Dong, Mengnan Dong, Nuo Xu, Pengyu Cheng, Qizheng Gu, Runjie Zhou, Shaowei Liu, Sihan Cao, Tao Yu, Tianhui Song, Tongtong Bai, Wei Song, Weiran He, Weixiao Huang, Weixin Xu, Xiaokun Yuan, Xingcheng Yao, Xingzhe Wu, Xinhao Li, Xinxing Zu, Xinyu Zhou, Xinyuan Wang, Y.~Charles, Yan Zhong, Yang Li, Yangyang Hu, Yanru Chen, Yejie Wang, Yibo Liu, Yibo Miao, Yidao Qin, Yimin Chen, Yiping Bao, Yiqin Wang, Yongsheng Kang, Yuanxin Liu, Yuhao Dong, Yulun Du, Yuxin Wu, Yuzhi Wang, Yuzi Yan, Zaida
  Zhou, Zhaowei Li, Zhejun Jiang, Zheng Zhang, Zhilin Yang, Zhiqi Huang, Zihao Huang, Zijia Zhao, Ziwei Chen, and Zongyu Lin.
\newblock Kimi-vl technical report, 2025{\natexlab{b}}.
\newblock URL \url{https://arxiv.org/abs/2504.07491}.

\bibitem[Bai et~al.(2025{\natexlab{a}})Bai, Cai, Chen, Chen, Chen, Cheng, Deng, Ding, Gao, Ge, Ge, Guo, Huang, Huang, Huang, Hui, Jiang, Li, Li, Li, Li, Lin, Lin, Liu, Liu, Liu, Liu, Liu, Liu, Lu, Luo, Lv, Men, Meng, Ren, Ren, Song, Sun, Tang, Tu, Wan, Wang, Wang, Wang, Wang, Xie, Xu, Xu, Xu, Yang, Yang, Yang, Yang, Yu, Zhang, Zhang, Zhang, Zheng, Zhong, Zhou, Zhou, Zhou, Zhu, and Zhu]{bai2025qwen3vltechnicalreport}
Shuai Bai, Yuxuan Cai, Ruizhe Chen, Keqin Chen, Xionghui Chen, Zesen Cheng, Lianghao Deng, Wei Ding, Chang Gao, Chunjiang Ge, Wenbin Ge, Zhifang Guo, Qidong Huang, Jie Huang, Fei Huang, Binyuan Hui, Shutong Jiang, Zhaohai Li, Mingsheng Li, Mei Li, Kaixin Li, Zicheng Lin, Junyang Lin, Xuejing Liu, Jiawei Liu, Chenglong Liu, Yang Liu, Dayiheng Liu, Shixuan Liu, Dunjie Lu, Ruilin Luo, Chenxu Lv, Rui Men, Lingchen Meng, Xuancheng Ren, Xingzhang Ren, Sibo Song, Yuchong Sun, Jun Tang, Jianhong Tu, Jianqiang Wan, Peng Wang, Pengfei Wang, Qiuyue Wang, Yuxuan Wang, Tianbao Xie, Yiheng Xu, Haiyang Xu, Jin Xu, Zhibo Yang, Mingkun Yang, Jianxin Yang, An~Yang, Bowen Yu, Fei Zhang, Hang Zhang, Xi~Zhang, Bo~Zheng, Humen Zhong, Jingren Zhou, Fan Zhou, Jing Zhou, Yuanzhi Zhu, and Ke~Zhu.
\newblock Qwen3-vl technical report, 2025{\natexlab{a}}.
\newblock URL \url{https://arxiv.org/abs/2511.21631}.

\bibitem[Huang et~al.(2025{\natexlab{a}})Huang, Jia, Zhai, Cao, Ye, Zhao, Xu, Hu, and Lin]{huang2025visionr1incentivizingreasoningcapability}
Wenxuan Huang, Bohan Jia, Zijie Zhai, Shaosheng Cao, Zheyu Ye, Fei Zhao, Zhe Xu, Yao Hu, and Shaohui Lin.
\newblock Vision-r1: Incentivizing reasoning capability in multimodal large language models.
\newblock \emph{arXiv preprint arXiv:2503.06749}, 2025{\natexlab{a}}.

\bibitem[Meng et~al.(2025)Meng, Du, Liu, Zhou, Lu, Fu, Han, Shi, Wang, He, Zhang, Luo, Qiao, Zhang, and Shao]{meng2025mmeurekaexploringfrontiersmultimodal}
Fanqing Meng, Lingxiao Du, Zongkai Liu, Zhixiang Zhou, Quanfeng Lu, Daocheng Fu, Tiancheng Han, Botian Shi, Wenhai Wang, Junjun He, Kaipeng Zhang, Ping Luo, Yu~Qiao, Qiaosheng Zhang, and Wenqi Shao.
\newblock Mm-eureka: Exploring the frontiers of multimodal reasoning with rule-based reinforcement learning, 2025.
\newblock URL \url{https://arxiv.org/abs/2503.07365}.

\bibitem[Li et~al.(2025)Li, Tang, Li, Zhang, Vulic, and S{\o}gaard]{li2025lostembeddingsinformationloss}
Wenyan Li, Raphael Tang, Chengzu Li, Caiqi Zhang, Ivan Vulic, and Anders S{\o}gaard.
\newblock Lost in embeddings: Information loss in vision-language models.
\newblock \emph{arXiv preprint arXiv:2509.11986}, 2, 2025.

\bibitem[Sun et~al.(2025{\natexlab{a}})Sun, Sun, Peng, and Ye]{sun2025mitigating}
Hai-Long Sun, Zhun Sun, Houwen Peng, and Han-Jia Ye.
\newblock Mitigating visual forgetting via take-along visual conditioning for multi-modal long cot reasoning.
\newblock \emph{arXiv preprint arXiv:2503.13360}, 2025{\natexlab{a}}.

\bibitem[Masry et~al.(2025)Masry, Rodriguez, Zhang, Wang, Wang, Feizi, Suresh, Puri, Jian, Noël, Madhusudhan, Pedersoli, Liu, Chapados, Bengio, Hoque, Pal, Laradji, Vazquez, Taslakian, Gella, and Rajeswar]{masry2025alignvlmbridgingvisionlanguage}
Ahmed Masry, Juan~A. Rodriguez, Tianyu Zhang, Suyuchen Wang, Chao Wang, Aarash Feizi, Akshay~Kalkunte Suresh, Abhay Puri, Xiangru Jian, Pierre-André Noël, Sathwik~Tejaswi Madhusudhan, Marco Pedersoli, Bang Liu, Nicolas Chapados, Yoshua Bengio, Enamul Hoque, Christopher Pal, Issam~H. Laradji, David Vazquez, Perouz Taslakian, Spandana Gella, and Sai Rajeswar.
\newblock Alignvlm: Bridging vision and language latent spaces for multimodal document understanding, 2025.
\newblock URL \url{https://arxiv.org/abs/2502.01341}.

\bibitem[Jian et~al.(2025)Jian, Wu, Sun, Wang, Ren, and Zhang]{jian2025look}
Pu~Jian, Junhong Wu, Wei Sun, Chen Wang, Shuo Ren, and Jiajun Zhang.
\newblock Look again, think slowly: Enhancing visual reflection in vision-language models.
\newblock In \emph{Proceedings of the 2025 Conference on Empirical Methods in Natural Language Processing}, pages 9262--9281, 2025.

\bibitem[Chu et~al.(2025)Chu, Chen, Wang, Tan, Huang, Lv, Mo, and Li]{chu2025qwen}
Xu~Chu, Xinrong Chen, Guanyu Wang, Zhijie Tan, Kui Huang, Wenyu Lv, Tong Mo, and Weiping Li.
\newblock Qwen look again: Guiding vision-language reasoning models to re-attention visual information.
\newblock \emph{arXiv preprint arXiv:2505.23558}, 2025.

\bibitem[Chen et~al.(2025)Chen, Zhang, Jiang, Zhou, Yan, Lin, and Li]{chen2025mintcot}
Xinyan Chen, Renrui Zhang, Dongzhi Jiang, Aojun Zhou, Shilin Yan, Weifeng Lin, and Hongsheng Li.
\newblock {MINT}-cot: Enabling interleaved visual tokens in mathematical chain-of-thought reasoning.
\newblock In \emph{The Thirty-ninth Annual Conference on Neural Information Processing Systems}, 2025.
\newblock URL \url{https://openreview.net/forum?id=vMpvtSmtXY}.

\bibitem[Gao et~al.(2025)Gao, Li, Cao, and Li]{gao2025interleaved}
Jun Gao, Yongqi Li, Ziqiang Cao, and Wenjie Li.
\newblock Interleaved-modal chain-of-thought.
\newblock In \emph{Proceedings of the Computer Vision and Pattern Recognition Conference}, pages 19520--19529, 2025.

\bibitem[Chung et~al.(2025)Chung, Kim, Kim, Lee, Kim, and Yu]{chung2025v1learningpointvisual}
Jiwan Chung, Junhyeok Kim, Siyeol Kim, Jaeyoung Lee, Min~Soo Kim, and Youngjae Yu.
\newblock v1: Learning to point visual tokens for multimodal grounded reasoning, 2025.
\newblock URL \url{https://arxiv.org/abs/2505.18842}.

\bibitem[Qin et~al.(2025)Qin, Wei, Ge, Kallidromitis, Fu, Darrell, and Wang]{qin2025chain}
Yiming Qin, Bomin Wei, Jiaxin Ge, Konstantinos Kallidromitis, Stephanie Fu, Trevor Darrell, and XuDong Wang.
\newblock Chain-of-visual-thought: Teaching vlms to see and think better with continuous visual tokens.
\newblock \emph{arXiv preprint arXiv:2511.19418}, 2025.

\bibitem[Cheng et~al.(2026)Cheng, Chen, Xu, WANG, Wang, Fei, Wang, Wang, Chen, Che, and Qin]{cheng2026visual}
Zihui Cheng, Qiguang Chen, Xiao Xu, Jiaqi WANG, Weiyun Wang, Hao Fei, Yidong Wang, Alex~Jinpeng Wang, Zhi Chen, Wanxiang Che, and Libo Qin.
\newblock Visual thoughts: A unified perspective of understanding multimodal chain-of-thought.
\newblock In \emph{The Thirty-ninth Annual Conference on Neural Information Processing Systems}, 2026.
\newblock URL \url{https://openreview.net/forum?id=xPcKmKSEis}.

\bibitem[Sun et~al.(2025{\natexlab{b}})Sun, Hua, Wang, Luo, Dianat, Rabbani, Rao, and Tao]{sun2025latentchainofthoughtvisualreasoning}
Guohao Sun, Hang Hua, Jian Wang, Jiebo Luo, Sohail Dianat, Majid Rabbani, Raghuveer Rao, and Zhiqiang Tao.
\newblock Latent chain-of-thought for visual reasoning, 2025{\natexlab{b}}.
\newblock URL \url{https://arxiv.org/abs/2510.23925}.

\bibitem[Wang et~al.(2025{\natexlab{a}})Wang, Guo, Stoica, Xu, Wang, Ha, Chen, Chen, Yan, Huang, and Ji]{wang2025perceptionawarepolicyoptimizationmultimodal}
Zhenhailong Wang, Xuehang Guo, Sofia Stoica, Haiyang Xu, Hongru Wang, Hyeonjeong Ha, Xiusi Chen, Yangyi Chen, Ming Yan, Fei Huang, and Heng Ji.
\newblock Perception-aware policy optimization for multimodal reasoning, 2025{\natexlab{a}}.
\newblock URL \url{https://arxiv.org/abs/2507.06448}.

\bibitem[Huang et~al.(2025{\natexlab{b}})Huang, Qu, Li, Luo, He, Liu, and Cheng]{huang2025spotlighttokenperceptionmultimodal}
Siyuan Huang, Xiaoye Qu, Yafu Li, Yun Luo, Zefeng He, Daizong Liu, and Yu~Cheng.
\newblock Spotlight on token perception for multimodal reinforcement learning, 2025{\natexlab{b}}.
\newblock URL \url{https://arxiv.org/abs/2510.09285}.

\bibitem[Li et~al.(2026)Li, Kuang, Zhang, Cao, Liu, Hou, and Cheng]{li2026rethinkingtokenlevelpolicyoptimization}
Yunheng Li, Hangyi Kuang, Hengrui Zhang, Jiangxia Cao, Zhaojie Liu, Qibin Hou, and Ming-Ming Cheng.
\newblock Rethinking token-level policy optimization for multimodal chain-of-thought, 2026.
\newblock URL \url{https://arxiv.org/abs/2603.22847}.

\bibitem[Shao et~al.(2024)Shao, Wang, Zhu, Xu, Song, Bi, Zhang, Zhang, Li, Wu, et~al.]{shao2024deepseekmathpushinglimitsmathematical}
Zhihong Shao, Peiyi Wang, Qihao Zhu, Runxin Xu, Junxiao Song, Xiao Bi, Haowei Zhang, Mingchuan Zhang, YK~Li, Yang Wu, et~al.
\newblock Deepseekmath: Pushing the limits of mathematical reasoning in open language models.
\newblock \emph{arXiv preprint arXiv:2402.03300}, 2024.

\bibitem[Yu et~al.(2025)Yu, Zhang, Zhu, Yuan, Zuo, YuYue, Dai, Fan, Liu, Liu, Liu, Liu, Lin, Lin, Ma, Sheng, Tong, Zhang, Zhang, Zhang, Zhang, Zhu, Zhu, Chen, Chen, Wang, Yu, Song, Wei, Zhou, Liu, Ma, Zhang, Yan, Wu, and Wang]{yu2025dapo}
Qiying Yu, Zheng Zhang, Ruofei Zhu, Yufeng Yuan, Xiaochen Zuo, YuYue, Weinan Dai, Tiantian Fan, Gaohong Liu, Juncai Liu, LingJun Liu, Xin Liu, Haibin Lin, Zhiqi Lin, Bole Ma, Guangming Sheng, Yuxuan Tong, Chi Zhang, Mofan Zhang, Ru~Zhang, Wang Zhang, Hang Zhu, Jinhua Zhu, Jiaze Chen, Jiangjie Chen, Chengyi Wang, Hongli Yu, Yuxuan Song, Xiangpeng Wei, Hao Zhou, Jingjing Liu, Wei-Ying Ma, Ya-Qin Zhang, Lin Yan, Yonghui Wu, and Mingxuan Wang.
\newblock {DAPO}: An open-source {LLM} reinforcement learning system at scale.
\newblock In \emph{The Thirty-ninth Annual Conference on Neural Information Processing Systems}, 2025.
\newblock URL \url{https://openreview.net/forum?id=2a36EMSSTp}.

\bibitem[Wang et~al.(2025{\natexlab{b}})Wang, Qu, Huang, Chu, Lin, and Chen]{wang2025vlrethinkerincentivizingselfreflectionvisionlanguage}
Haozhe Wang, Chao Qu, Zuming Huang, Wei Chu, Fangzhen Lin, and Wenhu Chen.
\newblock Vl-rethinker: Incentivizing self-reflection of vision-language models with reinforcement learning, 2025{\natexlab{b}}.
\newblock URL \url{https://arxiv.org/abs/2504.08837}.

\bibitem[Wan et~al.(2025)Wan, Dou, Liu, Zhang, Cui, Zhao, Shen, Xiong, Xin, Jiang, et~al.]{wan2025srpo}
Zhongwei Wan, Zhihao Dou, Che Liu, Yu~Zhang, Dongfei Cui, Qinjian Zhao, Hui Shen, Jing Xiong, Yi~Xin, Yifan Jiang, et~al.
\newblock Srpo: Enhancing multimodal llm reasoning via reflection-aware reinforcement learning.
\newblock \emph{arXiv preprint arXiv:2506.01713}, 2025.

\bibitem[Yao et~al.(2024)Yao, Huang, Wu, Zhang, Wang, Liu, Wang, Song, Feng, Shen, et~al.]{yao2024mulberry}
Huanjin Yao, Jiaxing Huang, Wenhao Wu, Jingyi Zhang, Yibo Wang, Shunyu Liu, Yingjie Wang, Yuxin Song, Haocheng Feng, Li~Shen, et~al.
\newblock Mulberry: Empowering mllm with o1-like reasoning and reflection via collective monte carlo tree search.
\newblock \emph{arXiv preprint arXiv:2412.18319}, 2024.

\bibitem[Zhang et~al.(2025)Zhang, Ding, Zhang, Zhang, Li, Li, Wang, Wu, Ji, Shen, et~al.]{zhang2025perl}
Yizhen Zhang, Yang Ding, Shuoshuo Zhang, Xinchen Zhang, Haoling Li, Zhong-zhi Li, Peijie Wang, Jie Wu, Lei Ji, Yelong Shen, et~al.
\newblock Perl: Permutation-enhanced reinforcement learning for interleaved vision-language reasoning.
\newblock \emph{arXiv preprint arXiv:2506.14907}, 2025.

\bibitem[Tian et~al.(2026)Tian, Zou, Yang, He, Waschkowski, Wesemann, Tu, and Zhang]{tian2026more}
Xinyu Tian, Shu Zou, Zhaoyuan Yang, Mengqi He, Fabian Waschkowski, Lukas Wesemann, Peter~Henry Tu, and Jing Zhang.
\newblock More thought, less accuracy? on the dual nature of reasoning in vision-language models.
\newblock In \emph{The Fourteenth International Conference on Learning Representations}, 2026.
\newblock URL \url{https://openreview.net/forum?id=XpL5eqjCjF}.

\bibitem[Wang et~al.(2024{\natexlab{a}})Wang, Bai, Tan, Wang, Fan, Bai, Chen, Liu, Wang, Ge, Fan, Dang, Du, Ren, Men, Liu, Zhou, Zhou, and Lin]{wang2024qwen2vlenhancingvisionlanguagemodels}
Peng Wang, Shuai Bai, Sinan Tan, Shijie Wang, Zhihao Fan, Jinze Bai, Keqin Chen, Xuejing Liu, Jialin Wang, Wenbin Ge, Yang Fan, Kai Dang, Mengfei Du, Xuancheng Ren, Rui Men, Dayiheng Liu, Chang Zhou, Jingren Zhou, and Junyang Lin.
\newblock Qwen2-vl: Enhancing vision-language model's perception of the world at any resolution, 2024{\natexlab{a}}.
\newblock URL \url{https://arxiv.org/abs/2409.12191}.

\bibitem[Bai et~al.(2025{\natexlab{b}})Bai, Chen, Liu, Wang, Ge, Song, Dang, Wang, Wang, Tang, Zhong, Zhu, Yang, Li, Wan, Wang, Ding, Fu, Xu, Ye, Zhang, Xie, Cheng, Zhang, Yang, Xu, and Lin]{bai2025qwen25vltechnicalreport}
Shuai Bai, Keqin Chen, Xuejing Liu, Jialin Wang, Wenbin Ge, Sibo Song, Kai Dang, Peng Wang, Shijie Wang, Jun Tang, Humen Zhong, Yuanzhi Zhu, Mingkun Yang, Zhaohai Li, Jianqiang Wan, Pengfei Wang, Wei Ding, Zheren Fu, Yiheng Xu, Jiabo Ye, Xi~Zhang, Tianbao Xie, Zesen Cheng, Hang Zhang, Zhibo Yang, Haiyang Xu, and Junyang Lin.
\newblock Qwen2.5-vl technical report, 2025{\natexlab{b}}.
\newblock URL \url{https://arxiv.org/abs/2502.13923}.

\bibitem[Yao et~al.(2025)Yao, Yin, Zhang, Yang, Wang, Wu, Su, Shen, Qiu, Tao, and Huang]{yao2025r1sharevlincentivizingreasoningcapability}
Huanjin Yao, Qixiang Yin, Jingyi Zhang, Min Yang, Yibo Wang, Wenhao Wu, Fei Su, Li~Shen, Minghui Qiu, Dacheng Tao, and Jiaxing Huang.
\newblock R1-sharevl: Incentivizing reasoning capability of multimodal large language models via share-grpo, 2025.
\newblock URL \url{https://arxiv.org/abs/2505.16673}.

\bibitem[Kwon et~al.(2023)Kwon, Li, Zhuang, Sheng, Zheng, Yu, Gonzalez, Zhang, and Stoica]{kwon2023efficient}
Woosuk Kwon, Zhuohan Li, Siyuan Zhuang, Ying Sheng, Lianmin Zheng, Cody~Hao Yu, Joseph Gonzalez, Hao Zhang, and Ion Stoica.
\newblock Efficient memory management for large language model serving with pagedattention.
\newblock In \emph{Proceedings of the 29th symposium on operating systems principles}, pages 611--626, 2023.

\bibitem[Wang et~al.(2024{\natexlab{b}})Wang, Pan, Shi, Lu, Ren, Zhou, Zhan, and Li]{wang2024measuring}
Ke~Wang, Junting Pan, Weikang Shi, Zimu Lu, Houxing Ren, Aojun Zhou, Mingjie Zhan, and Hongsheng Li.
\newblock Measuring multimodal mathematical reasoning with math-vision dataset.
\newblock \emph{Advances in Neural Information Processing Systems}, 37:\penalty0 95095--95169, 2024{\natexlab{b}}.

\bibitem[Zhang et~al.(2024)Zhang, Jiang, Zhang, Lin, Guo, Qiu, Zhou, Lu, Chang, Gao, and Li]{zhang2024mathversedoesmultimodalllm}
Renrui Zhang, Dongzhi Jiang, Yichi Zhang, Haokun Lin, Ziyu Guo, Pengshuo Qiu, Aojun Zhou, Pan Lu, Kai-Wei Chang, Peng Gao, and Hongsheng Li.
\newblock Mathverse: Does your multi-modal llm truly see the diagrams in visual math problems?, 2024.
\newblock URL \url{https://arxiv.org/abs/2403.14624}.

\bibitem[Hao et~al.(2025)Hao, Gu, Wang, Li, Yang, Wang, and Cheng]{hao2025mllmsreasonmultimodalityemma}
Yunzhuo Hao, Jiawei Gu, Huichen~Will Wang, Linjie Li, Zhengyuan Yang, Lijuan Wang, and Yu~Cheng.
\newblock Can mllms reason in multimodality? emma: An enhanced multimodal reasoning benchmark, 2025.
\newblock URL \url{https://arxiv.org/abs/2501.05444}.

\bibitem[Xiao et~al.(2024)Xiao, Sun, Liu, and Wang]{xiao2024logicvistamultimodalllmlogical}
Yijia Xiao, Edward Sun, Tianyu Liu, and Wei Wang.
\newblock Logicvista: Multimodal llm logical reasoning benchmark in visual contexts, 2024.
\newblock URL \url{https://arxiv.org/abs/2407.04973}.

\bibitem[Yue et~al.(2025)Yue, Zheng, Ni, Wang, Zhang, Tong, Sun, Yu, Zhang, Sun, Su, Chen, and Neubig]{yue2025mmmuprorobustmultidisciplinemultimodal}
Xiang Yue, Tianyu Zheng, Yuansheng Ni, Yubo Wang, Kai Zhang, Shengbang Tong, Yuxuan Sun, Botao Yu, Ge~Zhang, Huan Sun, Yu~Su, Wenhu Chen, and Graham Neubig.
\newblock Mmmu-pro: A more robust multi-discipline multimodal understanding benchmark, 2025.
\newblock URL \url{https://arxiv.org/abs/2409.02813}.

\bibitem[Team et~al.(2025{\natexlab{c}})Team, Kamath, Ferret, Pathak, Vieillard, Merhej, Perrin, Matejovicova, Ramé, Rivière, Rouillard, Mesnard, Cideron, bastien Grill, Ramos, Yvinec, Casbon, Pot, Penchev, Liu, Visin, Kenealy, Beyer, Zhai, Tsitsulin, Busa-Fekete, Feng, Sachdeva, Coleman, Gao, Mustafa, Barr, Parisotto, Tian, Eyal, Cherry, Peter, Sinopalnikov, Bhupatiraju, Agarwal, Kazemi, Malkin, Kumar, Vilar, Brusilovsky, Luo, Steiner, Friesen, Sharma, Sharma, Gilady, Goedeckemeyer, Saade, Feng, Kolesnikov, Bendebury, Abdagic, Vadi, György, Pinto, Das, Bapna, Miech, Yang, Paterson, Shenoy, Chakrabarti, Piot, Wu, Shahriari, Petrini, Chen, Lan, Choquette-Choo, Carey, Brick, Deutsch, Eisenbud, Cattle, Cheng, Paparas, Sreepathihalli, Reid, Tran, Zelle, Noland, Huizenga, Kharitonov, Liu, Amirkhanyan, Cameron, Hashemi, Klimczak-Plucińska, Singh, Mehta, Lehri, Hazimeh, Ballantyne, Szpektor, Nardini, Pouget-Abadie, Chan, Stanton, Wieting, Lai, Orbay, Fernandez, Newlan, yeong Ji, Singh, Black, Yu, Hui,
  Vodrahalli, Greff, Qiu, Valentine, Coelho, Ritter, Hoffman, Watson, Chaturvedi, Moynihan, Ma, Babar, Noy, Byrd, Roy, Momchev, Chauhan, Sachdeva, Bunyan, Botarda, Caron, Rubenstein, Culliton, Schmid, Sessa, Xu, Stanczyk, Tafti, Shivanna, Wu, Pan, Rokni, Willoughby, Vallu, Mullins, Jerome, Smoot, Girgin, Iqbal, Reddy, Sheth, Põder, Bhatnagar, Panyam, Eiger, Zhang, Liu, Yacovone, Liechty, Kalra, Evci, Misra, Roseberry, Feinberg, Kolesnikov, Han, Kwon, Chen, Chow, Zhu, Wei, Egyed, Cotruta, Giang, Kirk, Rao, Black, Babar, Lo, Moreira, Martins, Sanseviero, Gonzalez, Gleicher, Warkentin, Mirrokni, Senter, Collins, Barral, Ghahramani, Hadsell, Matias, Sculley, Petrov, Fiedel, Shazeer, Vinyals, Dean, Hassabis, Kavukcuoglu, Farabet, Buchatskaya, Alayrac, Anil, Dmitry, Lepikhin, Borgeaud, Bachem, Joulin, Andreev, Hardin, Dadashi, and Hussenot]{gemmateam2025gemma3technicalreport}
Gemma Team, Aishwarya Kamath, Johan Ferret, Shreya Pathak, Nino Vieillard, Ramona Merhej, Sarah Perrin, Tatiana Matejovicova, Alexandre Ramé, Morgane Rivière, Louis Rouillard, Thomas Mesnard, Geoffrey Cideron, Jean bastien Grill, Sabela Ramos, Edouard Yvinec, Michelle Casbon, Etienne Pot, Ivo Penchev, Gaël Liu, Francesco Visin, Kathleen Kenealy, Lucas Beyer, Xiaohai Zhai, Anton Tsitsulin, Robert Busa-Fekete, Alex Feng, Noveen Sachdeva, Benjamin Coleman, Yi~Gao, Basil Mustafa, Iain Barr, Emilio Parisotto, David Tian, Matan Eyal, Colin Cherry, Jan-Thorsten Peter, Danila Sinopalnikov, Surya Bhupatiraju, Rishabh Agarwal, Mehran Kazemi, Dan Malkin, Ravin Kumar, David Vilar, Idan Brusilovsky, Jiaming Luo, Andreas Steiner, Abe Friesen, Abhanshu Sharma, Abheesht Sharma, Adi~Mayrav Gilady, Adrian Goedeckemeyer, Alaa Saade, Alex Feng, Alexander Kolesnikov, Alexei Bendebury, Alvin Abdagic, Amit Vadi, András György, André~Susano Pinto, Anil Das, Ankur Bapna, Antoine Miech, Antoine Yang, Antonia Paterson, Ashish
  Shenoy, Ayan Chakrabarti, Bilal Piot, Bo~Wu, Bobak Shahriari, Bryce Petrini, Charlie Chen, Charline~Le Lan, Christopher~A. Choquette-Choo, CJ~Carey, Cormac Brick, Daniel Deutsch, Danielle Eisenbud, Dee Cattle, Derek Cheng, Dimitris Paparas, Divyashree~Shivakumar Sreepathihalli, Doug Reid, Dustin Tran, Dustin Zelle, Eric Noland, Erwin Huizenga, Eugene Kharitonov, Frederick Liu, Gagik Amirkhanyan, Glenn Cameron, Hadi Hashemi, Hanna Klimczak-Plucińska, Harman Singh, Harsh Mehta, Harshal~Tushar Lehri, Hussein Hazimeh, Ian Ballantyne, Idan Szpektor, Ivan Nardini, Jean Pouget-Abadie, Jetha Chan, Joe Stanton, John Wieting, Jonathan Lai, Jordi Orbay, Joseph Fernandez, Josh Newlan, Ju~yeong Ji, Jyotinder Singh, Kat Black, Kathy Yu, Kevin Hui, Kiran Vodrahalli, Klaus Greff, Linhai Qiu, Marcella Valentine, Marina Coelho, Marvin Ritter, Matt Hoffman, Matthew Watson, Mayank Chaturvedi, Michael Moynihan, Min Ma, Nabila Babar, Natasha Noy, Nathan Byrd, Nick Roy, Nikola Momchev, Nilay Chauhan, Noveen Sachdeva, Oskar
  Bunyan, Pankil Botarda, Paul Caron, Paul~Kishan Rubenstein, Phil Culliton, Philipp Schmid, Pier~Giuseppe Sessa, Pingmei Xu, Piotr Stanczyk, Pouya Tafti, Rakesh Shivanna, Renjie Wu, Renke Pan, Reza Rokni, Rob Willoughby, Rohith Vallu, Ryan Mullins, Sammy Jerome, Sara Smoot, Sertan Girgin, Shariq Iqbal, Shashir Reddy, Shruti Sheth, Siim Põder, Sijal Bhatnagar, Sindhu~Raghuram Panyam, Sivan Eiger, Susan Zhang, Tianqi Liu, Trevor Yacovone, Tyler Liechty, Uday Kalra, Utku Evci, Vedant Misra, Vincent Roseberry, Vlad Feinberg, Vlad Kolesnikov, Woohyun Han, Woosuk Kwon, Xi~Chen, Yinlam Chow, Yuvein Zhu, Zichuan Wei, Zoltan Egyed, Victor Cotruta, Minh Giang, Phoebe Kirk, Anand Rao, Kat Black, Nabila Babar, Jessica Lo, Erica Moreira, Luiz~Gustavo Martins, Omar Sanseviero, Lucas Gonzalez, Zach Gleicher, Tris Warkentin, Vahab Mirrokni, Evan Senter, Eli Collins, Joelle Barral, Zoubin Ghahramani, Raia Hadsell, Yossi Matias, D.~Sculley, Slav Petrov, Noah Fiedel, Noam Shazeer, Oriol Vinyals, Jeff Dean, Demis Hassabis,
  Koray Kavukcuoglu, Clement Farabet, Elena Buchatskaya, Jean-Baptiste Alayrac, Rohan Anil, Dmitry, Lepikhin, Sebastian Borgeaud, Olivier Bachem, Armand Joulin, Alek Andreev, Cassidy Hardin, Robert Dadashi, and Léonard Hussenot.
\newblock Gemma 3 technical report, 2025{\natexlab{c}}.
\newblock URL \url{https://arxiv.org/abs/2503.19786}.

\end{thebibliography}
\bibliographystyle{unsrtnat}

\appendix

\section{Related Works}
\label{app:relatedWorksApp}

\subsection{Mitigating Visual Forgetting in long CoT Reasoning}
A key challenge in applying Chain-of-Thought (CoT) reasoning to MLLMs is \textit{visual forgetting}, where reliance on initial visual inputs progressively decays as reasoning chains lengthen, often leading to ungrounded hallucinations. To address this, recent work explores modality interleaving, explicit visual re-attention during inference, and visual self-reflection mechanisms.

One line of work seeks to prevent visual attention decay by structurally embedding visual evidence throughout the reasoning process. Interleaved-Modal Chain-of-Thought (ICoT)~\cite{gao2025interleaved} interleaves paired visual-textual rationales with intermediate reasoning steps, maintaining access to visual evidence during multi-hop inference. Similarly, MINT-CoT~\cite{chen2025mintcot} employs interleaved image-text sequences with dynamic visual token selection to support long-context reasoning. From a representation perspective, Chain-of-Visual-Thought (CoVT)~\cite{qin2025chain} enables reasoning in a continuous visual token space trained with dense perceptual supervision, thereby strengthening spatial and geometric grounding over extended reasoning horizons. Visual Thoughts~\citep{cheng2026visual} provides a unified view of multimodal CoT by showing that both textual and interleaved MCoT can benefit from intermediate ``visual thoughts,'' which convey image information through different expression forms and act as intermediaries between the input image and deeper reasoning layers. 
Latent Chain-of-Thought~\citep{sun2025latentchainofthoughtvisualreasoning} instead treats visual reasoning as inference over latent CoT variables, using amortized variational inference and sparse token-level rewards to encourage diverse, high-likelihood latent rationales. 

Another line of approaches emphasize inference-time re-attention to visual inputs. TVC~\cite{sun2025mitigating} mitigates visual forgetting by reintroducing image conditioning at critical reasoning stages and dynamically pruning redundant visual tokens, effectively preserving visual context over long chains. In parallel, v1~\cite{chung2025v1learningpointvisual} proposes a \textit{point-and-copy} mechanism that selects relevant image patches and injects their representations into the reasoning sequence, ensuring continued access to visual evidence.

Recent studies further incorporate explicit visual self-reflection and reinforcement learning to address hallucinations induced by visual forgetting. Qwen-LA~\cite{chu2025qwen} introduces a vision-text reflection process and optimizes the timing of visual re-examination via reinforcement learning, complemented by visual token routing. Reflection-V~\cite{jian2025look} constructs vision-centered reasoning data for cold-start training and employs a visual-attention-based reward model to encourage reasoning trajectories that remain grounded in image evidence. 

\subsection{Multimodal Reinforcement Learning}
Recent work has extended reinforcement learning (RL) techniques from large language models (LLMs) to multimodal large language models (MLLMs) to improve complex reasoning. Foundational policy optimization relies on algorithms such as Group Relative Policy Optimization (GRPO)~\cite{shao2024deepseekmathpushinglimitsmathematical}, which establishes a baseline for group-based advantage estimation without value functions. Building on this, Decoupled Clip and Dynamic Sampling Policy Optimization (DAPO)~\cite{yu2025dapo} addresses training instability and entropy collapse in long-chain reasoning through decoupled clipping and dynamic sampling mechanisms.

Building on these frameworks, several studies aim to elicit robust reasoning behaviors in MLLMs. Vision-R1 \cite{huang2025visionr1incentivizingreasoningcapability} proposes \textit{progressive thinking suppression training} to reduce overthinking during cold starts, encouraging efficient reasoning trajectories. Similarly, MM-Eureka~\cite{meng2025mmeurekaexploringfrontiersmultimodal} extends rule-based RL to the multimodal domain, inducing emergent self-correction and reasoning patterns comparable to text-only models.

To further improve reasoning quality, recent methods incorporate explicit self-reflection into the RL loop. VL-Rethinker~\cite{wang2025vlrethinkerincentivizingselfreflectionvisionlanguage} addresses the vanishing advantage problem via \textit{selective sample replay} and a F\textit{orced Rethinking} trigger. SRPO (Self-Reflection enhanced reasoning with Group
Relative Policy Optimization)~\cite{wan2025srpo} applies a two-stage framework that trains models to generate and refine their own reflections, rewarding instructive self-corrections. Adopting a tree-search perspective, Mulberry~\cite{yao2024mulberry} leverages \textit{Collective Monte Carlo Tree Search} to construct high-quality reasoning paths for training reasoning-intensive MLLMs.

Beyond general reasoning, increasing focus has been placed on visual grounding in RL. PAPO (Perception-Aware Policy Optimization)~\cite{wang2025perceptionawarepolicyoptimizationmultimodal} introduces an \textit{Implicit Perception Loss} to anchor reasoning steps to visual inputs, thereby reducing perception errors. VPPO (Visually-Perceptive Policy Optimization)~\cite{huang2025spotlighttokenperceptionmultimodal} further refines this by analyzing token-level visual dependency to prioritize perceptually pivotal tokens. Extending to more complex inputs, PeRL~\cite{zhang2025perl} utilizes a permutation-enhanced RL framework for interleaved vision–language tasks, improving multi-image reasoning through rollout filtering and position-aware exploration.

\section{Theoretical Results}
\label{app::theory}

\providecommand{\KL}{D_{\mathrm{KL}}}

Throughout this appendix, \(\KL(\cdot\|\cdot)\) denotes \(D_{\mathrm{KL}}(\cdot\|\cdot)\),
and
\[
[a]_+=\max\{a,0\}.
\]
All discrete entropies are Shannon entropies and are denoted by \(\mathcal H(\cdot)\).
We follow the notation convention in the main text: uppercase letters denote random
variables and lowercase letters denote realizations.

This appendix justifies the sparse one-step-then-rollback objective in
Eq.~\eqref{eq:local_rate_max_new}. The selected set \(\mathcal A\) affects the
masked history at step \(t\) only through its earlier elements
\[
\mathcal A_{<t}=\mathcal A\cap[t-1].
\]
Thus, the ideal local gain at step \(t\) is set-dependent: changing earlier selected
positions changes the masked context \(\widetilde h_t(\mathcal A_{<t})\), and hence
changes both the local next-token marginal and the masked-reference suffix law.

The analysis is local in the following sense. For each selected step \(t\), we evaluate
a one-step intervention that changes only the next-token marginal and then rolls back
to the native decoder for the future suffix. The resulting quantity is a marginal
downstream visual gain. The joint sparse objective sums these marginal gains over
selected steps; it is not the mutual information of a full rollout in which all
selected steps are intervened on simultaneously.

When an optimizer need not exist, all occurrences of \(\max\) can be read as \(\sup\).
The proofs below are unchanged under this convention.


\subsection{Problem Setup}
\label{app::theory-problem-setup}

Let \(V\) be the visual input, \(X\) the textual prompt, and \(Y_{1:T}\) the reasoning
trajectory. Let \([T]=\{1,\ldots,T\}\). The native decoder \(p_\theta\) generates
\[
p_\theta(Y_{1:T}\mid V,X)
=
\prod_{s=1}^T p_\theta(Y_s\mid V,X,Y_{<s}).
\]
For a realized native trajectory, write
\[
h_t=(x,y_{<t}),
\qquad
Y_{\ge t}=(Y_t,\ldots,Y_T),
\qquad
Y_{>t}=(Y_{t+1},\ldots,Y_T).
\]

Fix a candidate selected set \(\mathcal A\subseteq[T]\). Let
\[
\mathcal A_{<t}=\mathcal A\cap[t-1],
\]
and define the realized masked history
\[
\widetilde h_t(\mathcal A_{<t})
=
\bigl(x,M_t(y_{<t};\mathcal A_{<t})\bigr).
\]
When \(\mathcal A\) and \(t\) are clear from context, we write
\(\widetilde h_t\) instead of \(\widetilde h_t(\mathcal A_{<t})\). All conditional
quantities below are evaluated conditioned on this realized masked history. Thus,
\(V\mid\widetilde h_t\) denotes the population distribution of visual inputs compatible
with the masked context \(\widetilde h_t\).

For a visual realization \(v\), define the masked-context native next-token marginal
\[
m_{t,v}(y)
:=
p_\theta(y\mid V=v,\widetilde h_t).
\]
Define the native masked-history rollback continuation by
\[
K_{t,v}(y_{>t}\mid y)
:=
p_\theta(y_{>t}\mid V=v,\widetilde h_t,Y_t=y).
\]
The corresponding native masked-history suffix law is
\begin{equation}
\label{eq:masked_native_suffix}
\widetilde p_{t,v}(y,y_{>t})
=
m_{t,v}(y)K_{t,v}(y_{>t}\mid y).
\end{equation}

A local reshaping replaces only the next-token marginal \(m_{t,v}\) by
\[
q_t(\cdot\mid V=v,\widetilde h_t),
\]
while keeping the native rollback continuation \(K_{t,v}(\cdot\mid y)\) unchanged.
The intervened suffix law is
\begin{equation}
\label{eq:intervened_suffix}
\widetilde p_{t,v}^{q_t}(y,y_{>t})
=
q_t(y\mid V=v,\widetilde h_t)K_{t,v}(y_{>t}\mid y).
\end{equation}

The downstream visual gain of this one-step intervention is
\begin{equation}
\label{eq:Delta_t_q_A_def}
\Delta_t(q_t;\mathcal A_{<t})
=
I_{q_t}(V;Y_{\ge t}\mid \widetilde h_t)
-
I_{p_\theta}(V;Y_{\ge t}\mid \widetilde h_t),
\end{equation}
where \(I_{q_t}\) is computed under the one-step-then-rollback rollout in
Eq.~\eqref{eq:intervened_suffix}, and the native term is computed under
Eq.~\eqref{eq:masked_native_suffix}.

The local feasible set is
\begin{equation}
\label{eq:Q_t_A_def}
\mathcal Q_t(\mathcal A_{<t},\varepsilon)
=
\left\{
q_t:
\KL\!\left(
q_t(\cdot\mid V,\widetilde h_t)
\,\middle\|\,
m_{t,V}
\right)
\le \varepsilon
\quad
\text{for }V\mid\widetilde h_t\text{-a.e.}
\right\}.
\end{equation}
The local optimal downstream gain is
\begin{equation}
\label{eq:Delta_t_star_A_def}
\Delta_t^\star(\mathcal A_{<t},\varepsilon)
=
\sup_{q_t\in\mathcal Q_t(\mathcal A_{<t},\varepsilon)}
\Delta_t(q_t;\mathcal A_{<t}).
\end{equation}

We also define the full-history native next-token distribution, evaluated at the same
realized full prefix \(h_t=(x,y_{<t})\), by
\[
\pi_{t,v}^{\mathrm{full}}(y)
:=
p_\theta(y\mid V=v,h_t),
\]
and the full-history native continuation by
\[
K_{t,v}^{\mathrm{full}}(y_{>t}\mid y)
:=
p_\theta(y_{>t}\mid V=v,h_t,Y_t=y).
\]
The full-history suffix law is
\begin{equation}
\label{eq:full_suffix_law}
p_{t,v}(y,y_{>t})
=
\pi_{t,v}^{\mathrm{full}}(y)K_{t,v}^{\mathrm{full}}(y_{>t}\mid y).
\end{equation}
The asymmetry between \(p_{t,v}\) and the masked-reference suffix law below is
intentional: \(p_{t,v}\) measures the visual separability of the realized full-history
continuation, while the masked-reference law removes visual evidence directly leaked
through previously selected positions.

Let
\[
K_\rho=\lfloor \rho T\rfloor.
\]
For a selected set \(\mathcal A\), define the decoder-feasible class
\[
\Pi(\mathcal A,\varepsilon)
=
\left\{
\pi:
\pi(\cdot\mid V,\widetilde h_t(\mathcal A_{<t}))
\in
\mathcal Q_t(\mathcal A_{<t},\varepsilon)
\quad
\forall t\in\mathcal A
\right\}.
\]
For \(\pi\in\Pi(\mathcal A,\varepsilon)\), write
\[
\Delta_t(\pi;\mathcal A_{<t})
=
\Delta_t(q_t^\pi;\mathcal A_{<t}),
\qquad
q_t^\pi(\cdot\mid V,\widetilde h_t)
=
\pi(\cdot\mid V,\widetilde h_t).
\]
The joint sparse oracle objective corresponding to Eq.~\eqref{eq:local_rate_max_new}
is
\begin{equation}
\label{eq:J_star_def}
\mathcal J^\star(\varepsilon)
=
\max_{\mathcal A\subseteq[T],\,|\mathcal A|\le K_\rho}
\;
\sup_{\pi\in\Pi(\mathcal A,\varepsilon)}
\sum_{t\in\mathcal A}
\Delta_t(\pi;\mathcal A_{<t}).
\end{equation}


\subsection{Additional Necessary Notations}
\label{app::theory-notations}

Fix a candidate set \(\mathcal A\), a step \(t\), and the induced masked history
\[
\widetilde h_t=\widetilde h_t(\mathcal A_{<t}).
\]

\paragraph{Masked-reference suffix law.}
Define the masked-reference suffix distribution by
\begin{equation}
\label{eq:p_M_t_def}
p_{\mathcal M,t}(y_{\ge t})
=
\mathbb E_{V\mid \widetilde h_t}
\left[
\widetilde p_{t,V}(y_{\ge t})
\right].
\end{equation}
Equivalently,
\[
p_{\mathcal M,t}(y,y_{>t})
=
\mathbb E_{V\mid \widetilde h_t}
\left[
\widetilde p_{t,V}(y,y_{>t})
\right].
\]
Its next-token marginal is denoted by
\[
m_{\mathcal M,t}(y)
=
\sum_{y_{>t}}p_{\mathcal M,t}(y,y_{>t}).
\]
For an intervened next-token distribution \(q_t\), define the corresponding intervened
masked-reference suffix law by
\[
p_{\mathcal M,t}^{q_t}(y,y_{>t})
=
\mathbb E_{V\mid \widetilde h_t}
\left[
\widetilde p_{t,V}^{q_t}(y,y_{>t})
\right].
\]

\paragraph{Downstream visual score.}
Define the pointwise masked-history-to-reference log-likelihood ratio
\[
\ell_t(v,y_{\ge t})
=
\log
\frac{
\widetilde p_{t,v}(y_{\ge t})
}{
p_{\mathcal M,t}(y_{\ge t})
}.
\]
The downstream visual score of choosing token \(y\) at step \(t\) is
\begin{equation}
\label{eq:psi_def}
\psi_t(v,y)
=
\mathbb E_{Y_{>t}\sim K_{t,v}(\cdot\mid y)}
\left[
\ell_t(v,(y,Y_{>t}))
\right].
\end{equation}
This score measures how distinguishable the future suffix induced by choosing \(y\)
remains from the masked-reference suffix distribution.

\begin{claim}
\label{claim::decomposition}
For any \(y\) with \(m_{t,v}(y)>0\),
\begin{equation}
\label{eq:psi_closed}
\psi_t(v,y)
=
\log
\frac{
m_{t,v}(y)
}{
m_{\mathcal M,t}(y)
}
+
\KL\!\left(
K_{t,v}(\cdot\mid y)
\,\middle\|\,
p_{\mathcal M,t}(\cdot\mid y)
\right).
\end{equation}
\end{claim}

\begin{claim}
\label{claim::mi-score-identity}
Under the native masked-history one-step-then-rollback suffix law
\(\widetilde p_{t,V}\),
\begin{equation}
\label{eq:mi_score_identity}
I_{p_\theta}(V;Y_{\ge t}\mid \widetilde h_t)
=
\mathbb E_{V\mid \widetilde h_t}
\mathbb E_{Y_t\sim m_{t,V}}
\left[
\psi_t(V,Y_t)
\right].
\end{equation}
\end{claim}


\subsection{Assumptions}
\label{app::theory-assumptions}

\begin{assumption}[Finite common support and absolute continuity]
\label{assumption::support}
For each step \(t\), all next-token distributions are supported on a finite set
\(\mathcal Y_t\). Since the horizon \(T\) is finite, the corresponding suffix space is
also finite. For every candidate set \(\mathcal A\), every step \(t\), and
\(V\mid\widetilde h_t(\mathcal A_{<t})\)-almost every \(v\), the distributions
\(\widetilde p_{t,v}\), \(\widetilde p_{t,v}^{q_t}\), \(p_{t,v}\), and
\(p_{\mathcal M,t}\) are mutually compatible in the following sense: whenever a
likelihood ratio or KL divergence appears below, the numerator is absolutely continuous
with respect to the denominator. In particular,
\[
q_t(\cdot\mid v,\widetilde h_t)\ll m_{t,v}
\]
for every feasible \(q_t\).
\end{assumption}

\begin{remark}[Role of Assumption~\ref{assumption::support}]
This assumption avoids support pathologies and ensures that the KL divergences,
likelihood ratios, and exponential-tilting perturbations are well-defined. For a
finite-vocabulary decoder with strictly positive softmax probabilities, the next-token
part of this assumption is automatic. For suffix distributions, it can be enforced by
restricting the analysis to the common support used for scoring.
\end{remark}

\begin{assumption}[Bounded reference ratios and bounded full-history gap]
\label{assumption::bounded-ratio}
There exists \(0<B<\infty\) such that, for every candidate set \(\mathcal A\), every
step \(t\), and every \((v,y,y_{>t})\) with
\(\widetilde p_{t,v}(y,y_{>t})>0\),
\[
\left|
\log
\frac{
\widetilde p_{t,v}(y,y_{>t})
}{
p_{\mathcal M,t}(y,y_{>t})
}
\right|
\le B.
\]
Moreover, the full-history suffix gap is bounded by the same constant:
\[
0
\le
\KL(p_{t,v}\|p_{\mathcal M,t})
\le B
\qquad
\text{for }V\mid \widetilde h_t(\mathcal A_{<t})\text{-a.e. }v.
\]
\end{assumption}

\begin{remark}[Role of Assumption~\ref{assumption::bounded-ratio}]
The first bound ensures that the downstream score \(\psi_t\) is uniformly bounded,
which is needed for the KL trust-region expansion. The second bound ensures that the
full-history reference gap \(\KL(p_{t,V}\|p_{\mathcal M,t})\) is finite and can be used
as a multiplicative reference-gap factor in the entropy lower bound. Both are standard
bounded-likelihood-ratio conditions. In practice, the same condition can be imposed on
the truncated candidate support used for scoring.
\end{remark}

\begin{assumption}[Generic non-coincidence]
\label{assumption::generic}
There exists \(\gamma>0\) such that, for every candidate set \(\mathcal A\), every step
\(t\), and \(V\mid\widetilde h_t(\mathcal A_{<t})\)-almost surely, the score map
\[
y\mapsto \psi_t(V,y)
\]
is injective on the support of \(m_{t,V}\) and \(\gamma\)-separated. That is, for any
\(y\neq y'\) with \(m_{t,V}(y)m_{t,V}(y')>0\),
\[
|\psi_t(V,y)-\psi_t(V,y')|
\ge \gamma.
\]
\end{assumption}

\begin{remark}[Role of Assumption~\ref{assumption::generic}]
Exact score ties are a measure-zero degeneracy. In our setting, \(\psi_t(V,y)\) is a
real-valued functional of logits and continuation likelihoods. Under non-degenerate
perturbations, distinct candidate tokens almost surely have distinct downstream scores.
The separation constant \(\gamma\) formalizes finite numerical resolution. Since
Assumption~\ref{assumption::bounded-ratio} also bounds \(|\psi_t|\), this assumption is
best interpreted on the finite candidate set used by the reshaping or scoring procedure,
rather than as a claim about an arbitrarily large full vocabulary.
\end{remark}

\begin{assumption}[Local target realizability]
\label{assumption::realizability}
For every candidate set \(\mathcal A\) and every collection of feasible local reshaping
targets
\[
\{q_t:t\in\mathcal A,\ q_t\in\mathcal Q_t(\mathcal A_{<t},\varepsilon)\},
\]
there exists a decoder \(\pi\in\Pi(\mathcal A,\varepsilon)\) such that, for every
\(t\in\mathcal A\),
\[
\pi(\cdot\mid V,\widetilde h_t(\mathcal A_{<t}))
=
q_t(\cdot\mid V,\widetilde h_t(\mathcal A_{<t}))
\quad
\text{for }V\mid\widetilde h_t(\mathcal A_{<t})\text{-a.e.}
\]
\end{assumption}

\begin{remark}[Role of Assumption~\ref{assumption::realizability}]
This is an oracle target-construction assumption. The theory identifies local reshaping
targets at selected masked histories, and training a single decoder \(\pi\) attempts to
approximate those targets. If exact realizability is replaced by approximation up to
error \(\delta_{\mathrm{real}}\) in downstream gain per selected step, the final joint
lower bound loses at most \(K_\rho\delta_{\mathrm{real}}\).
\end{remark}


\subsection{Local Technical Lemmas}
\label{app::theory-lemmas}

\begin{lemma}[Exact first-order decomposition and one-sided lower bound]
\label{lem:gain-first-order-kl}
Fix a candidate set \(\mathcal A\), a step \(t\), and the induced masked history
\(\widetilde h_t=\widetilde h_t(\mathcal A_{<t})\). Assume
Assumption~\ref{assumption::support}. For any
\(q_t\in\mathcal Q_t(\mathcal A_{<t},\varepsilon)\), the downstream visual gain admits
the exact decomposition
\begin{equation}
\label{eq:exact_decomp_kl}
\Delta_t(q_t;\mathcal A_{<t})
=
G_{1,t}(q_t;\mathcal A_{<t})
+
\mathbb E_{V\mid \widetilde h_t}
\left[
\KL\!\left(
q_t(\cdot\mid V,\widetilde h_t)
\,\middle\|\,
m_{t,V}
\right)
\right]
-
\KL\!\left(
p_{\mathcal M,t}^{q_t}
\,\middle\|\,
p_{\mathcal M,t}
\right),
\end{equation}
where
\begin{equation}
\label{eq:first_order_gain}
G_{1,t}(q_t;\mathcal A_{<t})
=
\mathbb E_{V\mid \widetilde h_t}
\left[
\sum_y
\left(
q_t(y\mid V,\widetilde h_t)
-
m_{t,V}(y)
\right)
\psi_t(V,y)
\right].
\end{equation}
Moreover,
\begin{equation}
\label{eq:lemma-b1-onesided}
\Delta_t(q_t;\mathcal A_{<t})
\ge
G_{1,t}(q_t;\mathcal A_{<t}).
\end{equation}
\end{lemma}

\begin{proof}
Throughout the proof, write
\[
\mathbb E[\cdot]
=
\mathbb E_{V\mid \widetilde h_t}[\cdot],
\qquad
q_V(y)
=
q_t(y\mid V,\widetilde h_t).
\]
Under the one-step intervention,
\[
\widetilde p_{t,V}^{q_t}(y,y_{>t})
=
q_V(y)K_{t,V}(y_{>t}\mid y),
\]
and
\[
p_{\mathcal M,t}^{q_t}(y,y_{>t})
=
\mathbb E
\left[
\widetilde p_{t,V}^{q_t}(y,y_{>t})
\right].
\]
Thus
\[
I_{q_t}(V;Y_{\ge t}\mid \widetilde h_t)
=
\mathbb E
\left[
\sum_{y,y_{>t}}
\widetilde p_{t,V}^{q_t}(y,y_{>t})
\log
\frac{
\widetilde p_{t,V}^{q_t}(y,y_{>t})
}{
p_{\mathcal M,t}^{q_t}(y,y_{>t})
}
\right].
\]
Insert the native masked-history ratio:
\[
\frac{
\widetilde p_{t,V}^{q_t}(y,y_{>t})
}{
p_{\mathcal M,t}^{q_t}(y,y_{>t})
}
=
\frac{
\widetilde p_{t,V}(y,y_{>t})
}{
p_{\mathcal M,t}(y,y_{>t})
}
\cdot
\frac{
q_V(y)
}{
m_{t,V}(y)
}
\cdot
\frac{
p_{\mathcal M,t}(y,y_{>t})
}{
p_{\mathcal M,t}^{q_t}(y,y_{>t})
}.
\]
Splitting the logarithm gives three terms. The first term equals
\[
\mathbb E
\left[
\sum_y
q_V(y)\psi_t(V,y)
\right].
\]
The second term equals
\[
\mathbb E
\left[
\KL(q_V\|m_{t,V})
\right].
\]
The third term equals
\[
-
\KL(p_{\mathcal M,t}^{q_t}\|p_{\mathcal M,t}).
\]
Therefore,
\[
I_{q_t}(V;Y_{\ge t}\mid \widetilde h_t)
=
\mathbb E
\left[
\sum_y
q_V(y)\psi_t(V,y)
\right]
+
\mathbb E
\left[
\KL(q_V\|m_{t,V})
\right]
-
\KL(p_{\mathcal M,t}^{q_t}\|p_{\mathcal M,t}).
\]
Under the native choice, \(q_V=m_{t,V}\) and
\(p_{\mathcal M,t}^{q_t}=p_{\mathcal M,t}\). By
Claim~\ref{claim::mi-score-identity},
\[
I_{p_\theta}(V;Y_{\ge t}\mid \widetilde h_t)
=
\mathbb E
\left[
\sum_y
m_{t,V}(y)\psi_t(V,y)
\right].
\]
Subtracting proves Eq.~\eqref{eq:exact_decomp_kl}.

It remains to show the one-sided bound. By joint convexity of KL divergence,
\[
\KL(p_{\mathcal M,t}^{q_t}\|p_{\mathcal M,t})
=
\KL\!\left(
\mathbb E[\widetilde p_{t,V}^{q_t}]
\,\middle\|\,
\mathbb E[\widetilde p_{t,V}]
\right)
\le
\mathbb E
\left[
\KL(\widetilde p_{t,V}^{q_t}\|\widetilde p_{t,V})
\right].
\]
Because the rollback continuation \(K_{t,V}(\cdot\mid y)\) is the same in
\(\widetilde p_{t,V}^{q_t}\) and \(\widetilde p_{t,V}\), the continuation cancels in
the inner KL:
\[
\KL(\widetilde p_{t,V}^{q_t}\|\widetilde p_{t,V})
=
\KL(q_V\|m_{t,V}).
\]
Therefore,
\[
\KL(p_{\mathcal M,t}^{q_t}\|p_{\mathcal M,t})
\le
\mathbb E
\left[
\KL(q_V\|m_{t,V})
\right].
\]
Substituting this inequality into the exact decomposition yields
\[
\Delta_t(q_t;\mathcal A_{<t})
\ge
G_{1,t}(q_t;\mathcal A_{<t}).
\]
\end{proof}


\begin{lemma}[Optimal first-order perturbation and variance lower bound]
\label{lem:kl-opt-and-var-lb}
Fix a candidate set \(\mathcal A\), a step \(t\), and the induced masked history
\(\widetilde h_t=\widetilde h_t(\mathcal A_{<t})\). Assume
Assumptions~\ref{assumption::support} and \ref{assumption::bounded-ratio}. For each
visual realization \(v\), define
\[
\mu_t(v)
=
\sum_y m_{t,v}(y)\psi_t(v,y),
\qquad
\widetilde\psi_t(v,y)
=
\psi_t(v,y)-\mu_t(v),
\]
and
\[
\sigma_t(v)
=
\sqrt{
\Var_{Y\sim m_{t,v}}
\big(
\psi_t(v,Y)
\big)
}.
\]
Define the averaged conditional score standard deviation
\begin{equation}
\label{eq:Sigma_bar_def}
\overline\Sigma_t(\mathcal A_{<t})
=
\mathbb E_{V\mid \widetilde h_t(\mathcal A_{<t})}
\left[
\sigma_t(V)
\right].
\end{equation}
Let
\[
G_{1,t}^\star(\mathcal A_{<t},\varepsilon)
=
\sup_{q_t\in\mathcal Q_t(\mathcal A_{<t},\varepsilon)}
G_{1,t}(q_t;\mathcal A_{<t}).
\]
Then the following statements hold.

\textbf{(i) Pointwise first-order optimizer.}
For \(V\mid\widetilde h_t\)-almost every \(v\), any finite interior optimizer of the
pointwise first-order problem has the exponential-tilting form
\begin{equation}
\label{eq:qstar_tilt_final}
q_t^\star(y\mid v,\widetilde h_t)
=
\frac{
m_{t,v}(y)
\exp\!\big(\eta_t^\star(v)\psi_t(v,y)\big)
}{
\sum_{y'}
m_{t,v}(y')
\exp\!\big(\eta_t^\star(v)\psi_t(v,y')\big)
}
=
\frac{
m_{t,v}(y)
\exp\!\big(\eta_t^\star(v)\widetilde\psi_t(v,y)\big)
}{
\sum_{y'}
m_{t,v}(y')
\exp\!\big(\eta_t^\star(v)\widetilde\psi_t(v,y')\big)
},
\end{equation}
where \(\eta_t^\star(v)\ge0\). If \(\sigma_t(v)=0\), the native choice
\(q_t^\star(\cdot\mid v,\widetilde h_t)=m_{t,v}(\cdot)\) is optimal.

\textbf{(ii) Lower bound for the optimal first-order gain.}
There exists a constant \(C_B>0\), depending only on the bounded-ratio constant \(B\),
such that for every \(\varepsilon>0\),
\begin{equation}
\label{eq:lemma-b2-bound}
G_{1,t}^\star(\mathcal A_{<t},\varepsilon)
\ge
\sqrt{2\varepsilon}\,\overline\Sigma_t(\mathcal A_{<t})
-
C_B\varepsilon.
\end{equation}
\end{lemma}

\begin{proof}
For a fixed visual realization \(v\), define the pointwise first-order objective
\[
g_v(q)
=
\sum_y
\left(
q(y)-m_{t,v}(y)
\right)
\psi_t(v,y)
\]
under the constraint
\[
\KL(q\|m_{t,v})\le\varepsilon.
\]
Because the KL constraint is pointwise in \(v\), the first-order optimization separates
across visual realizations.

For fixed \(v\), maximizing \(g_v(q)\) is equivalent to maximizing
\[
\sum_y q(y)\psi_t(v,y)
\]
subject to the KL constraint. The Lagrangian calculation gives, for some
\(\lambda_v>0\),
\[
q^\star(y)
\propto
m_{t,v}(y)
\exp\!\left(\frac{\psi_t(v,y)}{\lambda_v}\right).
\]
Writing \(\eta_t^\star(v)=1/\lambda_v\) gives the exponential-tilting form in
Eq.~\eqref{eq:qstar_tilt_final}. Adding the \(v\)-dependent constant \(-\mu_t(v)\) to
the score does not change the normalized distribution, so the centered form is
equivalent. If \(\sigma_t(v)=0\), then \(g_v(q)=0\) for all feasible \(q\), and the
native choice is optimal. This proves Part (i).

We now prove Part (ii). By Assumption~\ref{assumption::bounded-ratio}, for every
\(y\in\supp(m_{t,v})\),
\[
|\psi_t(v,y)|\le B.
\]
Therefore the centered score satisfies
\[
|\widetilde\psi_t(v,y)|\le 2B.
\]
Let
\[
M=2B.
\]
For a fixed \(v\), consider the centered exponential family
\[
q_{v,\eta}(y)
=
\frac{
m_{t,v}(y)\exp\!\big(\eta\widetilde\psi_t(v,y)\big)
}{
\sum_{y'}
m_{t,v}(y')\exp\!\big(\eta\widetilde\psi_t(v,y')\big)
},
\qquad
\eta\ge0.
\]
Let
\[
\Lambda_v(\eta)
=
\log
\sum_y
m_{t,v}(y)\exp\!\big(\eta\widetilde\psi_t(v,y)\big).
\]
Then
\[
\Lambda_v'(0)=0,
\qquad
\Lambda_v''(0)=\sigma_t^2(v),
\]
and
\[
\Lambda_v'(\eta)
=
\mathbb E_{Y\sim q_{v,\eta}}
\left[
\widetilde\psi_t(v,Y)
\right],
\qquad
\Lambda_v''(\eta)
=
\Var_{Y\sim q_{v,\eta}}
\left[
\widetilde\psi_t(v,Y)
\right].
\]

Since \(|\widetilde\psi_t|\le M\), the density ratio between \(q_{v,\eta}\) and
\(m_{t,v}\) is bounded between \(e^{-2M\eta}\) and \(e^{2M\eta}\). Hence
\[
e^{-2M\eta}\sigma_t^2(v)
\le
\Lambda_v''(\eta)
\le
e^{2M\eta}\sigma_t^2(v).
\]
For \(0\le \eta\le 1/(2M)\), using \(e^{-x}\ge 1-x\), we obtain
\[
\Lambda_v'(\eta)
=
\int_0^\eta \Lambda_v''(u)\,du
\ge
\sigma_t^2(v)\int_0^\eta (1-2Mu)\,du
=
\sigma_t^2(v)\bigl(\eta-M\eta^2\bigr).
\]
Moreover,
\[
\KL(q_{v,\eta}\|m_{t,v})
=
\eta\Lambda_v'(\eta)-\Lambda_v(\eta)
=
\int_0^\eta u\,\Lambda_v''(u)\,du.
\]
For \(0\le u\le 1/(2M)\), we have \(e^{2Mu}\le 1+4Mu\). Therefore,
\[
\KL(q_{v,\eta}\|m_{t,v})
\le
\sigma_t^2(v)
\int_0^\eta u(1+4Mu)\,du
=
\sigma_t^2(v)
\left(
\frac{\eta^2}{2}
+
\frac{4M\eta^3}{3}
\right).
\]
Equivalently,
\[
\KL(q_{v,\eta}\|m_{t,v})
\le
\frac{\sigma_t^2(v)\eta^2}{2}
\left(
1+\frac{8M\eta}{3}
\right).
\]

If \(\sigma_t(v)=0\), the desired lower bound holds because the native choice gives
first-order gain zero. Assume \(\sigma_t(v)>0\). Define
\[
a_v=\frac{\sqrt{2\varepsilon}}{\sigma_t(v)},
\qquad
s_v=M a_v,
\qquad
d_0=\frac{8}{3}.
\]
First consider the small-radius regime \(s_v\le 1/d_0\). Choose
\[
\eta_v=a_v(1-d_0s_v).
\]
Then \(0\le\eta_v\le a_v\) and \(M\eta_v\le s_v\le 1/d_0<1/2\). Hence the preceding
KL bound applies. Since \(M\eta_v\le s_v\),
\[
\KL(q_{v,\eta_v}\|m_{t,v})
\le
\frac{\sigma_t^2(v)a_v^2}{2}
(1-d_0s_v)^2(1+d_0s_v).
\]
Because \(\sigma_t^2(v)a_v^2/2=\varepsilon\) and
\[
(1-z)^2(1+z)\le1
\qquad
\text{for }0\le z\le1,
\]
with \(z=d_0s_v\), we have
\[
\KL(q_{v,\eta_v}\|m_{t,v})\le\varepsilon.
\]
Thus \(q_{v,\eta_v}\) is feasible. Its first-order gain satisfies
\[
g_v(q_{v,\eta_v})
=
\Lambda_v'(\eta_v)
\ge
\sigma_t^2(v)(\eta_v-M\eta_v^2).
\]
Using \(\eta_v=a_v(1-d_0s_v)\) and \(s_v=Ma_v\), we obtain
\[
g_v(q_{v,\eta_v})
\ge
\sqrt{2\varepsilon}\,\sigma_t(v)
-
2M(d_0+1)\varepsilon.
\]

Now consider the complementary regime \(s_v>1/d_0\). Since
\(a_v=\sqrt{2\varepsilon}/\sigma_t(v)\), this implies
\[
\sqrt{2\varepsilon}\,\sigma_t(v)
=
\frac{2\varepsilon}{a_v}
<
2Md_0\,\varepsilon.
\]
Choosing \(C_B\ge 2M\max\{d_0+1,d_0\}\), the right-hand side
\[
\sqrt{2\varepsilon}\,\sigma_t(v)-C_B\varepsilon
\]
is nonpositive. Since the native choice \(q=m_{t,v}\) is feasible and gives
\(g_v(q)=0\), the desired lower bound also holds in this regime.

Combining both regimes, for all \(v\),
\[
g_v^\star(\varepsilon)
\ge
\sqrt{2\varepsilon}\,\sigma_t(v)-C_B\varepsilon,
\]
where \(C_B\) depends only on \(B\), since \(M=2B\). Taking expectation over
\(V\mid\widetilde h_t(\mathcal A_{<t})\) yields
\[
G_{1,t}^\star(\mathcal A_{<t},\varepsilon)
\ge
\sqrt{2\varepsilon}\,
\mathbb E_{V\mid\widetilde h_t(\mathcal A_{<t})}[\sigma_t(V)]
-
C_B\varepsilon
=
\sqrt{2\varepsilon}\,\overline\Sigma_t(\mathcal A_{<t})
-
C_B\varepsilon.
\]
\end{proof}


\begin{lemma}[Entropy--reference-gap lower bound]
\label{lem:sigma-entropy-kl-lb}
Fix a candidate set \(\mathcal A\), a step \(t\), and the induced masked history
\(\widetilde h_t=\widetilde h_t(\mathcal A_{<t})\). Assume
Assumptions~\ref{assumption::support}, \ref{assumption::bounded-ratio}, and
\ref{assumption::generic}. Then
\begin{equation}
\label{eq:sigma_entropy_kl_lb}
\overline\Sigma_t(\mathcal A_{<t})
\ge
\frac{\gamma}{\sqrt{2\pi e\,B}}
\mathbb E_{V\mid \widetilde h_t(\mathcal A_{<t})}
\left[
\left(
\exp\!\big(\mathcal H(m_{t,V})\big)-2
\right)_+
\sqrt{
\KL\!\left(
p_{t,V}
\,\middle\|\,
p_{\mathcal M,t}
\right)
}
\right].
\end{equation}
\end{lemma}

\begin{remark}
Lemma~\ref{lem:sigma-entropy-kl-lb} expresses \(\overline\Sigma_t(\mathcal A_{<t})\)
as a product of an entropy-based factor and a reference-gap factor. The entropy factor
arises from a variance--entropy inequality applied to \(\psi_t(V,Y)\). The reference-gap
factor is introduced through Assumption~\ref{assumption::bounded-ratio}: since
\[
\KL(p_{t,V}\|p_{\mathcal M,t})/B\in[0,1],
\]
multiplying any nonnegative lower bound on \(\sigma_t^2(V)\) by this ratio yields a
possibly looser lower bound on the same quantity. The resulting factored form is
therefore conservative: it is tight in the entropy factor but loose in the reference-gap
factor, and reflects the operational principle that branching room is useful only when
the corresponding continuation remains distinguishable from the vision-marginalized
reference.
\end{remark}

\begin{proof}
Fix a candidate set \(\mathcal A\), a step \(t\), the induced masked history
\(\widetilde h_t\), and a visual realization \(v\). Let \(Y\sim m_{t,v}\), and define
\[
Z=\psi_t(v,Y).
\]
By Assumption~\ref{assumption::generic}, distinct values of \(Z\) are separated by at
least \(\gamma\). Let \(U\sim\operatorname{Unif}[0,\gamma)\) be independent of \(Z\),
and define
\[
W=Z+U.
\]
Since the support points of \(Z\) are \(\gamma\)-separated, the intervals
\([z,z+\gamma)\) are disjoint up to endpoints. Hence
\[
h(W)=\mathcal H(Z)+\log\gamma,
\]
where \(h(\cdot)\) denotes differential entropy. By the maximum-entropy property of
Gaussians,
\[
\Var(W)
\ge
\frac{1}{2\pi e}\exp(2h(W)).
\]
Since \(Z\) and \(U\) are independent,
\[
\Var(W)=\Var(Z)+\frac{\gamma^2}{12}.
\]
Therefore,
\[
\Var(Z)+\frac{\gamma^2}{12}
\ge
\frac{\gamma^2}{2\pi e}\exp(2\mathcal H(Z)).
\]
Because \(Z=\psi_t(v,Y)\) is a deterministic function of \(Y\), and
Assumption~\ref{assumption::generic} makes this function injective on
\(\supp(m_{t,v})\),
\[
\mathcal H(Z)=\mathcal H(Y)=\mathcal H(m_{t,v}).
\]
Thus
\[
\Var_{Y\sim m_{t,v}}
\big(
\psi_t(v,Y)
\big)
\ge
\left[
\frac{\gamma^2}{2\pi e}
\exp(2\mathcal H(m_{t,v}))
-
\frac{\gamma^2}{12}
\right]_+.
\]
Equivalently,
\[
\Var_{Y\sim m_{t,v}}
\big(
\psi_t(v,Y)
\big)
\ge
\frac{\gamma^2}{2\pi e}
\left[
\exp(2\mathcal H(m_{t,v}))
-
\frac{\pi e}{6}
\right]_+.
\]
By Assumption~\ref{assumption::bounded-ratio},
\[
0
\le
\KL(p_{t,v}\|p_{\mathcal M,t})
\le B.
\]
Therefore, for any nonnegative \(A_0\),
\[
A_0
\ge
\frac{\KL(p_{t,v}\|p_{\mathcal M,t})}{B}A_0.
\]
Applying this with
\[
A_0
=
\frac{\gamma^2}{2\pi e}
\left[
\exp(2\mathcal H(m_{t,v}))
-
\frac{\pi e}{6}
\right]_+
\]
gives
\[
\Var_{Y\sim m_{t,v}}
\big(
\psi_t(v,Y)
\big)
\ge
\frac{\gamma^2}{2\pi e\,B}
\left[
\exp(2\mathcal H(m_{t,v}))
-
\frac{\pi e}{6}
\right]_+
\KL(p_{t,v}\|p_{\mathcal M,t}).
\]
Taking square roots yields
\[
\sigma_t(v)
\ge
\frac{\gamma}{\sqrt{2\pi e\,B}}
\sqrt{
\left[
\exp(2\mathcal H(m_{t,v}))
-
\frac{\pi e}{6}
\right]_+
\KL(p_{t,v}\|p_{\mathcal M,t})
}.
\]
Using
\[
\sqrt{(x-c)_+}
\ge
(\sqrt{x}-\sqrt c)_+,
\qquad x\ge0,\ c>0,
\]
with
\[
x=\exp(2\mathcal H(m_{t,v})),
\qquad
c=\frac{\pi e}{6},
\]
and using \(\sqrt{\pi e/6}<2\), we obtain
\[
\sigma_t(v)
\ge
\frac{\gamma}{\sqrt{2\pi e\,B}}
\left(
\exp(\mathcal H(m_{t,v}))-2
\right)_+
\sqrt{
\KL(p_{t,v}\|p_{\mathcal M,t})
}.
\]
Averaging over \(V\mid\widetilde h_t(\mathcal A_{<t})\) proves
Eq.~\eqref{eq:sigma_entropy_kl_lb}.
\end{proof}


\begin{lemma}[Local downstream visual-gain lower bound]
\label{lem:local-gain-lower-bound}
Fix a candidate set \(\mathcal A\), a step \(t\), and the induced masked history
\(\widetilde h_t=\widetilde h_t(\mathcal A_{<t})\). Assume
Assumptions~\ref{assumption::support}, \ref{assumption::bounded-ratio}, and
\ref{assumption::generic}. Then
\begin{equation}
\label{eq:local_entropy_gain_bound}
\Delta_t^\star(\mathcal A_{<t},\varepsilon)
\ge
\frac{\gamma\sqrt{\varepsilon}}{\sqrt{\pi e\,B}}
\mathbb E_{V\mid \widetilde h_t(\mathcal A_{<t})}
\left[
\left(
\exp\!\big(\mathcal H(m_{t,V})\big)-2
\right)_+
\sqrt{
\KL\!\left(
p_{t,V}
\,\middle\|\,
p_{\mathcal M,t}
\right)
}
\right]
-
C_B\varepsilon.
\end{equation}
\end{lemma}

\begin{proof}
By Lemma~\ref{lem:gain-first-order-kl}, for any
\(q_t\in\mathcal Q_t(\mathcal A_{<t},\varepsilon)\),
\[
\Delta_t(q_t;\mathcal A_{<t})
\ge
G_{1,t}(q_t;\mathcal A_{<t}).
\]
Taking the supremum over \(q_t\in\mathcal Q_t(\mathcal A_{<t},\varepsilon)\) gives
\[
\Delta_t^\star(\mathcal A_{<t},\varepsilon)
\ge
G_{1,t}^\star(\mathcal A_{<t},\varepsilon).
\]
By Lemma~\ref{lem:kl-opt-and-var-lb}(ii),
\[
G_{1,t}^\star(\mathcal A_{<t},\varepsilon)
\ge
\sqrt{2\varepsilon}\,\overline\Sigma_t(\mathcal A_{<t})
-
C_B\varepsilon.
\]
Combining the two inequalities,
\[
\Delta_t^\star(\mathcal A_{<t},\varepsilon)
\ge
\sqrt{2\varepsilon}\,\overline\Sigma_t(\mathcal A_{<t})
-
C_B\varepsilon.
\]
By Lemma~\ref{lem:sigma-entropy-kl-lb},
\[
\overline\Sigma_t(\mathcal A_{<t})
\ge
\frac{\gamma}{\sqrt{2\pi e\,B}}
\mathbb E_{V\mid \widetilde h_t(\mathcal A_{<t})}
\left[
\left(
\exp\!\big(\mathcal H(m_{t,V})\big)-2
\right)_+
\sqrt{
\KL(p_{t,V}\|p_{\mathcal M,t})
}
\right].
\]
Substituting this into the previous display yields
\[
\Delta_t^\star(\mathcal A_{<t},\varepsilon)
\ge
\frac{\gamma\sqrt{\varepsilon}}{\sqrt{\pi e\,B}}
\mathbb E_{V\mid \widetilde h_t(\mathcal A_{<t})}
\left[
\left(
\exp\!\big(\mathcal H(m_{t,V})\big)-2
\right)_+
\sqrt{
\KL(p_{t,V}\|p_{\mathcal M,t})
}
\right]
-
C_B\varepsilon,
\]
which proves the claim.
\end{proof}


\subsection{Main Theorem: Joint Sparse Selection}
\label{app::theory-main-theorem}

For each candidate set \(\mathcal A\) and each step \(t\), define the ideal local
lower-bound score
\begin{equation}
\label{eq:L_t_ideal_def}
\mathcal L_t(\mathcal A,\varepsilon)
=
\left[
\frac{\gamma\sqrt{\varepsilon}}{\sqrt{\pi e\,B}}
\mathbb E_{V\mid \widetilde h_t(\mathcal A_{<t})}
\left[
\left(
\exp\!\big(\mathcal H(m_{t,V})\big)-2
\right)_+
\sqrt{
\KL\!\left(
p_{t,V}
\,\middle\|\,
p_{\mathcal M,t}
\right)
}
\right]
-
C_B\varepsilon
\right]_+.
\end{equation}
Here \(p_{\mathcal M,t}\) and \(m_{t,V}\) are defined using the masked history
\(\widetilde h_t(\mathcal A_{<t})\). The positive part is valid because the native
choice is feasible and gives zero downstream gain.

\begin{theorem}[Joint sparse lower bound for one-step rollback gains]
\label{thm:joint-sparse-visual-retention}
Let \(K_\rho=\lfloor\rho T\rfloor\). Assume
Assumptions~\ref{assumption::support}, \ref{assumption::bounded-ratio},
\ref{assumption::generic}, and \ref{assumption::realizability}. Then the following
statements hold.

\textbf{(1) Lower bound for every candidate selected set.}
For every \(\mathcal A\subseteq[T]\) with \(|\mathcal A|\le K_\rho\),
\begin{equation}
\label{eq:set_lower_bound}
\sup_{\pi\in\Pi(\mathcal A,\varepsilon)}
\sum_{t\in\mathcal A}
\Delta_t(\pi;\mathcal A_{<t})
\ge
\sum_{t\in\mathcal A}
\mathcal L_t(\mathcal A,\varepsilon).
\end{equation}

\textbf{(2) Joint lower bound for the sparse oracle.}
The joint sparse optimum in Eq.~\eqref{eq:J_star_def} satisfies
\begin{equation}
\label{eq:joint_lower_bound}
\mathcal J^\star(\varepsilon)
\ge
\max_{\mathcal A\subseteq[T],\,|\mathcal A|\le K_\rho}
\sum_{t\in\mathcal A}
\mathcal L_t(\mathcal A,\varepsilon).
\end{equation}

\textbf{(3) Ideal lower-bound selection.}
The ideal lower-bound selected set is
\begin{equation}
\label{eq:A_LB_star_ideal}
\mathcal A_{\mathrm{LB}}^\star
\in
\arg\max_{\mathcal A\subseteq[T],\,|\mathcal A|\le K_\rho}
\sum_{t\in\mathcal A}
\mathcal L_t(\mathcal A,\varepsilon).
\end{equation}

\textbf{(4) Top-\(K_\rho\) rule as a special case.}
If the local scores are fixed independently of \(\mathcal A\), that is,
\[
\mathcal L_t(\mathcal A,\varepsilon)
=
\mathcal L_t(\varepsilon)
\quad
\text{for all }\mathcal A,
\]
then the ideal lower-bound selection reduces to
\[
\mathcal A_{\mathrm{LB}}^\star
\in
\operatorname{TopK}_{t\in[T]}
\left\{
\mathcal L_t(\varepsilon)
\right\}_{K_\rho}.
\]
Equivalently, it selects the largest \(K_\rho=\lfloor\rho T\rfloor\) ideal local
lower-bound scores, with ties broken arbitrarily.
\end{theorem}

\begin{remark}[Set-dependent scores]
Because \(\widetilde h_t(\mathcal A_{<t})\) depends on the earlier selected positions,
the score
$\mathcal L_t(\mathcal A,\varepsilon)$
can change when \(\mathcal A\) changes. Thus the general lower-bound selection problem
in Eq.~\eqref{eq:A_LB_star_ideal} is a genuine joint set optimization. The simple
top-\(K_\rho\) rule is rigorously valid only in the fixed-score special case.
\end{remark}

\begin{remark}[Scope of the lower bound]
The theorem lower bounds the sparse objective used in Eq.~\eqref{eq:local_rate_max_new},
namely the sum of one-step-then-rollback marginal downstream gains. It does not claim
that this sum is equal to the conditional mutual information of a policy that applies
all selected interventions simultaneously during a full rollout. The masking construction
and rollback evaluation intentionally isolate marginal credit for each selected step.
\end{remark}

\begin{remark}[Connection to practical reflection anchors]
The lower-bound score \(\mathcal L_t(\mathcal A,\varepsilon)\) contains two signals:
local branching room, through \(\mathcal H(m_{t,V})\), and downstream visual relevance,
through \(\KL(p_{t,V}\|p_{\mathcal M,t})\). Computing the full score would require
set-dependent masking, counterfactual suffix laws, and visual marginalization. The
practical reflection-anchor rule in the main text therefore uses native-history entropy
as a tractable proxy for the branching-room factor, while the continuation-level training
objective supplies the downstream visual-relevance pressure.
\end{remark}

\begin{proof}
Fix any candidate set \(\mathcal A\subseteq[T]\) with \(|\mathcal A|\le K_\rho\). For
each \(t\in\mathcal A\), Lemma~\ref{lem:local-gain-lower-bound} gives
\[
\Delta_t^\star(\mathcal A_{<t},\varepsilon)
\ge
\frac{\gamma\sqrt{\varepsilon}}{\sqrt{\pi e\,B}}
\mathbb E_{V\mid \widetilde h_t(\mathcal A_{<t})}
\left[
\left(
\exp\!\big(\mathcal H(m_{t,V})\big)-2
\right)_+
\sqrt{
\KL(p_{t,V}\|p_{\mathcal M,t})
}
\right]
-
C_B\varepsilon.
\]
Since the native choice \(q_t=m_{t,V}\) is feasible and gives zero gain,
\[
\Delta_t^\star(\mathcal A_{<t},\varepsilon)\ge0.
\]
Therefore,
\[
\Delta_t^\star(\mathcal A_{<t},\varepsilon)
\ge
\mathcal L_t(\mathcal A,\varepsilon).
\]
Summing over \(t\in\mathcal A\) gives
\[
\sum_{t\in\mathcal A}
\Delta_t^\star(\mathcal A_{<t},\varepsilon)
\ge
\sum_{t\in\mathcal A}
\mathcal L_t(\mathcal A,\varepsilon).
\]

We next relate the local optima to a single decoder. Fix any \(\eta>0\). For each
\(t\in\mathcal A\), choose a feasible local target
\[
q_t^\eta\in\mathcal Q_t(\mathcal A_{<t},\varepsilon)
\]
such that
\[
\Delta_t(q_t^\eta;\mathcal A_{<t})
\ge
\Delta_t^\star(\mathcal A_{<t},\varepsilon)
-
\frac{\eta}{\max\{|\mathcal A|,1\}}.
\]
By Assumption~\ref{assumption::realizability}, there exists
\(\pi^\eta\in\Pi(\mathcal A,\varepsilon)\) realizing all these local targets at the
selected masked histories. Hence
\[
\sup_{\pi\in\Pi(\mathcal A,\varepsilon)}
\sum_{t\in\mathcal A}
\Delta_t(\pi;\mathcal A_{<t})
\ge
\sum_{t\in\mathcal A}
\Delta_t(q_t^\eta;\mathcal A_{<t})
\ge
\sum_{t\in\mathcal A}
\Delta_t^\star(\mathcal A_{<t},\varepsilon)
-
\eta.
\]
Letting \(\eta\downarrow0\) gives
\[
\sup_{\pi\in\Pi(\mathcal A,\varepsilon)}
\sum_{t\in\mathcal A}
\Delta_t(\pi;\mathcal A_{<t})
\ge
\sum_{t\in\mathcal A}
\Delta_t^\star(\mathcal A_{<t},\varepsilon).
\]
Combining this with the lower bound on
\(\sum_{t\in\mathcal A}\Delta_t^\star(\mathcal A_{<t},\varepsilon)\) proves
Eq.~\eqref{eq:set_lower_bound}.

Taking the maximum over all \(\mathcal A\subseteq[T]\) with
\(|\mathcal A|\le K_\rho\), and using the definition of \(\mathcal J^\star(\varepsilon)\),
gives
\[
\mathcal J^\star(\varepsilon)
\ge
\max_{\mathcal A\subseteq[T],\,|\mathcal A|\le K_\rho}
\sum_{t\in\mathcal A}
\mathcal L_t(\mathcal A,\varepsilon),
\]
which proves Eq.~\eqref{eq:joint_lower_bound}. The selection rule in
Eq.~\eqref{eq:A_LB_star_ideal} follows directly from maximizing this lower-bound
surrogate over \(\mathcal A\).

Finally, suppose the local scores are fixed independently of \(\mathcal A\), so that
\[
\mathcal L_t(\mathcal A,\varepsilon)
=
\mathcal L_t(\varepsilon).
\]
Then the surrogate becomes modular:
\[
\max_{\mathcal A\subseteq[T],\,|\mathcal A|\le K_\rho}
\sum_{t\in\mathcal A}
\mathcal L_t(\varepsilon).
\]
Let
\[
\mathcal L_{(1)}(\varepsilon)
\ge
\mathcal L_{(2)}(\varepsilon)
\ge
\cdots
\ge
\mathcal L_{(T)}(\varepsilon)
\]
be the sorted scores. Since all scores are nonnegative by construction, any set of size
at most \(K_\rho\) has value at most
$\sum_{i=1}^{K_\rho}
\mathcal L_{(i)}(\varepsilon),$
and this value is achieved by selecting the \(K_\rho\) indices with the largest scores.
Hence, in the fixed-score special case, the lower-bound oracle selects the largest
\(K_\rho=\lfloor\rho T\rfloor\) ideal local lower-bound scores.
\end{proof}


\subsection{Proofs of Claim~\ref{claim::decomposition} and Claim~\ref{claim::mi-score-identity}}
\label{app::theory-claim-proofs}

\begin{proof}[Proof of Claim~\ref{claim::decomposition}]
Fix a candidate set \(\mathcal A\), a step \(t\), the induced masked history
\(\widetilde h_t\), and a visual realization \(v\). For any \(y\) with
\(m_{t,v}(y)>0\), by the definition of \(\psi_t\),
\[
\psi_t(v,y)
=
\mathbb E_{Y_{>t}\sim K_{t,v}(\cdot\mid y)}
\left[
\log
\frac{
\widetilde p_{t,v}(y,Y_{>t})
}{
p_{\mathcal M,t}(y,Y_{>t})
}
\right].
\]
Using the factorizations
\[
\widetilde p_{t,v}(y,y_{>t})
=
m_{t,v}(y)K_{t,v}(y_{>t}\mid y),
\]
and
\[
p_{\mathcal M,t}(y,y_{>t})
=
m_{\mathcal M,t}(y)p_{\mathcal M,t}(y_{>t}\mid y),
\]
we obtain
\[
\log
\frac{
\widetilde p_{t,v}(y,Y_{>t})
}{
p_{\mathcal M,t}(y,Y_{>t})
}
=
\log
\frac{
m_{t,v}(y)
}{
m_{\mathcal M,t}(y)
}
+
\log
\frac{
K_{t,v}(Y_{>t}\mid y)
}{
p_{\mathcal M,t}(Y_{>t}\mid y)
}.
\]
Taking expectation over \(Y_{>t}\sim K_{t,v}(\cdot\mid y)\) yields
\[
\psi_t(v,y)
=
\log
\frac{
m_{t,v}(y)
}{
m_{\mathcal M,t}(y)
}
+
\KL\!\left(
K_{t,v}(\cdot\mid y)
\,\middle\|\,
p_{\mathcal M,t}(\cdot\mid y)
\right).
\]
\end{proof}

\begin{proof}[Proof of Claim~\ref{claim::mi-score-identity}]
By the definition of conditional mutual information under the native masked-history
one-step-then-rollback suffix law,
\[
I_{p_\theta}(V;Y_{\ge t}\mid \widetilde h_t)
=
\mathbb E_{V\mid \widetilde h_t}
\left[
\sum_{y_{\ge t}}
\widetilde p_{t,V}(y_{\ge t})
\log
\frac{
\widetilde p_{t,V}(y_{\ge t})
}{
p_{\mathcal M,t}(y_{\ge t})
}
\right].
\]
Using
\[
\widetilde p_{t,V}(y_{\ge t})
=
m_{t,V}(y_t)K_{t,V}(y_{>t}\mid y_t),
\]
we rewrite this as
\[
I_{p_\theta}(V;Y_{\ge t}\mid \widetilde h_t)
=
\mathbb E_{V\mid \widetilde h_t}
\mathbb E_{Y_t\sim m_{t,V}}
\mathbb E_{Y_{>t}\sim K_{t,V}(\cdot\mid Y_t)}
\left[
\log
\frac{
\widetilde p_{t,V}(Y_t,Y_{>t})
}{
p_{\mathcal M,t}(Y_t,Y_{>t})
}
\right].
\]
The inner conditional expectation is exactly \(\psi_t(V,Y_t)\). Therefore,
\[
I_{p_\theta}(V;Y_{\ge t}\mid \widetilde h_t)
=
\mathbb E_{V\mid \widetilde h_t}
\mathbb E_{Y_t\sim m_{t,V}}
\left[
\psi_t(V,Y_t)
\right].
\]
\end{proof}

\section{Experimental Details of Main Results}
\label{app:exp_settings}

\begin{table*}[h]
\centering
\caption{\textbf{Main results with standard deviation (mean@8 acc \%).}}
\footnotesize
\setlength{\tabcolsep}{0.7pt}
\begin{tabular}{clccccccc}
\toprule
\multirow{2}{*}{\textbf{Base Model}} 
& \multirow{2}{*}{\textbf{Model}} 
& \multicolumn{4}{c}{\textbf{Reasoning-Intensive}} 
& \multicolumn{2}{c}{\textbf{General-Domain}} 
& \multirow{2}{*}{\textbf{Avg.}} \\
\cmidrule(lr){3-6} \cmidrule(lr){7-8}
& & MathVision & MathVerse & EMMA & LogicVista & MMMU-Pro & RealWorldQA & \\

\midrule
\multirow{5}{*}{\makecell{\textbf{Qwen3-VL-8B}\\\textbf{Instruct}}}
& Base  & 41.74\err{1.21} & 50.21\err{0.87} & 38.34\err{0.25} & 62.28\err{0.04} & 37.59\err{0.20} & 67.20\err{0.55} & 49.56 \\
& + GRPO                & \underline{48.03}\err{0.14} & 55.23\err{0.49} & \underline{41.34}\err{0.41} & 59.51\err{0.96} & 48.50\err{0.60} & 67.66\err{0.62} & 53.38 ($\uparrow$ 3.82) \\
& + PAPO$_G$            & 43.75\err{0.29} & 54.57\err{0.15} & 39.00\err{0.11} & \underline{64.29}\err{0.32} & \underline{50.00}\err{0.50} & \underline{69.41}\err{0.74} & 53.50 ($\uparrow$ 3.94) \\
& + VPPO$_G$            & \underline{48.03}\err{0.27} & \underline{56.35}\err{0.38} & 36.00\err{0.23} & \textbf{66.07}\err{0.38} & 49.16\err{0.52} & 67.48\err{0.17} & \underline{53.85} ($\uparrow$ 4.29) \\
& \cellcolor{blue!8}\textbf{+ RAPO$_G$}   
& \cellcolor{blue!8}\textbf{48.27}\err{0.19} & \cellcolor{blue!8}\textbf{56.72}\err{0.09} & \cellcolor{blue!8}\textbf{42.50}\err{0.44} & \cellcolor{blue!8}62.50\err{0.86} & \cellcolor{blue!8}\textbf{51.10}\err{0.25} & \cellcolor{blue!8}\textbf{69.54}\err{0.18} & \cellcolor{blue!8}\textbf{55.11} ($\uparrow$ 5.55) \\

\midrule

\multirow{5}{*}{\makecell{\textbf{Qwen3-VL-2B}\\\textbf{Instruct}}}
& Base   & 27.30{\tiny $\pm$\,1.67} & 36.75\err{0.75} & 26.87\err{1.72} & 47.13\err{1.16} & 24.44\err{0.49} & 60.47\err{0.38} & 37.16 \\
& + GRPO                & 28.29\err{1.71} & 44.35\err{0.87} & 28.50\err{0.30} & \underline{48.71}\err{0.99} & 30.61\err{0.60} & \underline{64.28}\err{0.85} & 40.79 ($\uparrow$ 3.63) \\
& + PAPO$_G$            & 28.29\err{0.33} & 45.18\err{0.67} & 27.25\err{0.33} & 46.43\err{0.66} & 31.45\err{0.14} & 64.05\err{0.57} & 40.44 ($\uparrow$ 3.28) \\
& + VPPO$_G$            & \underline{30.59}\err{0.25} & \underline{{45.43}}\err{0.32} & \underline{28.75}\err{0.65} & 45.76\err{1.06} & \underline{31.73}\err{0.32} & 63.14\err{0.44} & \underline{40.90} ($\uparrow$ 3.74) \\
& \cellcolor{blue!8}\textbf{+ RAPO$_G$}   
& \cellcolor{blue!8}\textbf{33.55}\err{1.21} & \cellcolor{blue!8}\textbf{46.13}\err{0.56} & \cellcolor{blue!8}\textbf{31.25}\err{0.71} & \cellcolor{blue!8}\textbf{49.55}\err{0.48} & \cellcolor{blue!8}\textbf{32.20}\err{0.63} & \cellcolor{blue!8}\textbf{65.36}\err{0.40} & \cellcolor{blue!8}\textbf{43.01} ($\uparrow$ 5.85) \\

\midrule
\multirow{6}{*}{\makecell{\textbf{Qwen2.5-VL-7B}\\\textbf{Instruct}}}
& Base  & 21.05\err{0.52} & 29.44\err{0.44} & 23.25\err{0.42} & 42.19\err{0.57} & 21.22\err{0.42} & 54.51\err{0.46} & 31.94 \\
& MM-Eureka         & \underline{32.57}\err{0.27} & 45.18\err{0.45} & {28.75}\err{0.46} & 44.64\err{1.03} & 26.82\err{0.33} & \textbf{61.96}\err{0.65} & 39.99 ($\uparrow$ 8.05) \\
& VL-Rethinker       & \textbf{34.21}\err{0.33} & 45.30\err{0.89} & {28.75}\err{0.77} & 44.20\err{0.34} & 34.10\err{0.23} & {59.22}\err{0.92} & 40.96 ($\uparrow$ 9.02) \\
& R1-ShareVL         & 25.99\err{0.11} & {46.22}\err{0.34} & \textbf{30.50}\err{0.56} & \underline{46.88}\err{0.25} & {34.45}\err{0.28} & 56.47\err{0.10} & 40.09 ($\uparrow$ 8.15) \\
& Base+RAPO$_G$       & 27.96\err{0.33} & \underline{{47.34}}\err{0.71} & \underline{29.00}\err{0.67} & 45.31\err{0.55} & \underline{36.01}\err{0.48} & \underline{60.26}\err{0.13}& \underline{40.98} ($\uparrow$ 9.04) \\
& \cellcolor{blue!8}\textbf{Base+RAPO$_D$}     &\cellcolor{blue!8}30.26\err{0.50} & \cellcolor{blue!8}\textbf{47.59}\err{0.42} & \cellcolor{blue!8}\textbf{30.50}\err{0.99} & \cellcolor{blue!8}\textbf{47.13}\err{0.13} &\cellcolor{blue!8}\textbf{36.04}\err{0.48} &\cellcolor{blue!8}58.40\err{0.77}& \cellcolor{blue!8}\textbf{41.65} ($\uparrow$ 9.71) \\

\midrule
\multirow{5}{*}{\makecell{\textbf{Qwen2-VL-7B}\\\textbf{Instruct}}}
& Base    & 15.79\err{1.16} & 10.25\err{0.65} & 11.75\err{0.46} & 13.17\err{0.94} & 11.73\err{0.45} & 31.83\err{0.51} & 15.75 \\
& TVC                   & 18.75\err{0.82} & 21.98\err{1.10} & 20.75\err{0.45} & \textbf{41.96}\err{0.77} & {23.34}\err{0.87} & {52.88}\err{0.38} & {29.94} ($\uparrow$ 14.19) \\
& MINT-CoT              & 22.04\err{0.19} & 24.62\err{0.23} & 19.00\err{0.66} & 31.47\err{0.71} & 17.63\err{0.86} & 49.15\err{0.12} & 27.32 ($\uparrow$ 11.57) \\
& Base+RAPO$_G$       & \underline{22.37}\err{0.33} & \underline{32.99}\err{0.71} & \underline{25.25}\err{0.40} & 37.72\err{0.55} & \underline{28.22}\err{0.48} & \underline{64.67}\err{0.13} & \underline{35.21 }($\uparrow$ 19.46) \\
& \cellcolor{blue!8}\textbf{Base+RAPO$_D$}     & \cellcolor{blue!8}\textbf{23.36}\err{0.87}& \cellcolor{blue!8}\textbf{36.31}\err{0.66} & \cellcolor{blue!8}\textbf{28.75}\err{0.56} & \cellcolor{blue!8}\underline{39.73}\err{1.52} & \cellcolor{blue!8}\textbf{29.39}\err{0.86} & \cellcolor{blue!8}\textbf{65.01}\err{0.35} & \cellcolor{blue!8}\textbf{37.09} ($\uparrow$ 21.34) \\

\bottomrule
\end{tabular}
\label{tab:std_main_results}
\vspace{-0.09cm}
\end{table*}

\subsection{Training Details}
\label{app:training}
We adopt ViRL39K\footnote{\url{https://huggingface.co/datasets/TIGER-Lab/ViRL39K}} as the RL training dataset, consisting of 38{,}870 curated and verifiable multimodal QA pairs. All training samples are formatted using the prompt template illustrated in the box below.
\begin{tcolorbox}[
    colframe=boxborder,     
    colback=boxbg,          
    boxrule=2pt,            
    arc=12pt,                
    auto outer arc,       
    left=15pt, right=15pt, top=15pt, bottom=15pt, 
    fontupper=\fontfamily{ptm}\selectfont 
]
    { \textbf{\textcolor{boxborder}{CoT Training Data Template:}}}

    \verb|<image>|
   
    \verb|{Question}| \\ \\ You FIRST think about the reasoning process as an internal monologue and then provide the final answer. The reasoning process MUST BE enclosed within \verb|<think>| \verb|</think>| tags. The final answer MUST BE put in \verb|\boxed{}|.
    
\end{tcolorbox}

We implement GRPO and RAPO using the VeRL framework\footnote{\url{https://github.com/volcengine/verl}},
and PAPO and VPPO using Easy-R1\footnote{\url{https://github.com/hiyouga/EasyR1}},
following the training settings reported in their original papers. In the main experiments, among the baselines we only conduct RL training for PAPO and VPPO. Both methods follow the original training settings reported in their respective papers, with minor adjustments to the \texttt{batch\_size}, \texttt{training steps}, \texttt{max\_prompt\_length}, and \texttt{max\_response\_length} to ensure consistency with RAPO when training on Qwen3-VL-2B-Instruct and Qwen3-VL-8B-Instruct. For VPPO, we adopt the same training settings as provided in its official code repository to train Qwen3-VL-8B-Instruct. For PAPO, we follow the settings used for Qwen2.5-VL-3B-Instruct and Qwen2.5-VL-7B to train Qwen3-VL-2B-Instruct and Qwen3-VL-8B-Instruct, respectively.  The reward signal used in the main experiments, as well as in all additional experiments, is a binary accuracy score, where a correct answer receives a reward of 1 and an incorrect answer receives a reward of 0. In addition, the vision tower is frozen during training for all RL experiments. For computational resources, Qwen3-VL-2B-Instruct and Qwen3-VL-8B-Instruct are trained on 8$\times$ NVIDIA RTX PRO 6000 Blackwell GPUs, while Qwen2-VL-7B-Instruct and Qwen2.5-VL-7B-Instruct are trained on 4 GPUs. All key hyperparameters for RL training are summarized in Table~\ref{tab:RL_setting}.

\begin{table}[t]
\centering
\small
\caption{RL hyperparameters for different experimental groups.}
\label{tab:RL_setting}
\begin{tabular}{lcc}
\toprule
\textbf{Hyperparameter} & \textbf{Group (i)} & \textbf{Group (ii) \& (iii)} \\
\midrule

\rowcolor{blue!8}
\multicolumn{3}{l}{\textbf{General Training}} \\
\midrule
Learning rate              & $1e{-6}$ & $1e{-6}$ \\
RL steps                   & 100              & 200              \\
Optimizer                  & AdamW            & AdamW            \\
LR scheduler               & Constant         & Constant         \\
Freeze vision tower        & True             & True             \\

\midrule
\rowcolor{blue!8}
\multicolumn{3}{l}{\textbf{Basic GRPO Settings}} \\
\midrule
Rollout batch size         & 256              & 384              \\
Rollouts per prompt        & 5                & 5                \\
PPO minibatch size         & 64               & 128              \\
Rollout temperature        & 1.0              & 1.0              \\
Rollout top-$p$            & 1.0              & 1.0              \\
Max prompt length          & 2048             & 4096             \\
Max response length        & 8192             & 2048             \\
KL penalty coefficient     & 0.02             & 0.02             \\
Loss aggregation mode      & Token-level   & Token-level   \\

\midrule
\rowcolor{blue!8}
\multicolumn{3}{l}{\textbf{DAPO Recipe}} \\
\midrule
Dynamic sampling           & True             & True             \\
Clip ratio (low)           & 0.20             & 0.20             \\
Clip ratio (high)          & 0.28             & 0.28             \\
Overlong reward shaping    & True             & False            \\
Overlong buffer length     & 1024             & --               \\
Overlong penalty factor    & 1.0              & --               \\
Loss aggregation mode      & Token-level      & Token-level      \\

\midrule
\rowcolor{blue!8}
\multicolumn{3}{l}{\textbf{RAPO-Specific Settings}} \\
\midrule
MI enhancement factor $\gamma$  & 0.01         & 0.01         \\
Token selection threshold      & Top 20\%     & Top 20\%     \\
Window length                  & 3 (1 for 8B)           & 1            \\

\bottomrule
\end{tabular}

\end{table}

\subsection{Evaluation Details}
\textbf{Benchmarks} We evaluate our method on two categories of benchmarks: reasoning-intensive and general-domain, enabling comprehensive assessment of VLM capabilities. All evaluations are based on directly verifiable answers rather than LLM-as-a-judge evaluation. Detailed benchmark statistics are summarized in Table~\ref{tab:benchmarks}.

\begin{table}[!h]
\centering
\small
\setlength{\tabcolsep}{5pt}
\caption{Evaluation benchmarks used in our experiments.}
\label{tab:benchmarks}
\begin{tabular}{lccccc}
\toprule
\textbf{Benchmark} & \textbf{Size} & \textbf{Domain} & \textbf{Task Type} & \textbf{Subset} & \textbf{Source} \\
\midrule
MathVision
& 304
& Math
& Free-form QA
& Test-mini
& \href{https://huggingface.co/datasets/MathLLMs/MathVision}{HuggingFace} \\

MathVerse
& 788
& Math
& Multiple-choice \& Free-form QA
& Visual-dominant
& \href{https://huggingface.co/datasets/AI4Math/MathVerse}{HuggingFace} \\

EMMA-mini
& 400
& STEM
& Multiple-choice QA
& Full
& \href{https://huggingface.co/datasets/luckychao/EMMA-mini}{HuggingFace} \\

LogicVista
& 448
& Logic
& Multiple-choice QA
& Full
& \href{https://huggingface.co/datasets/lscpku/LogicVista}{HuggingFace} \\

MMMU-Pro
& 1{,}730
& General
& Multiple-choice QA
& Vision
& \href{https://huggingface.co/datasets/MMMU/MMMU_Pro}{HuggingFace} \\

RealWorldQA
& 765
& General
& Free-form QA
& Full
& \href{https://huggingface.co/datasets/lmms-lab/RealWorldQA}{HuggingFace} \\
\bottomrule
\end{tabular}

\end{table}

\textbf{Evaluation Settings.} 
Benchmark evaluation is conducted using vLLM v0.11.0\footnote{\url{https://github.com/vllm-project/vllm}}. Multiple-choice questions are evaluated via exact-match accuracy, while free-form questions are verified using math-verify\footnote{\url{https://github.com/huggingface/Math-Verify}}. The prompt template used during evaluation is identical to that employed for training. During evaluation, we set the sampling temperature to 1.0 with $\text{top-}p=1.0$) and top-$k$ set to $-1$. Inference is performed with a batch size of 512. We enable the vLLM V1 execution mode (\texttt{VLLM\_USE\_V1=True}) and the maximum number of newly generated tokens is set to 8{,}192. We report mean and standard deviation over 8 sampled evaluation responses per
input in Table~\ref{tab:std_main_results}.

\section{Experimental Details of Empirical Analysis of RAPO's Mechanism}
\label{app:why}
We randomly sample 500 instances from the ViRL dataset and evaluate the Qwen3-VL-2B-Instruct model before training and after GRPO and RAPO training, using the same settings as in the main experiments. Responses are generated via greedy decoding with a maximum length of 4096 tokens. We conduct three groups of experiments.

\subsection{Trajectory Analysis}
In this subsection, we visualize the KL divergence and attention distributions. The KL divergence is used to quantify \emph{token-level visual dependency} by measuring the Kullback--Leibler (KL) divergence between the policy’s predictive distributions under vision-conditioned and vision-masked settings.
Formally, let \( V \) denote the visual input, \( X \) the textual context, and \( Y_{1:T} = (Y_1, \ldots, Y_T) \) a length-\(T\) reasoning trajectory. We define the history at decoding step \( t \) as
\(
H_t = (X, Y_{<t}) .
\)
The token-level visual dependency at step \( t \) is defined as the KL divergence between the policy output distributions conditioned on \( (V, H_t) \) and on \( H_t \) alone:
\[
\mathcal{S}(H_t, V)
= {\mathbb{D}_{\mathrm{KL}}}\!\left(
\pi_{\theta}(\cdot \mid H_t, V)
\;\|\;
\pi_{\theta}(\cdot \mid H_t)
\right).
\]
A high value of \( \mathcal{S}(H_t, V) \) indicates that visual information plays a critical role in predicting the token at step \( t \), identifying a key moment of visually grounded reasoning.
We further visualize the attention distribution from the next generated token to visual tokens. 
\begin{equation*}
a_t^{\mathrm{vis}}
=
\sum_{j\in\mathcal V}
\left[
\operatorname{softmax}
\left(
\frac{q_t K_{\mathrm{all}}^\top}
{\sqrt d}
\right)
\right]_j,
\end{equation*}
where $\mathcal V$ denotes the set of
visual-token positions and $K_{\mathrm{all}}$
contains all accessible visual and textual keys.


We use Qwen3-VL-2B-Instruct to generate responses for 500 questions, which are divided into two groups based on correctness: incorrect (322) and correct (178). For each group, we compute the token-wise average of the KL divergence and attention distributions. For training comparison, we average the KL distributions over all 500 instances to compare RAPO and GRPO. This analysis shows that RAPO effectively increases visual dependency in the later stages of the chain-of-thought (CoT), whereas GRPO does not. As shown in Figures~\ref{fig:kl_base}, GRPO fails to improve the KL divergence relative to the base model, while RAPO leads to a clear enhancement.

\begin{figure}[ht]
  \centering
  \begin{subfigure}[t]{0.48\linewidth}
    \centering
    \includegraphics[width=\linewidth]{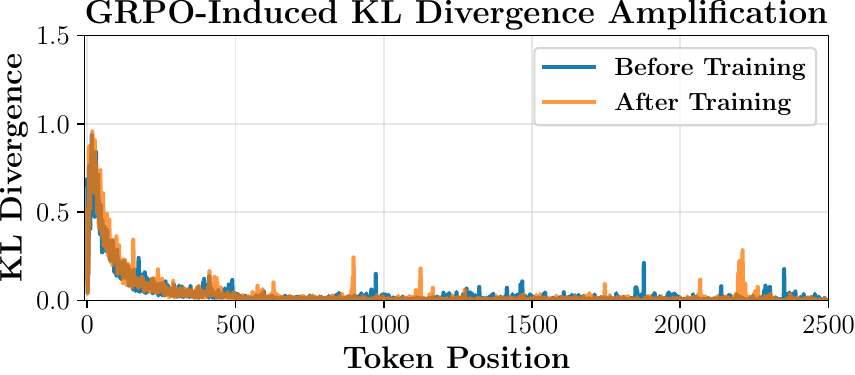}
    \label{fig:kl_grpo_base}
  \end{subfigure}
   \hfill
  \begin{subfigure}[t]{0.48\linewidth}
    \centering
    \includegraphics[width=\linewidth]{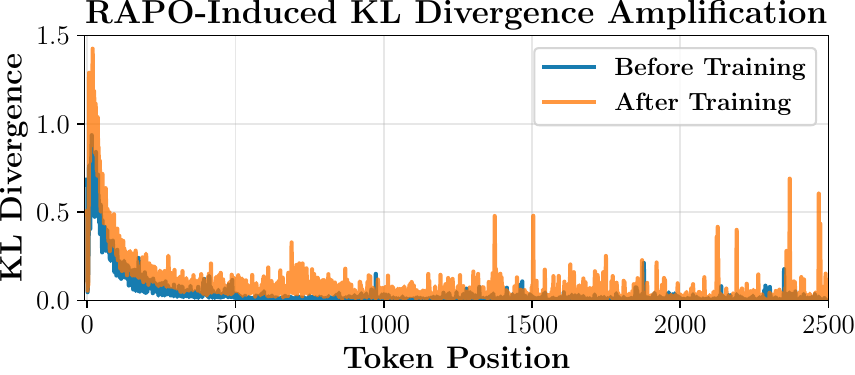}
    \label{fig:kl_rapo_base}
  \end{subfigure}
  \caption{\textbf{KL Distributions along CoT.} (a) KL distribution after GRPO training; (b) KL distribution after RAPO training.}
  \label{fig:kl_base}
\end{figure}

\subsection{Anchor Token Analysis}
For the anchor-token concentration visualization, we identify the top-20\%
highest-entropy tokens across all 500 responses generated by
Qwen3-VL-2B-Instruct. To avoid rare-token artifacts, we restrict the
analysis to token types appearing at least $N$ times in the corpus
($N=10$ in our experiments), and report for each retained type the fraction
of its occurrences selected as a top-entropy anchor.

We construct a noised visual input \( \tilde V \) by injecting additive Gaussian noise:
\[
\tilde V = V + \varepsilon,
\quad
\varepsilon \sim \mathcal{N}(0, \sigma^2 I),
\]
where \( \sigma \) controls the noise magnitude and \( I \) denotes the identity covariance matrix with the same dimensionality as \( V \). We then analyze the distributional shifts of different token categories under varying noise levels \( \sigma \). Specifically, we use
\[
{\mathbb{D}_{\mathrm{KL}}}\!\left(
\pi_{\theta}(\cdot \mid H_t, V)
\;\|\;
\pi_{\theta}(\cdot \mid H_t, \tilde V)
\right)
\]
to measure token sensitivity to visual perturbations, where larger KL values indicate stronger dependence on visual information.

\subsection{Propagation of Visual Influence Analysis}
We design a targeted experiment to measure \emph{visual information propagation from anchor token positions}. Given a reflection anchor \( t_k \in \mathcal{A} \), we evaluate how blocking its visual information, implemented by masking the corresponding visual attention, affects the visual dependency of other tokens. Specifically, we examine (i) the average impact on the subsequent three tokens \( Y_{t_k:t_k+3} \), and (ii) the impact on the visual dependency of the next anchor token \( t_{k+1} \). To quantify this effect, we define
\[
\Delta\KL = \mathcal{S}(H'_t, V) - \mathcal{S}(H_t, V),
\]
where $\mathcal{S}(H_t, V) = \KL(\pi_\theta(\cdot|H_t,V)\,\|\,\pi_\theta(\cdot|H_t))$ measures the visual dependence at step $t$ under the original history, and $\mathcal{S}(H'_t, V)$ is the analogous quantity after masking the visual information associated with the selected anchor. By construction, $\Delta D_{\mathrm{KL}} \le 0$: more negative values indicate that masking the anchor produces a larger drop in downstream visual dependence, i.e., stronger propagation of visual influence from the anchor.

 We evaluate this effect across anchor tokens at different positions and compare RAPO with baseline methods. The results demonstrate that RAPO significantly enhances both the visual influence at anchor tokens and the propagation of visual information along the reasoning trajectory.

\section{Intervention-Based Diagnostics for Reflection Anchors}
\label{app:causal_sanity}

We provide an intervention-based diagnostic to examine whether reflection anchors are associated
with task-relevant visual evidence. This diagnostic does not directly estimate the oracle gain
$\Delta_t(\pi)$, but tests whether the visual regions attended at reflection anchors have measurable
influence on final prediction.

\subsection{Outcome-Level Patch Masking}
\textbf{Attention-guided patch masking.}
We test whether early anchors attend to task-relevant visual evidence. On 100 randomly sampled
ViRL validation examples, we select early reflection anchors and mask $20\%$ of image patches under
two strategies: attention-guided masking, which removes the highest-attended patches at the selected
anchor, and random masking, which removes the same number of randomly sampled patches. We
report the accuracy drop $\Delta\mathrm{Acc}$.

\begin{table}[h]
\centering
\caption{\textbf{Attention-guided patch masking.}
More negative values indicate larger accuracy degradation.}
\label{tab:causal_sanity}
\small
\begin{tabular}{lccc}
\toprule
\textbf{Model} & \textbf{Attention-guided} & \textbf{Random} & \textbf{Gap} \\
\midrule
GRPO & -24.3 & -19.7 & 4.6 \\
RAPO & -27.7 & -21.3 & 6.4 \\
\bottomrule
\end{tabular}
\end{table}

As shown in Table~\ref{tab:causal_sanity}, attention-guided masking causes larger accuracy
degradation than random masking for both GRPO and RAPO. This suggests that early reflection
anchors attend to image regions that are more relevant to final prediction than randomly selected
patches. The larger gap under RAPO is consistent with the intended effect of strengthening the
coupling between reflection anchors and task-relevant visual evidence. Together with the KL
propagation analysis in Figure~\ref{fig:delta_kl}, this diagnostic provides intervention-based support
for the proposed mechanism.

\subsection{Visual-Specific Anchor Interventions}
\label{app:visual_specific_anchors}
We further examine whether entropy-selected
reflection anchors correspond to visually
consequential reasoning positions.Using the frozen native policy, we conduct
a paired intervention study on 100 matched
samples from each of MathVerse's Text Dominant
and Vision Only subsets. For each selected anchor, we replay the same
prefix, suppress visual access at the anchor
position, restore normal access afterward,
and generate the remaining response.
We compare the resulting accuracy with
a position-matched control intervention. The additional anchor-specific effect is defined as
\(
G =
\Delta_{\mathrm{anchor}}
-
\Delta_{\mathrm{control}},
\)
where each $\Delta$ denotes the accuracy
degradation relative to clean decoding.

Table~\ref{tab:anchor_interventions} shows that the additional effect of anchor intervention
is substantially stronger on Vision Only questions
than on Text Dominant questions.
This provides functional evidence that
entropy-selected anchors identify branching
points with visual-specific influence.

\begin{table}[h]
\centering
\caption{\textbf{Anchor-specific intervention effects on
MathVerse.} Accuracy and intervention effects
are reported in percentage points.}
\label{tab:anchor_interventions}
\begin{tabular}{lrrrrrr}
\toprule
Subset & Clean & Anchor & Control
& $\Delta_a$ & $\Delta_c$ & $G$ \\
\midrule
Text Dominant & 34.0 & 31.5 & 31.0
& 2.5 & 3.0 & -0.5 \\
Vision Only & 24.0 & 18.6 & 20.8
& 5.4 & 3.2 & 2.2 \\
\bottomrule
\end{tabular}
\end{table}

\subsection{Robustness Across Visual Counterfactuals}
\label{app:visual_counterfactuals}

We evaluate whether the visual-dependence
patterns identified by chain masking also
appear under alternative counterfactual
interventions in image space. Using frozen Base, GRPO, and RAPO checkpoints,
we conduct teacher-forced evaluations on
the same 500-example ViRL validation subset. We consider four visual interventions:
chain masking, full-image masking,
Gaussian image noise, and question-guided
region removal.

Table~\ref{tab:counterfactuals} shows that all four interventions reproduce the same
qualitative ordering: correct trajectories
exhibit stronger visual contrast than
incorrect trajectories, and RAPO retains
stronger visual-dependence signals than GRPO. These consistent patterns across attention-level
and pixel-space interventions demonstrate the
robustness of the observed visual-retention effect.

\begin{table}[h]
\centering
\caption{\textbf{Visual contrast under alternative
counterfactual interventions.}}
\label{tab:counterfactuals}
\begin{tabular}{lcc}
\toprule
Intervention & Correct / Incorrect & RAPO / GRPO \\
\midrule
Chain mask
& 0.7612 / 0.6329 & 1.0250 / 0.7648 \\
Full-image mask
& 0.1000 / 0.0770 & 0.1940 / 0.1047 \\
Gaussian noise ($\sigma=0.5$)
& 0.0292 / 0.0258 & 0.0589 / 0.0317 \\
Question-guided region removal
& 0.0219 / 0.0147 & 0.0476 / 0.0228 \\
\bottomrule
\end{tabular}
\end{table}

\section{Compute-Matched Comparison}
\label{app:compute_matched}

RAPO introduces an additional chain-masked contrastive KL signal during training, which increases per-step training cost relative to GRPO. 
To verify that the performance gain is not merely due to additional computation, we first report wall-clock training cost and then compare RAPO against a compute-matched GRPO baseline trained for additional steps under a similar wall-clock budget. 
All measurements are conducted on \(8\times\) NVIDIA RTX PRO 6000 Blackwell GPUs.

\begin{table}[h]
\centering
\caption{\textbf{Training wall-clock cost.}
RAPO incurs additional training overhead due to the chain-masked contrastive KL computation, but it uses the standard autoregressive forward pass at inference and therefore adds no inference-time overhead.}
\label{tab:wall_clock_cost}
\small
\setlength{\tabcolsep}{4pt}
\begin{tabular}{lccccc}
\toprule
\textbf{Model} & \textbf{GRPO Total} & \textbf{RAPO Total} & \textbf{GRPO / Step} & \textbf{RAPO / Step} & \textbf{Overhead} \\
\midrule
Qwen3-VL-2B & 10h38m & 12h31m & 383s & 451s & \(\sim\)18\% \\
Qwen3-VL-8B & 28h04m & 35h58m & 1010s & 1295s & \(\sim\)28\% \\
\bottomrule
\end{tabular}
\end{table}

We then train GRPO for 120 steps to approximately match RAPO's 100-step wall-clock budget on Qwen3-VL-2B-Instruct. 
The GRPO-120 run takes 12h43m, comparable to RAPO-100 at 12h31m. 
As shown in Table~\ref{tab:compute_matched}, increasing GRPO from 100 to 120 steps improves the average score from 40.79 to 41.45, but still remains below RAPO's 43.01. 
This indicates that RAPO's gain is not explained by additional training computation, but by the proposed anchor-based propagation objective.

\begin{table}[h]
\centering
\caption{\textbf{Compute-matched comparison.}
GRPO is trained for additional steps to match RAPO's wall-clock budget. RAPO still achieves higher average performance.}
\label{tab:compute_matched}
\small
\setlength{\tabcolsep}{4pt}
\begin{tabular}{lcccccccc}
\toprule
\textbf{Method} & \textbf{Steps} & \textbf{MathVis.} & \textbf{MathVer.} & \textbf{EMMA} & \textbf{LogicV.} & \textbf{MMMU-P.} & \textbf{RealWQ} & \textbf{Avg.} \\
\midrule
GRPO & 100 & 28.29 & 44.35 & 28.50 & 48.71 & 30.61 & 64.28 & 40.79 \\
GRPO & 120 & 30.92 & 44.73 & 31.25 & 49.11 & 31.20 & 61.50 & 41.45 \\
RAPO$_G$ & 100 & \textbf{33.55} & \textbf{46.13} & \textbf{31.25} & \textbf{49.55} & \textbf{32.20} & \textbf{65.36} & \textbf{43.01} \\
\bottomrule
\end{tabular}
\end{table}

RAPO's additional cost is training-only: at inference, the learned policy runs with the same standard autoregressive decoding procedure as the base model and does not require anchor search, chain masking, or additional visual re-attention modules.

\section{Ablation and Robustness Study}
\label{app:abaltion}

\subsection{Experimental Details}
We additionally conduct experiments on two other models: Qwen3-VL-2B-Thinking and Gemma3-4B-IT. 
Qwen3-VL-2B-Thinking is designed to generate substantially longer CoT, which allows us to further validate the effectiveness of our method under long-CoT reasoning settings. 
For this model, we adopt the same training configuration as Qwen3-VL-2B-Instruct used in the main experiments, except that the maximum response length is increased to 16{,}384 tokens. 
The evaluation protocol is also kept identical, with the maximum number of newly generated tokens set to 16{,}384. Gemma3-4B-IT is a vision--language model with a substantially different architecture from the Qwen3-VL family. 
To assess the generality of our method across model architectures, we evaluate it under the same experimental settings as Qwen3-VL-Instruct, except that the maximum response length is set to 4{,}096 tokens.

Unless otherwise specified, all remaining ablation studies are conducted under the same training and evaluation configuration as the Qwen3-VL-2B-Instruct experiments reported in the main results.

\subsection{Ablation on Anchor Selection}
\label{app:anchor_aba}
We further analyze how the anchor set \(\mathcal A\) should be selected. The results are shown in Table~\ref{tab:anchor_selection_ablation}.

\begin{table}[h]
\centering
\small
\setlength{\tabcolsep}{4pt}
\caption{
Ablation on anchor-selection strategies.
}
\begin{tabular}{lcccccccc}
\toprule
Strategy & Frac. & MathVis. & MathVer. & EMMA & LogicV. & MMMU-P. & RealWQ & Avg. \\
\midrule
No anchors (GRPO) & 100\% & 28.29 & 44.35 & 28.50 & 48.71 & 30.61 & 64.28 & 40.79 \\
Random 20\%       & 20.0\% & 29.61 & 42.93 & 26.56 & 48.88 & 30.81 & 60.98 & 39.96 \\
Noun              & 28.7\% & 27.30 & 43.94 & 28.56 & 46.88 & 31.36 & 62.09 & 40.02 \\
Fixed 100          & --     & 29.93 & 45.11 & 29.75 & \textbf{50.39} & 31.33 & 64.58 & 41.85 \\
Fixed 200          & --     & 30.26 & 43.50 & 25.75 & 45.76 & 30.48 & 63.14 & 39.82 \\
Outlier \((\mu+\sigma)\) & 16.0\% & \textbf{34.21} & {45.69} & 26.50 & 47.54 & 30.14 & 63.01 & 41.18 \\
\rowcolor{blue!8}Top-20\% entropy   & 20.0\% & 33.55 & \textbf{46.13} & \textbf{31.25} & 49.55 & \textbf{32.20} & \textbf{65.36} & \textbf{43.01} \\
\bottomrule
\end{tabular}

\label{tab:anchor_selection_ablation}
\end{table}

Table~\ref{tab:anchor_selection_ablation} shows that anchor selection is crucial.
Unrestricted GRPO-style updating without anchors achieves \(40.79\) average, while applying the objective to Top-20\% entropy anchors improves the score to \(43.01\).
Random and noun-based anchors perform poorly (\(39.96\) and \(40.02\)), even though noun selection uses a larger token fraction than Top-20\% entropy selection (\(28.7\%\) vs. \(20.0\%\)).
This suggests that effective anchors are not simply visually named content words, but positions with high structural uncertainty and visual leverage.

Top-20\% entropy selection is also more robust than fixed or threshold-based alternatives.
It outperforms Fixed-100 and Fixed-200 anchors (\(43.01\) vs. \(41.85\) and \(39.82\)), indicating that a fraction-based rule better adapts to variable sequence lengths.
It also outperforms the fixed outlier criterion \(c_t>\mu+\sigma\) (\(43.01\) vs. \(41.18\)), supporting our use of a per-trajectory Top-\(\rho\) criterion for sparse and adaptive anchor selection.

\subsection{Robustness on Decoding Temperature}
We further evaluate whether RAPO relies primarily on stochastic sampling or also improves deterministic decoding. 
Specifically, we compare greedy decoding with temperature $0$ and pass@1 against stochastic decoding with temperature $1$ and mean@8 on Qwen3-VL-2B-Instruct.

\begin{table}[h]
\centering
\small
\setlength{\tabcolsep}{4pt}
\caption{Robustness to decoding temperature.}
\begin{tabular}{llccccccc}
\toprule
Method & Decoding & MathVision & MathVerse & EMMA & LogicVista & MMMU-Pro & RealWorldQA & Avg \\
\midrule
Base & greedy & 13.16 & 26.78 & 14.50 & 31.25 & 18.90 & 58.82 & 27.24 \\
Base & temp=1 & 27.30 & 36.75 & 26.87 & 47.13 & 24.44 & 60.47 & 37.16 \\
\midrule
GRPO & greedy & 11.51 & 27.41 & 13.75 & 32.81 & 22.49 & 59.22 & 27.87 \\
GRPO & temp=1 & 28.29 & 44.35 & 28.50 & 48.71 & 30.61 & 64.28 & 40.79 \\
\midrule
RAPO$_G$ & greedy & \textbf{16.12} & \textbf{33.63} & \textbf{15.50} & \textbf{33.71} & \textbf{25.55} & \textbf{63.92} & \textbf{31.41} \\
RAPO$_G$ & temp=1 & \textbf{33.55} & \textbf{46.13} & \textbf{31.25} & \textbf{49.55} & \textbf{32.20} & \textbf{65.36} & \textbf{43.01} \\
\bottomrule
\end{tabular}
\label{tab:robust_decoding}
\end{table}

GRPO only marginally improves greedy decoding over the base model, from 27.24 to 27.87, while its larger gain appears under stochastic decoding. 
In contrast, RAPO improves both settings, reaching 31.41 under greedy decoding and 43.01 under temperature-$1$ sampling. 
Moreover, the greedy--stochastic gap is smaller for RAPO, 11.46, than for GRPO, 12.92, suggesting that RAPO learns a more stable policy rather than depending mainly on stochastic sampling.

\subsection{Robustness on Context Length}
Finally, we evaluate robustness under a longer context budget. 
We set the maximum context length to 16,384 tokens and compare GRPO with RAPO.
\begin{table}[h]
\centering
\small
\caption{Robustness to longer context length.}
\setlength{\tabcolsep}{4pt}
\begin{tabular}{lccccccc}
\toprule
Model & MathVision & MathVerse & EMMA & LogicVista & MMMU-Pro & RealWorldQA & Avg \\
\midrule
GRPO & 31.6 & 44.0 & 28.0 & 51.3 & 31.2 & 62.2 & 41.4 \\
RAPO & \textbf{34.5} & \textbf{46.2} & \textbf{31.5} & \textbf{51.6} & \textbf{33.0} & \textbf{65.8} & \textbf{43.8} \\
\bottomrule
\end{tabular}
\label{tab:robust_context}
\end{table}
Under the longer-context setting, RAPO improves the average score from 41.4 to 43.8. 
This result shows that RAPO remains effective when the context budget is increased and that its gains are not an artifact of a particular maximum generation length.

\subsection{Adaptive Window Strategies}
\label{app:adaptive_windows}

We compare the fixed-window propagation objective
with two adaptive window selection strategies (Table~\ref{tab:adaptive_windows}).
\textbf{Entropy-gated window.}
For each anchor, the window extends until the
first subsequent token whose entropy falls
below the trajectory median.
\textbf{Inter-anchor window.}
Each window extends from the current anchor
up to, but excluding, the next selected anchor.
Both adaptive variants normalize the KL
objective by their resulting window lengths. The fixed three-token window achieves the best overall accuracy, particularly on reasoning-intensive
benchmarks. This result supports the use of
a compact propagation horizon for retaining
locally relevant visual influence.

\begin{table}[t]
\centering
\caption{Comparison of fixed and adaptive
propagation windows on Qwen3-VL-2B-Instruct.}
\label{tab:adaptive_windows}
\begin{tabular}{lrrrr}
\toprule
Window & Mean Length & Reas. & Gen. & Avg. \\
\midrule
Fixed ($w=3$) & 3.00 & 40.12 & 48.78
& $\mathbf{43.01}$ \\
Entropy-gated & 4.85 & 38.09 & 47.79 & 41.32 \\
Inter-anchor & 4.99 & 38.67 & 48.76 & 42.04 \\
\bottomrule
\end{tabular}
\end{table}

\section{RAPO based on DAPO}
\subsection{RAPO$_D$ Optimization Objective}
To extend DAPO with RAPO, we incorporate the same
\(\overline{\mathbb{D}_{\mathrm{KL}}}_{i,t}(\theta) \) used in RAPO$_G$ (Section~\ref{sec:method}).The complete RAPO$_D$ objective is shown as follows:
\[
\begin{aligned}
\mathcal{J}_{\mathrm{RAPO}_D}(\theta)
&=
\mathbb{E}_{\{o_i\}_{i=1}^G \sim \pi_{\theta_{\mathrm{old}}}(O \mid q, I),t \in \mathcal{A}_i, \mathcal{A}_i \subset o_i}
\Bigg[
\frac{1}{\sum_{i=1}^G |\mathcal{A}_i|}
\sum_{i=1}^G \sum_{t\in\mathcal A_i}
\Big\{
\\
&\qquad
\min\!\Big[
r_{i,t}(\theta)\,\hat A_{i,t},
\;
\mathrm{clip} \!\big(r_{i,t}(\theta),\, 1-\epsilon_l,\, 1+\epsilon_h\big)\,
\hat A_{i,t}
\Big] + \gamma\, \omega^{\mathrm{br}}_{i,t}\,\overline{\mathbb{D}_{\mathrm{KL}}}_{i,t}(\theta) 
\Big\}
\Bigg],
\end{aligned}
\tag{5}
\]
with
\[
0 < \left|\{\, o_i \mid \texttt{is\_equivalent}(a, o_i) \,\}\right| < G .
\]

Here, \( \gamma \) is the weighting coefficient for
\( \omega^{\mathrm{br}}_{i,t}\,\overline{\mathbb{D}_{\mathrm{KL}}}_{i,t}(\theta) \) and $\mathcal{A}_i$ is the set of reflection anchors.

\subsection{Results of DAPO Experiments}
The DAPO results for the pure RL experiment group are reported in Table~\ref{tab:dapo}.
The evaluation settings follow the main experiments, except that the temperature is set to 0.6. Results show that RAPO$_D$ still outperforms the baselines.

\begin{table*}[ht]
\centering
\footnotesize
\caption{\textbf{DAPO Results of pure RL Group.} (mean@8 acc \%). Performance comparison across reasoning-intensive and general-domain multimodal benchmarks. The best results in each group are highlighted in \textbf{bold}. The second-best results are \underline{underlined}. RAPO$_D$, PAPO$_D$ and VPPO$_D$ denote RAPO, PAPO and VPPO built upon DAPO, respectively.}
\label{tab:dapo}
\setlength{\tabcolsep}{3pt}
\renewcommand{\arraystretch}{1.0}
\begin{tabular}{l l c c c c c c c}
\toprule
\multirow{2}{*}{\textbf{Base Model}} & \multirow{2}{*}{\textbf{Model}}
& \multicolumn{4}{c}{\textbf{Reasoning-Intensive}}
& \multicolumn{2}{c}{\textbf{General-Domain}}
& \multirow{2}{*}{\textbf{Avg.}} \\
\cmidrule(lr){3-6} \cmidrule(lr){7-8}
& & MathVision & MathVerse & EMMA & LogicVista
& MMMU-Pro & RealWorldQA & \\
\midrule
\multirow{5}{*}{\makecell{\textbf{Qwen3-VL-8B} \\ \textbf{Instruct}}}
& Base        & 41.74 & 50.21 & 38.34 & 62.28 & 37.59 & 67.20  & 49.56 \\    
& + DAPO      & {50.33} & \textbf{60.95} & 38.75 & 60.94 & 51.85 & 68.30 & 55.19 ($\uparrow$ 5.63) \\
& + PAPO$_D$  & 49.67 & \underline{59.64} & 42.25 & \underline{62.95} & 52.54 & \underline{71.37} & \underline{56.40} ($\uparrow$ 6.84)\\
& + VPPO$_D$  & \underline{50.99} & \underline{59.64} & \underline{43.75} & 61.61 & \underline{52.60} & 66.80 & 55.90 ($\uparrow$ 6.34)\\
& + RAPO$_D$  & \textbf{51.36}  & 58.50 & \textbf{45.50} & \textbf{64.51} & \textbf{52.83} & \textbf{71.63} &\textbf{57.39} ($\uparrow$ 7.83)\\
\bottomrule
\end{tabular}
\end{table*}

\section{Licenses of Existing Assets}
\label{app:asset_licenses}

We summarize the existing external assets used in this work, including codebases,
datasets, benchmarks, and models, together with their publicly stated licenses or
access terms. We cite the original papers and official repositories, dataset cards,
or model cards throughout the paper. All assets are used only for research training,
evaluation, and reproducibility purposes. When an asset is distributed under a
model-specific license or usage policy, we use it in accordance with the corresponding
license agreement, acceptable use policy, and redistribution terms. The results are
presented in Table~\ref{tab:asset_licenses}.

\begin{table}[ht]
\centering
\small
\setlength{\tabcolsep}{4pt}
\caption{Existing assets used in this work and their publicly stated licenses or access terms.}
\begin{tabular}{@{}lll@{}}
\toprule
Asset & Type & License / Terms of Use \\
\midrule

VeRL~\href{https://github.com/volcengine/verl}{Github}
& Code
& Apache-2.0 \\

EasyR1~\href{https://github.com/hiyouga/EasyR1}{Github}
& Code
& Apache-2.0 \\

vLLM~\cite{kwon2023efficient}
& Code
& Apache-2.0 \\

Math-Verify~\href{https://github.com/huggingface/Math-Verify}{Github}
& Code
& Apache-2.0 \\

ViRL39K~\citep{wang2025vlrethinkerincentivizingselfreflectionvisionlanguage}
& Dataset
& MIT \\

MathVision~\cite{wang2024measuring}
& Benchmark
& MIT \\

MathVerse~\cite{zhang2024mathversedoesmultimodalllm}
&  Benchmark
& MIT \\

EMMA~\cite{hao2025mllmsreasonmultimodalityemma}
& Benchmark
& \makecell[l]{No explicit standard dataset license found;\\ research evaluation only} \\

LogicVista~\cite{xiao2024logicvistamultimodalllmlogical}
& Benchmark
& Apache-2.0 \\

MMMU-Pro~\cite{yue2025mmmuprorobustmultidisciplinemultimodal}
& Benchmark
& Apache-2.0 \\

RealWorldQA~\href{https://huggingface.co/datasets/xai-org/RealworldQA}{Huggingface}
& Benchmark
& CC BY-ND 4.0 \\

Qwen3-VL-2B-Instruct~\cite{bai2025qwen3vltechnicalreport}
& Model
& Apache-2.0 \\

Qwen3-VL-8B-Instruct~\cite{bai2025qwen3vltechnicalreport}
& Model
& Apache-2.0 \\

Qwen3-VL-2B-Thinking~\cite{bai2025qwen3vltechnicalreport}
& Model
& Apache-2.0 \\

Qwen2.5-VL-7B-Instruct~\cite{bai2025qwen25vltechnicalreport}
& Model
& Apache-2.0 \\

Qwen2-VL-7B-Instruct~\cite{wang2024qwen2vlenhancingvisionlanguagemodels}
& Model
& Apache-2.0 \\

Gemma3-4B-IT~\cite{gemmateam2025gemma3technicalreport}
& Model
& \makecell[l]{Gemma Terms of Use} \\

MM-Eureka~\cite{meng2025mmeurekaexploringfrontiersmultimodal}
& Model
& Apache-2.0 \\

VL-Rethinker~\cite{wang2025vlrethinkerincentivizingselfreflectionvisionlanguage}
& Model
& Apache-2.0 \\

R1-ShareVL~\cite{yao2025r1sharevlincentivizingreasoningcapability}
& Model
& Apache-2.0 \\

TVC~\cite{sun2025mitigating}
& Model
& Apache-2.0 \\

MINT-CoT~\cite{chen2025mintcot}
& Model
& \makecell[l]{No explicit standard license found;\\ research baseline only} \\

\bottomrule
\end{tabular}
\label{tab:asset_licenses}
\end{table}

\section{Limitations and Broader Impacts}
\label{app:limitations}
\subsection{Limitations}
RAPO provides a practical training-time route to visual-retention optimization
in long-chain multimodal reasoning. The current implementation uses
native-entropy anchors and a chain-masked finite-window KL target, which keeps
the method simple, scalable, and architecture-free at inference, while leaving
richer anchor criteria and counterfactual references to future work. Future
extensions could incorporate semantic or causal anchor signals, adaptive
counterfactual references, and dynamically selected window lengths. Also, RAPO tends to provide larger gains when the base policy exhibits substantial visual forgetting or weak long-chain visual retention. 
On stronger or architecturally different backbones, the gains may be smaller, suggesting that propagation-aware shaping is most useful when visual-dependence signals are not already well preserved. While our
experiments span multiple LVLM backbones and benchmark families, applying
propagation-aware visual-retention objectives to broader real-world visual
reasoning tasks and more fine-grained faithfulness diagnostics remains an
important next step.

\subsection{Broader Impacts}
This work aims to improve visual grounding in long-chain multimodal reasoning. Positive impacts include more reliable reasoning over diagrams, charts, documents, and real-world visual inputs, especially in settings where models must maintain visual evidence across long reasoning chains. Potential negative impacts follow from improving the capability of general-purpose LVLMs: stronger multimodal reasoning could be misused for more persuasive automated content generation, surveillance-oriented visual analysis, or incorrect high-confidence decisions in sensitive domains. RAPO is evaluated as a research training method and is not intended for deployment in safety-critical settings without task-specific risk assessment, calibration, and human oversight. We do not release a new dataset or high-risk model checkpoint; released materials are intended for research reproduction.

\section{Case Study}
\subsection{Visualization of Entropy Distribution on Tokens}
Figure~\ref{fig:entropy1} and Figure~\ref{fig:entropy2} show the distribution of token entropy. We observe that high-entropy tokens are more concentrated around key reasoning steps, indicating increased uncertainty at decision points where multiple continuations remain plausible.

\begin{figure}[ht]
  \centering
  \begin{subfigure}[t]{0.95\linewidth}
    \centering
    \includegraphics[width=\linewidth]{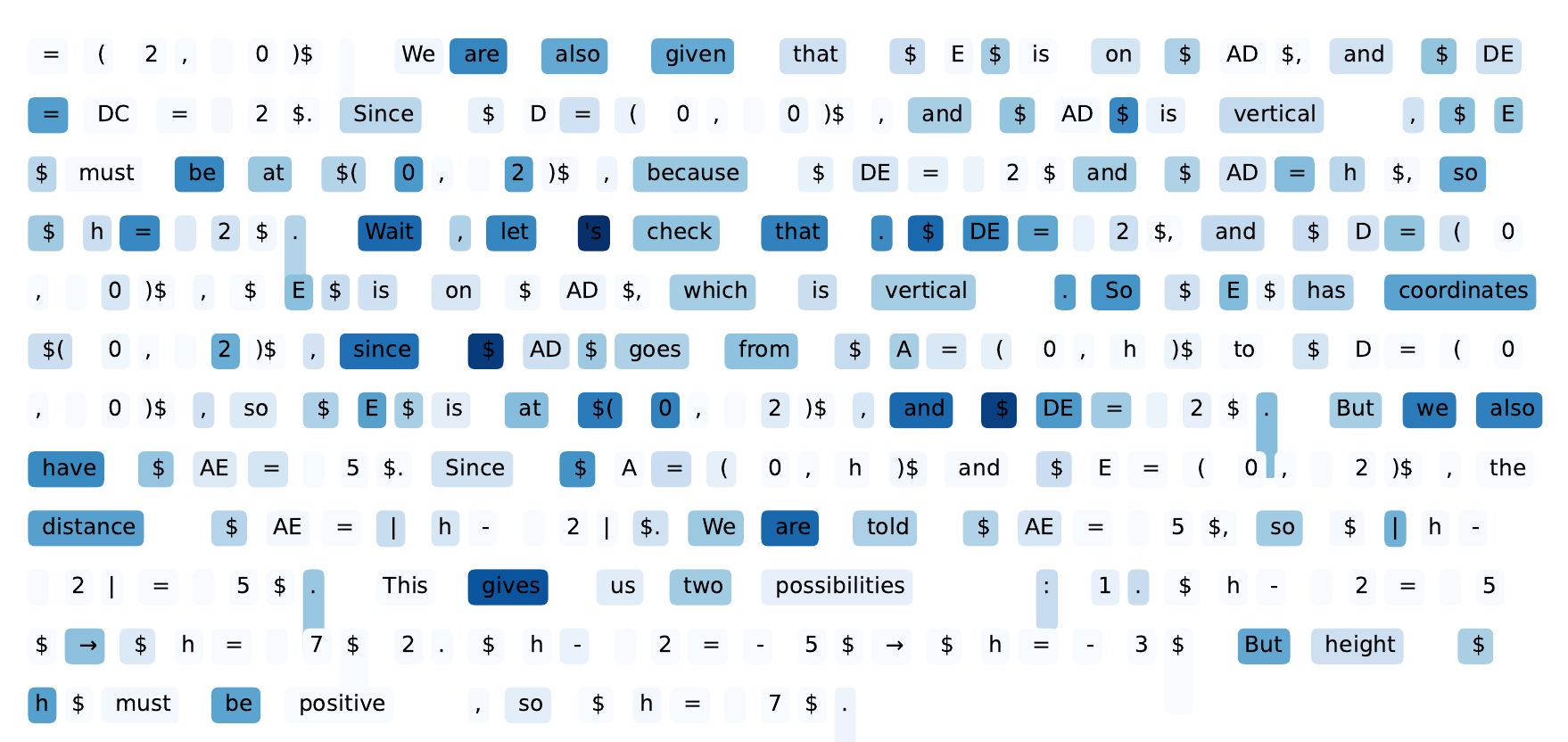}
    \caption{\textbf{Case 1.}}
    \label{fig:entropy1}
  \end{subfigure}
   \hfill
  \begin{subfigure}[t]{0.95\linewidth}
    \centering
    \includegraphics[width=\linewidth]{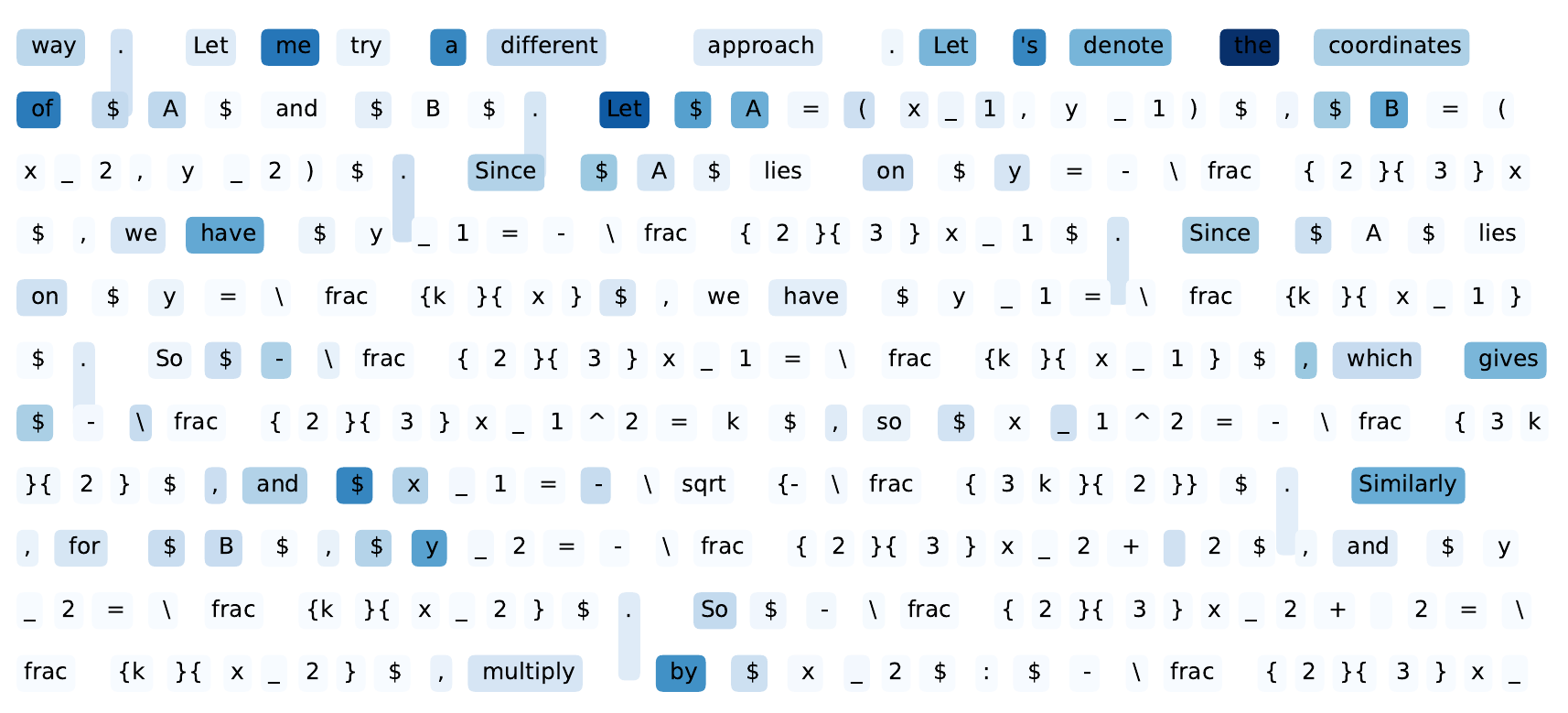}
    \caption{\textbf{Case 2.}}
    \label{fig:entropy2}
  \end{subfigure}
  \caption{\textbf{Case Study of Entropy Distributions along CoT.}}
   \vspace{0.5cm}
\end{figure}

\newpage

\subsection{Visualization of KL Distribution before and after RAPO training}
Figure~\ref{fig:kl0} and Figure~\ref{fig:kl1} show the KL distributions before and after RAPO training. Compared to the base model, RAPO consistently amplifies KL values in the later stages of generation, suggesting enhanced and sustained utilization of visual information throughout the reasoning process.

\begin{tcolorbox}[
    colframe=boxborder,     
    colback=boxbg,          
    boxrule=2pt,            
    arc=12pt,                
    auto outer arc,       
    left=15pt, right=15pt, top=15pt, bottom=15pt, 
    fontupper=\fontfamily{ptm}\selectfont 
]
    { \textbf{\textcolor{boxborder}{Case Study}}}

\includegraphics[width=0.4\linewidth]{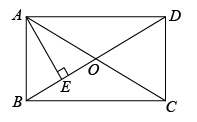}
\vspace{8pt}

{ \textbf{\textcolor{boxborder}{Question}}}

As shown in the figure, within rectangle $ABCD$, diagonals $AC$ and $BD$ intersect at point $O$, $AE$ perpendicular bisects $BO$, with $AE = \sqrt{3},\text{cm}$, what is the length of $OD$? \\ \\ You FIRST think about the reasoning process as an internal monologue and then provide the final answer. The reasoning process MUST BE enclosed within \verb|<think>| \verb|</think>| tags. The final answer MUST BE put in \verb|\boxed{}|.

\end{tcolorbox}

\begin{figure}[ht]
  \centering
\includegraphics[width=1.0\linewidth]{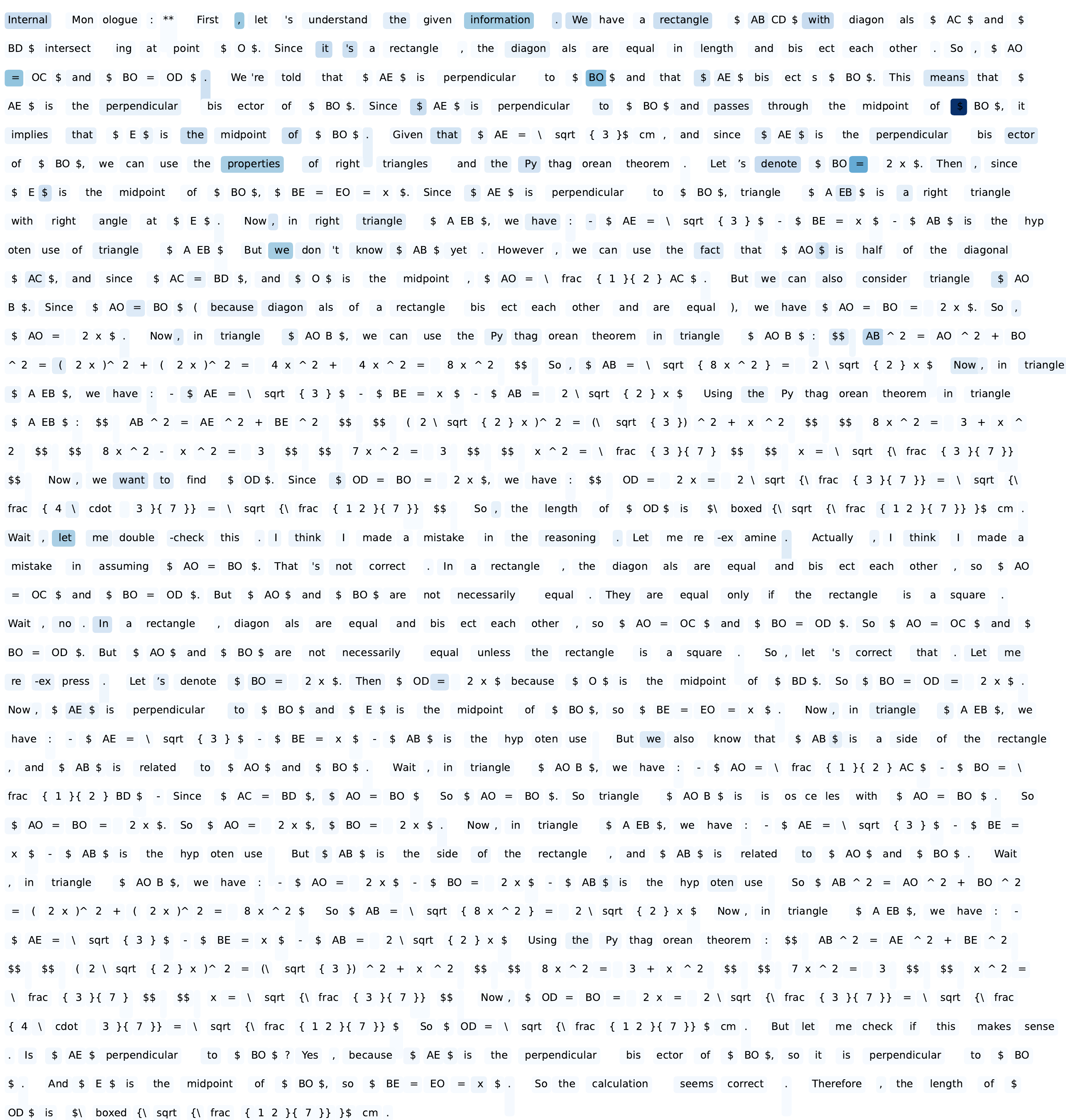}
\caption{\textbf{Token KL Distribution of Base Model (Wrong Answer).}}
\label{fig:kl0}
\end{figure}

\begin{figure}[ht]
  \centering
\includegraphics[width=1.0\linewidth]{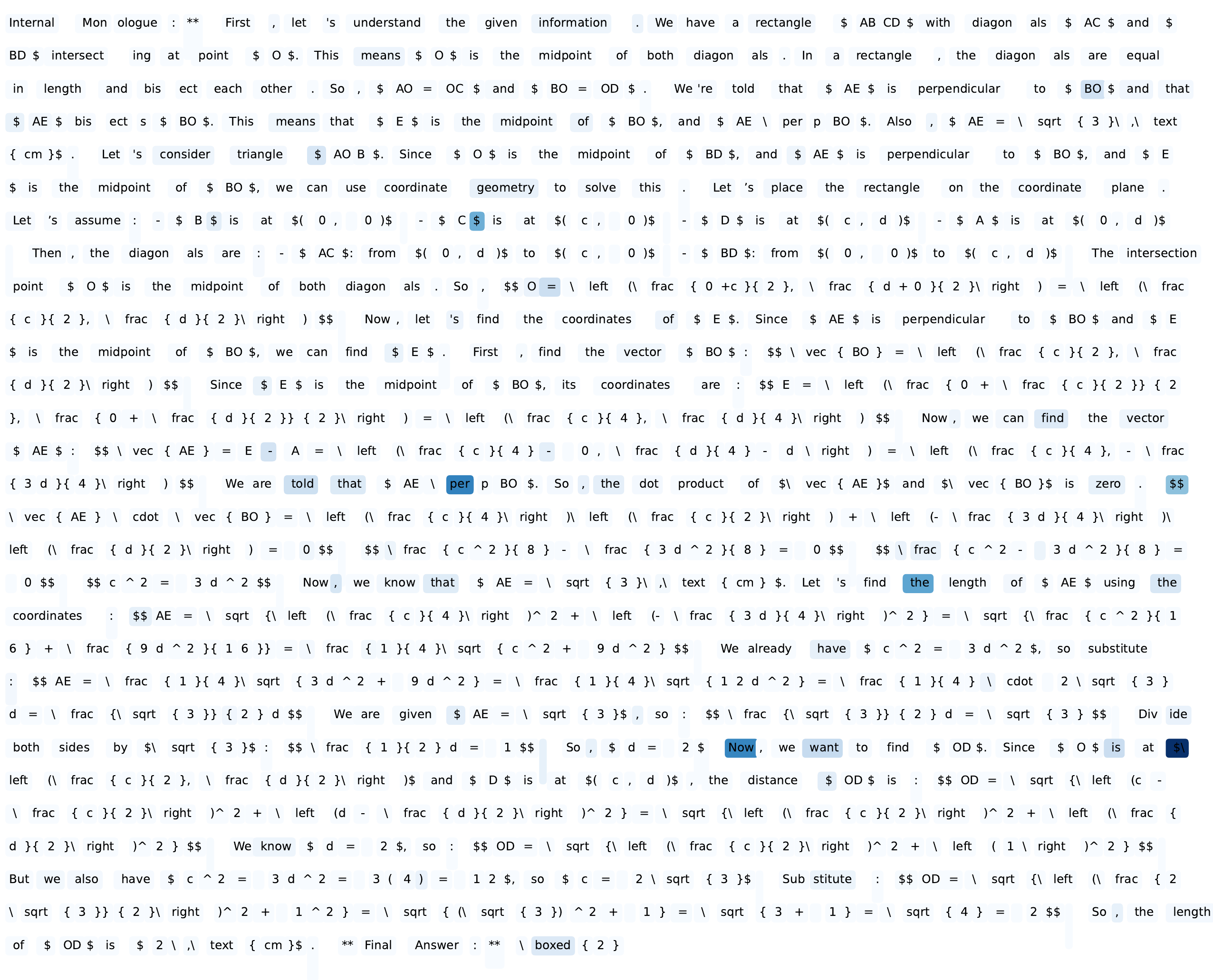}
\caption{\textbf{Token KL Distribution after RAPO Training (Right Answer).}}
\label{fig:kl1}
\end{figure}



\end{document}